\documentclass{article}
\def\shownotes{1}

\usepackage[main,final,nonatbib]{neurips_2026}

\usepackage[utf8]{inputenc} %
\usepackage[T1]{fontenc}    %
\usepackage{hyperref}       %
\usepackage{xcolor}         %
\definecolor{darkred}{rgb}{0.55, 0.0, 0.0}
\hypersetup{colorlinks=true, linkcolor=darkred, citecolor=blue,
  urlcolor=blue}
\usepackage{url}            %
\usepackage{booktabs}       %
\usepackage{amsfonts}       %
\usepackage{nicefrac}       %
\usepackage{microtype}      %

\usepackage{tabularx}
\usepackage{booktabs}
\usepackage{pifont}

\usepackage{graphicx}
\usepackage{subcaption}
\usepackage{multirow, makecell}
\usepackage{algorithm,algorithmic}
\let\algoAND\AND
\let\AND\classAND
\AtBeginEnvironment{algorithmic}{\let\AND\algoAND}

\usepackage{amsmath,amsfonts,bm,amsthm,mathrsfs,amssymb}

\theoremstyle{plain}
\newtheorem{theorem}{Theorem}[section]

\newtheorem{lemma}[theorem]{Lemma}

\theoremstyle{definition}
\newtheorem{definition}[theorem]{Definition}
\newtheorem{assumption}[theorem]{Assumption}
\theoremstyle{remark}
\newtheorem{remark}[theorem]{Remark}

\theoremstyle{definition}

\def\eqref#1{equation~(\ref{#1})}

\def\1{\bm{1}}

\DeclareMathAlphabet{\mathsfit}{\encodingdefault}{\sfdefault}{m}{sl}
\SetMathAlphabet{\mathsfit}{bold}{\encodingdefault}{\sfdefault}{bx}{n}

\newcommand{\set}[1]{\{#1\}}

\newcommand{\exarg}[2]{\mathop{\mathbb{E}}_{#1}\left[ #2 \right]}

\newcommand{\norm}[1]{\|#1\|}

\newcommand{\bignorm}[1]{\left\| #1 \right\|}

\DeclareMathOperator{\id}{Id}
\DeclareMathOperator{\tr}{Tr}
\newcommand{\bigtr}[1]{\tr\left[#1\right]}

\newcommand{\normFro}[1]{\norm{#1}_{_{F}}}

\newcommand{\bignorms}[1]{\bignorm{#1}}

\newcommand{\real}{\mathbb{R}}

\usepackage{macro}
\usepackage[backend=biber,
            style=numeric,
            natbib=true,
            maxbibnames=99,
            minalphanames=3,
            uniquename=false,
            uniquelist=false,            
            maxcitenames=2,
            sorting=ynt,
            giveninits=true,                        
            labeldateparts=true,
            sortcites=true,
            backref=false]{biblatex}
\title{Quantifying the Stability of Multi-Step Reasoning via Error Amplification}

\author{%
  Dongyue Li\\
  Northeastern University\\
  Boston, MA\\
  \texttt{li.dongyu@northeastern.edu}
  \And
  Ziniu Zhang\\
  Northeastern University\\
  Boston, MA\\
  \texttt{zhang.zini@northeastern.edu}
  \AND
  Minxuan Duan\\
  Northeastern University\\
  Boston, MA\\
  \texttt{dminxuan@gmail.com}
  \And
  Hongyang R. Zhang\\
  Northeastern University\\
  Boston, MA\\
  \texttt{ho.zhang@northeastern.edu}
}

\begin{document}

\maketitle

\begin{abstract}
We consider the stability of multi-step reasoning processes, which have extensive applications in language models, including chain-of-thought and algorithmic reasoning. While longer sequences of reasoning can improve a model's generation capability at test time, the errors due to intermediate reasoning steps can accumulate in autoregressive generation, and thus grow substantially at the end. In this paper, we ask: \emph{What are the key factors determining the stability of multi-step reasoning?} First, we show an inference error bound governed by the product of \emph{spectral norms of the Jacobians taken through the input space} across generation steps. This product can be viewed as an \emph{error amplification} factor, which could scale exponentially with the number of reasoning steps, serving as a quantitative measure of reasoning stability. Second, we analyze this measure in transformer models trained to predict simple tasks like linear and quadratic functions. We theoretically prove that the transformer model converges to a solution where the stability measure decays, thus yielding nearly zero inference loss over (arbitrarily) long steps. Finally, the stability analysis leads to several algorithmic implications for controlling the stability, through (i) chain-of-thought length compression that reduces the sensitivity of each step, and (ii) quantization-aware training that regularizes the input Jacobian norms. We validate the proposed algorithms by fine-tuning language models on graph-algorithmic reasoning tasks and symbolic state-tracking tasks. Across seven evaluations, our algorithms improve over baseline comparisons by 3.5\% on average, and by 8.2\% for longer-length inputs. Ablation analysis validates that the stability measure is drastically reduced by 3-8$\times$, confirming the regularization effect on the spectral norms of the (input space) Jacobians.
\end{abstract}

\section{Introduction}\label{sec_intro}

Multi-step reasoning enhances language model capabilities by generating intermediate steps before producing a final answer. Specifically, for an input $\mathbf{x}$, a model $f_W$ generates a reasoning sequence $(\mathbf{x}, \hat{y}^{(1)}, \ldots, \hat{y}^{(T)})$ autoregressively. As the error in previous outputs propagates over the steps, this can destabilize long reasoning and cause severe final prediction errors \cite{zhang2024language,ton2025understanding}. However, standard training like supervised fine-tuning does not explicitly control this error accumulation. 
In this work, we theoretically analyze such error amplification across steps, aiming to identify the key factors governing reasoning stability and design training methods that control accumulated errors.

Understanding such error amplification is valuable in many practical problems, especially for scaling test-time computation \cite{geiping2025scaling,ehrlich2025codemonkeys}. One concrete setting is algorithmic reasoning, where models perform the step-by-step computation of classical algorithms like Bellman-Ford~\cite{li2025efficiently}. Models are expected to learn these underlying algorithms to generalize and extrapolate to larger, unseen inputs. However, a mistake in the early steps can corrupt the procedure. In our own experiments, we observe that the test error in a fine-tuned language model (e.g., Qwen-1.5B and Gemma-2B) can scale super-linearly with steps, as shown in Figure \ref{fig_test_loss_scaling_intro}. Empirical methods, such as training models to generate fewer steps \cite{deng2024explicit} or reason in the continuous feature space \cite{hao2025training}, have been proposed to reduce these errors. Yet, the reason for their empirical benefits remains underexplored. 

\begin{figure*}[t!]
    \centering
    \begin{minipage}[t!]{0.33\textwidth}
        \centering
        \includegraphics[width=0.95\linewidth]{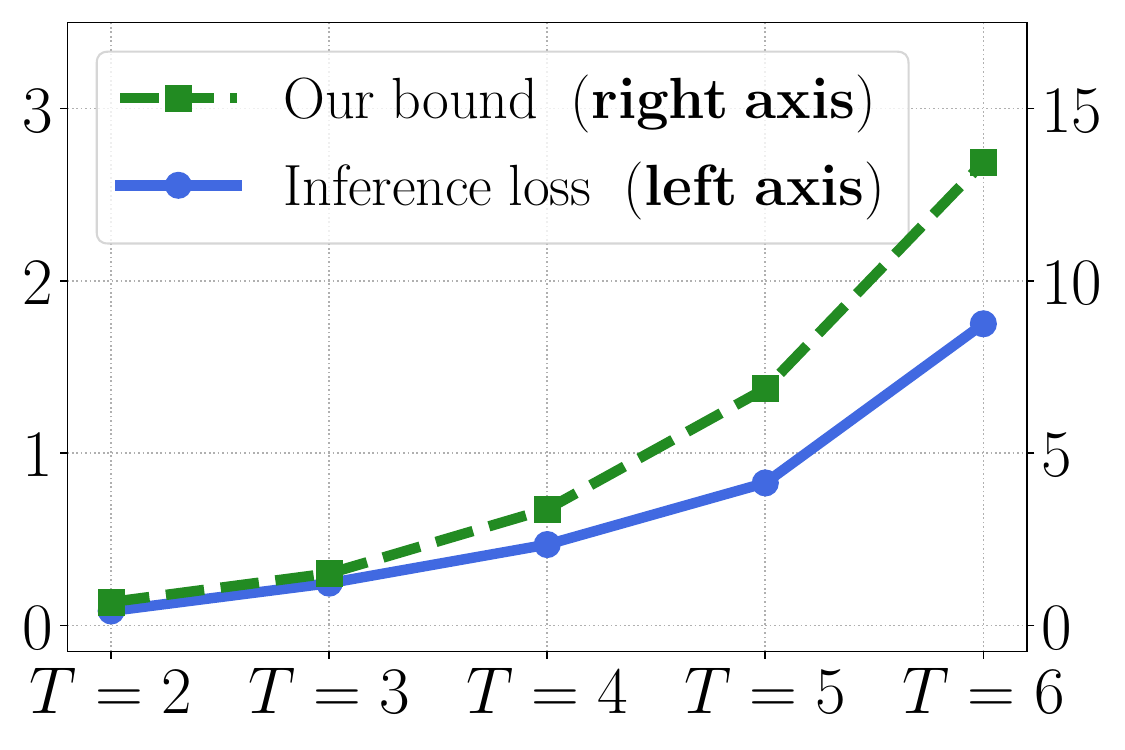}
        \subcaption{Evaluation of our bound values}\label{fig_test_loss_scaling_intro}
    \end{minipage}\hfill
    \begin{minipage}[t!]{0.33\textwidth}
        \centering
        \includegraphics[width=0.95\linewidth]{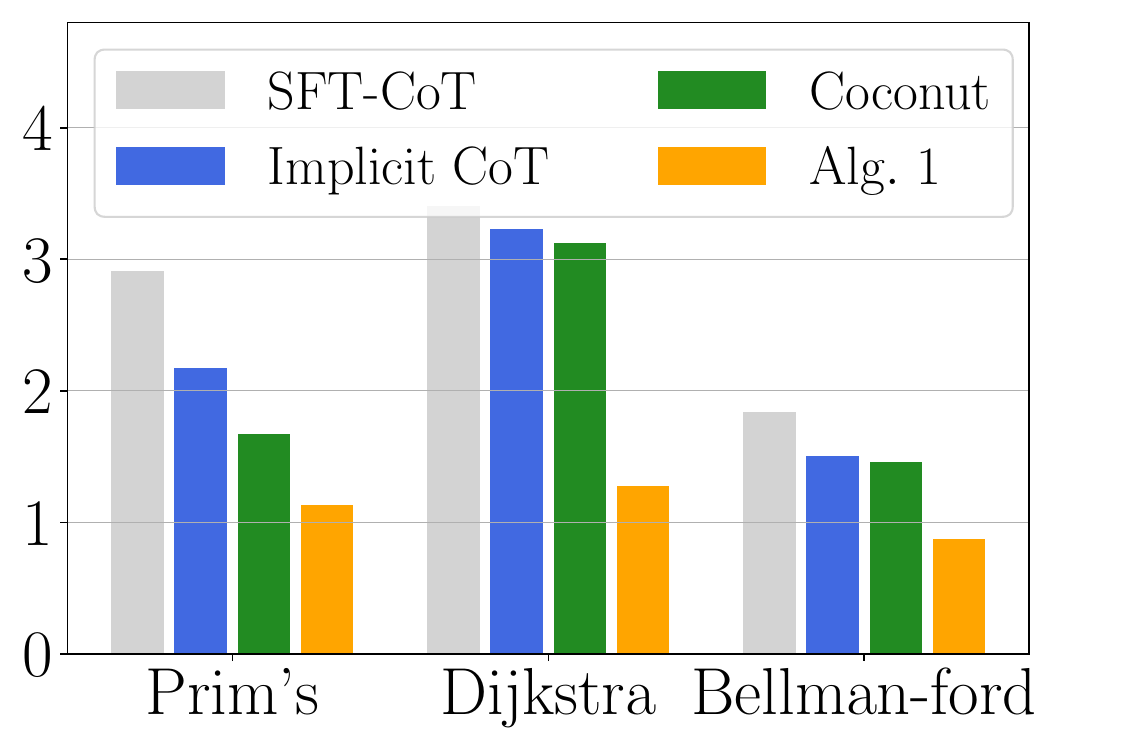}
        \subcaption{Inference losses}\label{fig_test_loss_intro}
    \end{minipage}\hfill
    \begin{minipage}[t!]{0.33\textwidth}
        \centering
        \includegraphics[width=0.95\linewidth]{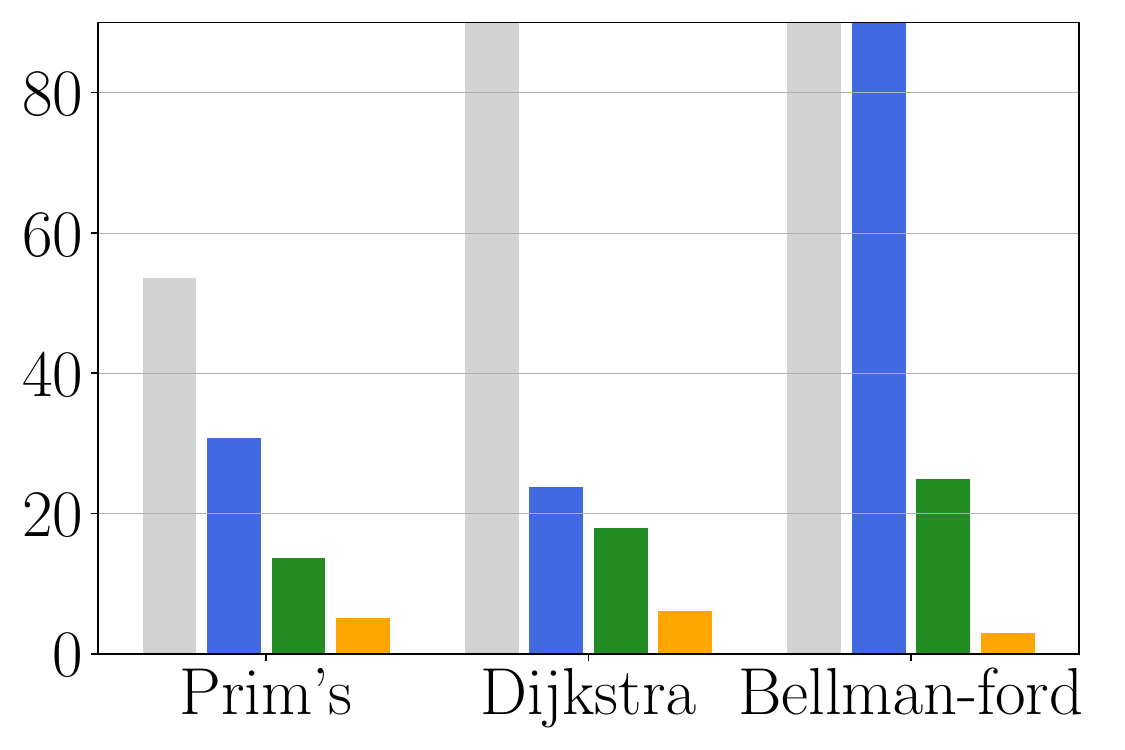}
        \subcaption{Spectral norm of Jacobians}\label{ref_error_amplification_factors_intro}
    \end{minipage}
    \caption{We present a bound on the inference loss in multi-step reasoning governed by the product of spectral norms of Jacobians taken through the input space across generation steps. Figure \ref{fig_test_loss_scaling_intro}: Our bound (\eqref{eq_main_bound}) tracks the inference loss on the Bellman-Ford algorithmic reasoning task, showing that both can scale exponentially with steps.  Figure \ref{fig_test_loss_intro}: We design a training method that regularizes the Jacobian spectral norms, which improves the test performance over baselines. Figure \ref{ref_error_amplification_factors_intro}: Models fine-tuned with our method exhibit drastically reduced Jacobian spectral norms.}\label{fig_main}
\end{figure*}

Recent theoretical work has presented convergence and generalization guarantees of transformers trained on specific reasoning tasks, including linear regression \cite{huang2025transformers,zhang2025linear}, state tracking \cite{huang2025transformers_length}, and path-finding in trees \cite{yang2025multi}. Their analyses focus on the task-specific transformer structures, leaving open the general cases. Studies based on PAC learning theory \cite{joshi2025theory,altabaa2025cot} yield VC-dimension generalization bounds dependent on the entire hypothesis class. This motivates us to conduct a general analysis of multi-step reasoning with computable measures on data and fine-tuned models.

We begin by analyzing the inference error amplification of multi-step reasoning. Our key finding is that the inference loss is governed by the products of the \emph{spectral norms of Jacobians taken through the input space} across generation steps, which control reasoning stability. The idea is to bound the gap between the $T$-step inference loss, $\ell(f_W([\mathbf{x}, \hat{y}^{(1)}, \ldots, \hat{y}^{(T-1)}]), {y}^{(T)})$, and the training loss, $\ell(f_W([\mathbf{x}, {y}^{(1)}, \ldots, {y}^{(T-1)}]), {y}^{(T)})$, conditioned on the correct steps. Using a first-order approximation in the input space, we show that the prediction errors at step $i$ are propagated forward by the Jacobian matrix $J_{i}^{(t)}$ (evaluated at the $t$-th step output to the input $y^{(i)}$). Let $\rho_{i}^{(T)}$ be the sum of products of $\norm{J_{i}^{(t)}}$ across all paths from step $i$ to $T$. Then, recursively decomposing the errors yields a bound for the inference loss governed by $\sum_{i=1}^{T-1} \rho_{i}^{(T)}$ (cf. Lemma \ref{thm_test_error_bound}), which we define as an error amplification factor. 
Empirically, this bound closely tracks the test losses on fine-tuned language models (including Qwen, Gemma, and Llama models), which can scale exponentially with $T$, as shown in Figure \ref{fig_test_loss_scaling_intro}. 

Next, we instantiate this bound by analyzing one-layer multi-head transformers trained on weight-prediction tasks for linear and quadratic functions. In both settings, we prove that the transformer converges to a solution where $\sum_{i=1}^{T-1}\rho_{i}^{(T)}$ decays with $T$, with the spectral norm of the Jacobians evaluated at each step input strictly less than one. Consequently, the trained transformer achieves a test loss of $O(1/\text{poly}(d))$ with input dimension $d$ and $T = \Theta(\log d)$ steps (cf. Theorem \ref{theorem_convergence_three_head}). Extending prior work on linear functions \cite{huang2025transformers}, our analysis for quadratic functions introduces a novel three-head transformer construction that exactly implements gradient descent (cf. Lemma \ref{theorem_construction_three_head}), a finding that may be of independent interest.

Motivated by this analysis, we propose a training method to control accumulated error by regularizing the Jacobian-based error amplification factor. Our method incorporates chain-of-thought length compression that downsamples the reasoning steps during training, and quantization-aware training \cite{kwun2025lotion, li2026winq} that trains weights with gradients computed on the quantized model. We show that our method reduces the Jacobian spectral norms by regularizing the Hessian trace of the model loss during quantized training. 
We evaluate this method on five algorithmic reasoning tasks \cite{li2025efficiently} and two state-tracking tasks from LEGO \cite{zhang2022unveiling}. 
First, across all seven tasks, our method yields an average test accuracy improvement of \textbf{3.5}$\%$ over baselines, including implicit chain-of-thought \cite{deng2024explicit} and Coconut \cite{hao2025training}, shown in Figure \ref{fig_test_loss_intro}. Second, on evaluation of length generalization with larger inputs, our method enhances the test performance by \textbf{8.2}$\%$ over baselines. Finally, ablation studies confirm that our method reduces the Jacobian-based error amplification factor by up to \textbf{8$\times$} compared to existing methods, as shown in Figure \ref{ref_error_amplification_factors_intro}. 
The code for
replicating these empirical findings is available at \href{https://github.com/VirtuosoResearch/Multi-step-reasoning-experiments}{https://github.com/VirtuosoResearch/Multi-step-reasoning-experiments}.

\textbf{Summary of contributions.} First, our analysis reveals that the inference error in multi-step reasoning scales with the product of spectral norms of Jacobians taken through the input space. Second, we theoretically show that transformers trained on linear and quadratic functions converge to a regime where these norms are strictly less than one, yielding near-zero inference error. Finally, we propose a training method that regularizes these input Jacobian norms, empirically enhancing test performance and length generalization across diverse reasoning tasks.

\section{Preliminaries}

We now describe the problem setup for multi-step reasoning. 
Let $\cX$ denote the feature space. Let $f_W: \cX^\ast \to \cX$ denote a next-token predictor with parameters $W \in \mathbb{R}^p$, such as a language model. $f_W$ takes a sequence of vectors in $\cX$ of arbitrary length as input and outputs a single vector in $\cX$, which is then appended to the sequence before the next prediction during inference.

For an input sequence $\mathbf{x} \in \cX^\ast$, we assume that there exists an underlying rule generating a ground-truth sequence $[\mathbf{x}, y^{(1)},\dots,y^{(T(\mathbf{x}))}]$, where $y^{(t)} \in \cX$ is the ground-truth output at step $t$, and $T(\mathbf{x})$ denotes the total number of steps, which depends on the input $\mathbf{x}$. We denote the maximum number of steps as $T$ for a particular task and bounded input lengths. As a concrete example, suppose $\mathbf{x}$ contains a graph and a source node; then each $y^{(i)}$ encodes an intermediate execution step of a graph traversal algorithm such as breadth-first search \cite{li2025efficiently}.

At inference time, the model $f_W$ generates $T(\mathbf{x})$ steps for an input $\mathbf{x}$ autoregressively. We denote the prediction at step $t$ as \[ \widehat{y}^{(t)} = f_W\big([\mathbf{x}, \widehat{y}^{(1)}, \dots, \widehat{y}^{(t-1)}]\big), \text{ with } ~\widehat{y}^{(1)} = f_W(\mathbf{x}). \]
For simplicity of notation, we write the generated sequence after $t$ steps as \[ \widehat{\mathbf z}^{(t)} = [\mathbf{x}, \widehat{y}^{(1)}, \dots, \widehat{y}^{(t)}], \text{ with } ~\widehat{\mathbf z}^{(0)} = [ \mathbf x ]. \] 
The corresponding ground-truth sequence with the correct intermediate steps is denoted as \[ \mathbf{z}^{(t)} = [\mathbf{x}, y^{(1)}, \dots, y^{(t)}], \text{ for } t = 1, \dots, T(\mathbf{x}), \text{ and } \mathbf{z}^{(0)} = \mathbf{x}. \] 
During training, the model $f_W$ predicts the next step conditioned on the preceding steps that encode the ground truth. Let $S = \big \{(\mathbf{x}_i, y_i^{(1)}, \dots, y_i^{(T(\mathbf{x}_i))}) \big \}_{i=1}^n$ denote a training dataset of $n$ examples, where every $\mathbf{x}_i$ is drawn independently from an unknown distribution $\cD$. Given a loss function $\ell(\cdot, \cdot)$, the training loss is defined as:
\begin{align}
    \hat{L}(f_W; S) = \frac{1}{n} \sum_{i=1}^n \sum_{t=0}^{T(\mathbf{x}_i)-1} \ell \bigl(f_{W}\big(\mathbf{z}_{i}^{(t)}\big), y_{i}^{(t+1)}\bigr).\label{eq_training_loss}
\end{align}
At test time, the model predicts the next step conditioned on previous predictions. The outcome is evaluated by the loss function at the final step. We thus define this test loss over a random sample $\mathbf{x}$ drawn from the data distribution $\cD$ as:
\begin{align}
  L(f_W; \cD) = \exarg{\mathbf{x} \sim \cD}{\ell\bigl(f_W(\widehat{\mathbf{z}}^{(T(\mathbf{x})-1)}), y^{(T(\mathbf{x}))}\bigr)}.
  \label{eq_test_loss}
\end{align} 
To be clear, we will call the above \eqref{eq_test_loss} the \emph{inference loss}.
At training time, minimizing \eqref{eq_training_loss} measures the loss relative to the ground truth. At inference time, the model assimilates prediction errors that occurred during previous steps back into the model. Thus, the prediction errors during inference can accumulate during long generation steps.

\textbf{Problem statement.} 
We quantify how inference loss accumulates throughout generation. We study conditions under which {multi-step reasoning is stable}, where the inference loss $L(f_W; \cD)$ remains bounded in long chains.
Additionally, we design training algorithms to mitigate error accumulation and improve the stability of multi-step reasoning.  

\textbf{Notations.}
For any matrix $X$, let $\norm{X}$ denote the largest singular value
(or spectral norm) of $X$. Let $\norm{X}_F$ denote the Frobenius norm of $X$.
$f(N) \lesssim g(N)$ indicates that there exists
a fixed constant $c$ that does not grow with $N$ such that
$f(N) \leq c \cdot g(N)$ for large enough integers of $N$.

\section{Theoretical Analysis}\label{sec_error_accumulation}

This section presents a theoretical analysis of multi-step reasoning in the above setup. First, we present a bound on the inference loss relative to the ground truth generation. We show that this bound depends on an \emph{error amplification factor} defined by the product of spectral norms of Jacobian matrices measured at the input space. Second, we instantiate this bound in in-context prediction tasks for linear and quadratic functions. We show that transformer models converge to a regime where the error amplification factor decays across steps, thus yielding nearly zero inference loss.

\subsection{Error Amplification in Multi-Step Reasoning}

We now quantify how the inference loss accumulates across generation steps. We write the number of steps as $T$ when the input is clear from the context. For an input $\mathbf{x}$, we compare the inference loss $\ell\bigl(f_W(\widehat{\mathbf{z}}^{(T-1)}), y^{(T)}\bigr)$ of \eqref{eq_test_loss} with the loss $\ell\bigl(f_W(\mathbf{z}^{(T-1)}), y^{(T)}\bigr)$ of the prediction conditioned on the correct steps, which is what training minimizes in \eqref{eq_training_loss}. 

To measure the difference, let $e^{(t)}(\mathbf{x}) = f_W(\mathbf{z}^{(t-1)}) - y^{(t)}$ and $\widehat{e}^{(t)}(\mathbf{x}) = f_W(\widehat{\mathbf{z}}^{(t-1)}) - y^{(t)}$ denote the prediction error at step $t$ conditioned on the correct steps and on the model's generations, respectively.
Consider the first-order Taylor expansion in the input space, for every $t$,
\[ f_W(\widehat{\mathbf{z}}^{(t-1)}) - f_W(\mathbf{z}^{(t-1)}) \; = \; \sum_{i=1}^{t-1} J^{(t)}_i(\mathbf{x}) \, \widehat{e}^{(i)}(\mathbf{x}) \; + \; \mathcal{R}^{(t)}, \text{ where } J^{(t)}_i(\mathbf{x}) = \frac{\partial f_W(\mathbf{z}^{(t-1)})}{\partial y^{(i)}} \]
is the input Jacobian evaluated at step $t$, and $\mathcal{R}^{(t)}$ is the second-order remainder error term.
By recursively unrolling this expansion, an error at step $i$ affects a later step $t$ by first propagating to an intermediate step $k$ and then from $k$ to $t$. To capture the propagation path, we define a term:
\begin{equation*}
    \rho^{(t)}_i = \sup_{\mathbf{x} \in \cX^\ast} \norm{J^{(t)}_i(\mathbf{x})} + \sum_{k=i+1}^{t-1} \left( \sup_{\mathbf{x} \in \cX^\ast} \norm{J^{(t)}_k(\mathbf{x})} \cdot \rho^{(k)}_i \right), \text{ with } \rho^{(t)}_t = 1,  
\end{equation*}
which sums, over all paths from $i$ to $t$, the product of the spectral norms of the input Jacobians along the path.  Therefore, an error at step $t$ reaches the final step amplified by a factor $\rho^{(T)}_t$. 

Our main result is to incorporate the spectral norm of Jacobians to quantify the accumulation of inference errors across generation steps. For every $\mathbf{x}$ in the support of $\cD$, we denote the maximum of the two errors, for $t$ from $1$ to $T-1$, by:
\begin{align}
    \sup_{\mathbf{x}\in\cX^\ast} \norm{e^{(t)}(\mathbf{x})} =  \varepsilon_t,~ \sup_{\mathbf{x}\in\cX^\ast} \norm{\widehat{e}^{(t)}(\mathbf{x})}  =  r_t.
    \label{eq_per_step_levels}
\end{align}
We present a bound on the (expected) inference loss as follows. 

\begin{lemma}
\label{thm_test_error_bound}
Suppose that $f_W$ is twice differentiable with respect to its input, the loss function $\ell(\cdot, \cdot)$ is $C_1$-Lipschitz continuous in its first argument, and the Jacobian of $f_W$ with respect to its input is $C_2$-Lipschitz continuous.
Let $L^{\star}(f_W; \cD) = \exarg{\mathbf{x} \sim \cD}{\ell\bigl(f_W(\mathbf{z}^{(T-1)}, y^{(T)})\bigr)}$ be the expected loss at the final step of the model prediction conditioned on correct steps.
Then, for any fixed $W \in \mathbb{R}^p$ that satisfies these conditions, and under \eqref{eq_per_step_levels}, the inference loss is bounded away from the expected loss as follows:
\begin{align}
\label{eq_main_bound}
    \big|L(f_W; \cD) - L^{\star}(f_W; \cD)\big| \leq  { C_1 \sum_{t=1}^{T-1} \rho^{(T)}_t \varepsilon_t +  \frac{C_1 C_2}{2} \sum_{t=2}^{T} \rho^{(T)}_t \left( \sum_{i=1}^{t-1} r_i \right)^2 }.
\end{align}
\end{lemma}

The proof involves a decomposition of the error between predictions on the model-generated sequence and the correct step sequence by applying the first-order expansion recursively. The full proof is in Appendix~\ref{sec_proof_test_error_bound}. 
Lemma \ref{thm_test_error_bound} highlights that the inference loss is governed by $\sum_{i=1}^{T-1} \rho^{(T)}_i$, which we term an error amplification factor.  
Even if the training error is small, this factor can grow exponentially with $T$, resulting in large inference errors on long reasoning sequences.

\begin{figure*}[t!]
\centering
    \begin{minipage}[t!]{0.3\textwidth}
        \centering
        \includegraphics[width=\linewidth]{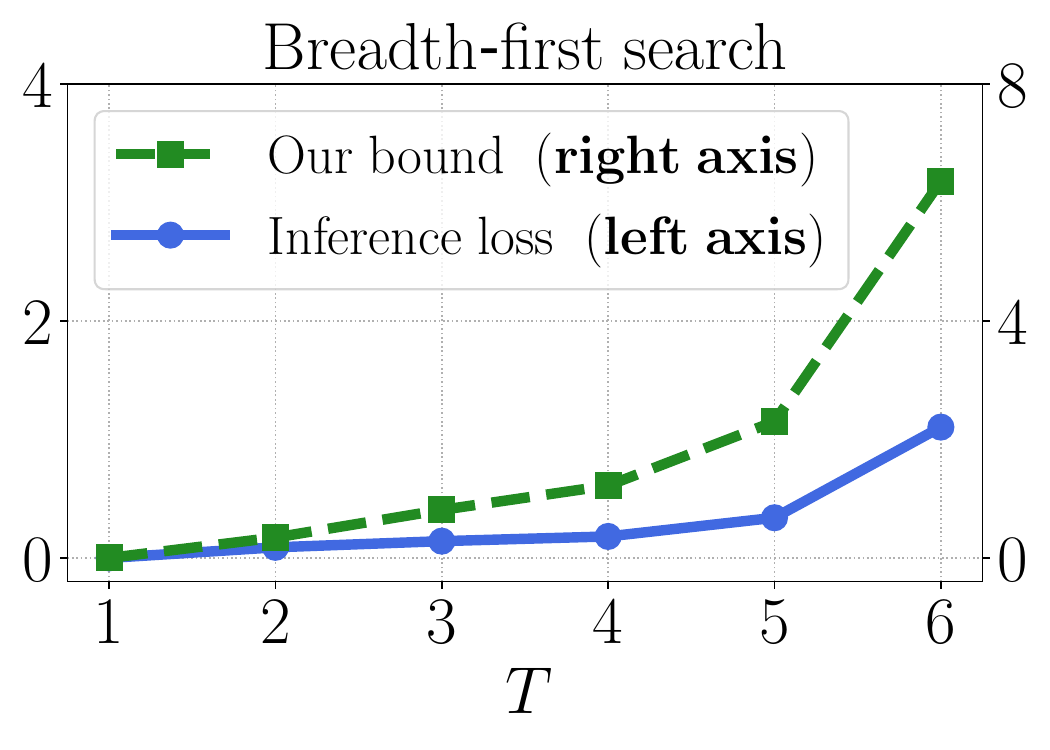}
    \end{minipage}\hfill
    \begin{minipage}[t!]{0.3\textwidth}
        \centering
        \includegraphics[width=\linewidth]{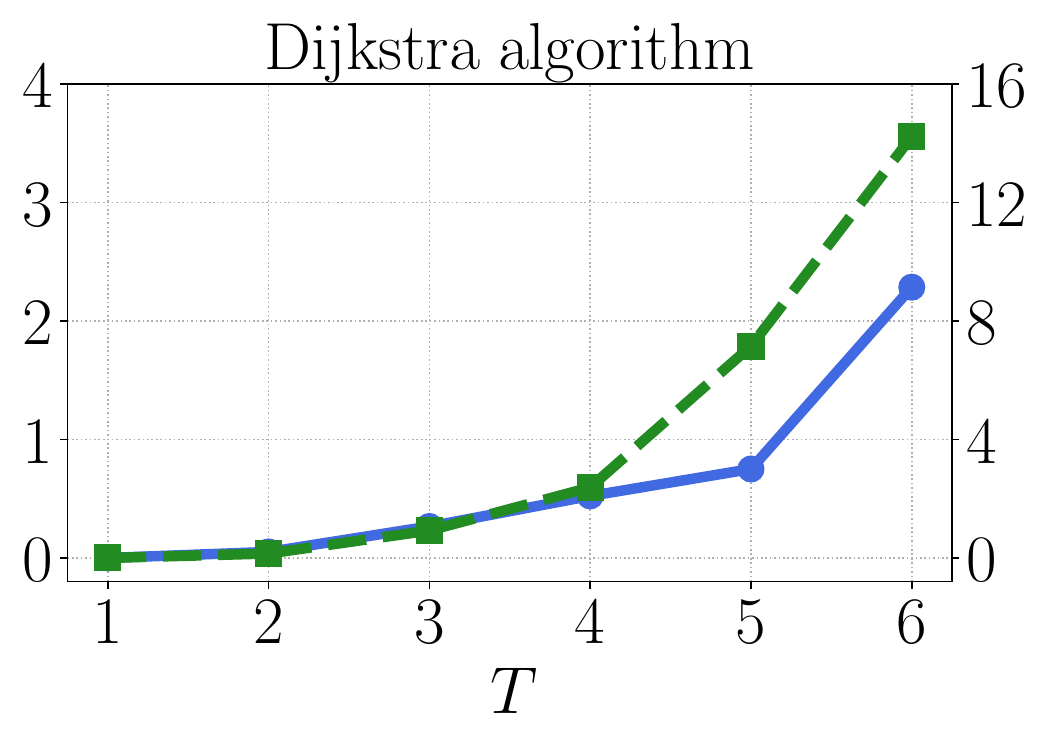}
    \end{minipage}\hfill
    \begin{minipage}[t!]{0.3\textwidth}
        \centering
        \includegraphics[width=\linewidth]{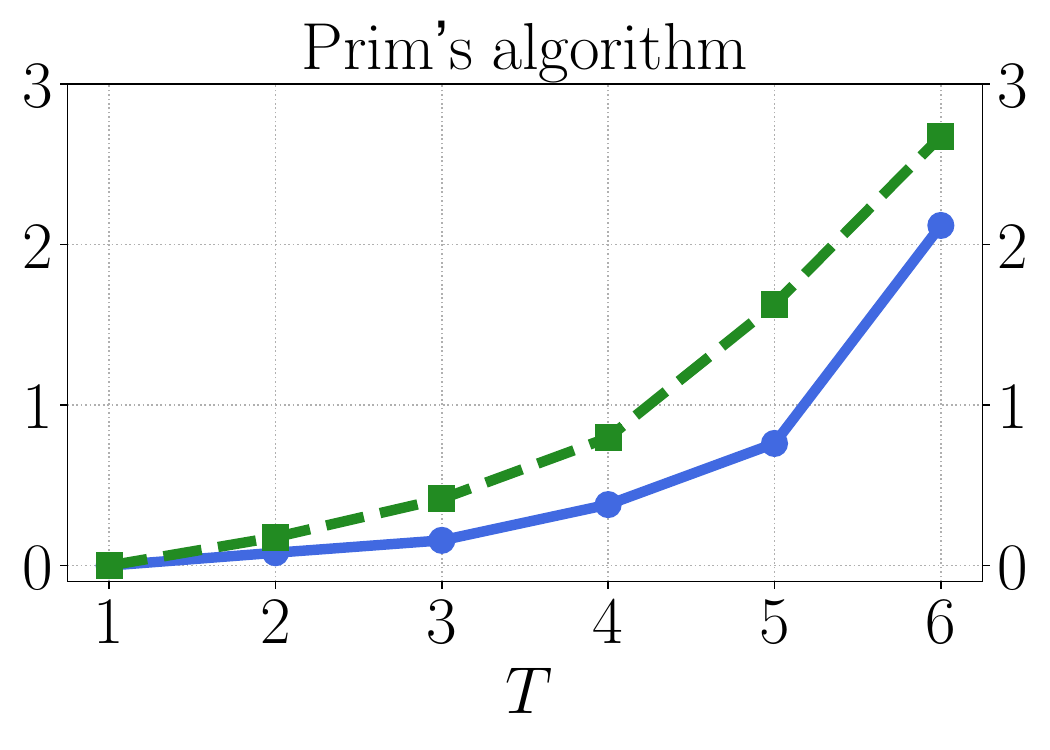}
    \end{minipage}
    \caption{We illustrate the loss gap and our bound in \eqref{eq_main_bound}, which can scale exponentially across generation steps, tracking the growth of the inference loss. Results are evaluated on breadth-first search, Dijkstra, and Prim's algorithm, using QWen-1.5B. The loss gap and the bound are plotted on separate axes; \eqref{eq_main_bound} is a worst-case bound and is numerically larger than the gap it bounds. Table \ref{tab_bound_scale} in Appendix \ref{sec_additional_experiments} reports the corresponding bound-to-gap ratios and constants.}\label{fig_scaling}
\end{figure*}

\smallskip
\textbf{Examples.} To empirically verify our analysis, we study the scaling of inference losses on graph algorithmic reasoning tasks \cite{li2025efficiently}. We fine-tune Qwen-1.5B, Llama-1B, and Gemma-2B models on tasks with graphs of $10$ nodes. Inputs are text-based graph descriptions, and intermediate steps are encoded as node labels, such as the nodes along the shortest path. As a reasoning step in text-based datasets typically involves several tokens, we compute the loss, the $\ell_2$ error, and the Jacobian matrices token-by-token within a step, and then aggregate them over the tokens within one step.

For text-based reasoning tasks, we view $\cX$ as the continuous space of input embeddings. The output of $f_W$ at step $t$ is then the expected embedding of the predicted next token, calculated as the weighted average of all token embeddings based on their predicted softmax probabilities. 
We compute the Jacobian of the output vector with respect to the input embeddings.
We estimate its spectral norm via power iteration using Jacobian-vector products. To approximate $\sup_{\mathbf{x}} \norm{J_{k}^{(t)}}$ in \eqref{eq_main_bound}, we take the maximum norm across the dataset.

After minimizing the training loss, we measure the loss gap and the empirical values of our bound in \eqref{eq_main_bound}, for $T$ ranging from 1 to 6. In Figure \ref{fig_scaling}, the inference loss increases significantly with the number of steps. Our bound,  which involves $\sum_{i=1}^{T-1} \rho^{(T)}_i$, scales exponentially with $T$, closely tracking the growth in inference loss. We also observe similar results on Llama-1B and Gemma-2B models, which are described in Appendix \ref{sec_additional_experiments}. 
These results demonstrate that controlling input Jacobian norms during training is crucial for reducing accumulated errors in multi-step generation. 

\smallskip
\begin{remark}
Lemma \ref{thm_test_error_bound} provides a worst-case upper bound that may not be numerically tight.  Empirically, while we observe a consistent correlation between the bound values and inference losses across models, the numerical scales of the error amplification factor $\sum_{i=1}^{T-1} \rho^{(T)}_i$ vary significantly. The bound-to-gap ratio is around $8$ for Qwen-1.5B and around $10^5$–$10^6$ for Llama-1B and Gemma-2B. 
We report the empirical loss gap, the bound, the bound-to-gap ratio, both terms of \eqref{eq_main_bound}, and the estimated constants $C_1$ and $C_2$, for every model in Appendix \ref{sec_additional_experiments}.
\end{remark}

\subsection{Learning Dynamics in Quadratic Models} 
\label{sec_training_dynamics}

Next, we explicitly instantiate the bound in Lemma \ref{thm_test_error_bound} in training transformers on an in-context weight prediction task. We show that if the spectral norm of the input Jacobian matrix is less than one, the inference loss can be reduced close to zero. 

Based on prior work on in-context learning \cite{zhang2025linear}, we study a task that predicts the weight vector $\theta^\star \sim \mathcal{N}(0, \id_d)$ of a function, given $n$ examples $(x_i, y_i)$ where $x_i \sim \mathcal{N}(0, \id_d)$ and $y_i = g_{\theta^\star}(x_i)$ for a known function $g$. We define the ground-truth steps as iterates by performing gradient descent on the regression examples
\[ \theta_{i+1} = \theta_{i} - \frac{\eta}{n}\sum_{k=1}^n \nabla_\theta g_{\theta_i}(x_k, y_k), \]
with a learning rate $\eta$ and an initialization $\theta_0$.

We consider a one-layer transformer with $h$ non-linear self-attention heads:
\begin{align}\label{eq_nonlinear_transformer}
f_W(Z) = Z_{[:, -1]}
+
\sum_{i=1}^h
W^{PV_i}Z \cdot \phi_i \Big(\tfrac{Z^\top W^{KQ_i} Z_{[:, -1]}}{n}\Big), \text{ and } Z  = 
\begin{bmatrix}
X & 0_d & \cdots & 0_d \\
y & 0 & \cdots & 0 \\
0_{d\times n}   & \theta_0 & \cdots & \theta_i \\
0   & 1    & \cdots & 1
\end{bmatrix},
\end{align}
where $Z \in \real^{(2d + 2) \times (n+i+1)}$ is the input embedding, $X = [x_1, \dots, x_n]$, and $y = [y_1,\dots,y_n]$. The ones in the bottom are appended for formatting convenience.
For step $T$, the last column contains the target vector $\theta^\star$.
Following prior work \cite{huang2025transformers,huang2025transformers_length,yang2025multi}, $W^{PV_i}$ and $W^{KQ_i}$ are the merged projection-value and key-query matrices, $\phi_i$ is a non-linear function, and $Z_{[:, -1]}$ denotes the last column.
We use the squared loss for this setting.

\textbf{Numerical example.} We validate that the transformer can achieve near-zero inference loss in this setting, for both linear and quadratic function classes.
The results are shown in Figure \ref{fig_simulation}. The inference loss converges to near-zero, with longer steps yielding lower inference loss. %

\begin{figure*}[t!]
\centering
    \begin{minipage}[t!]{0.245\textwidth}
        \centering
        \includegraphics[width=0.95\linewidth]{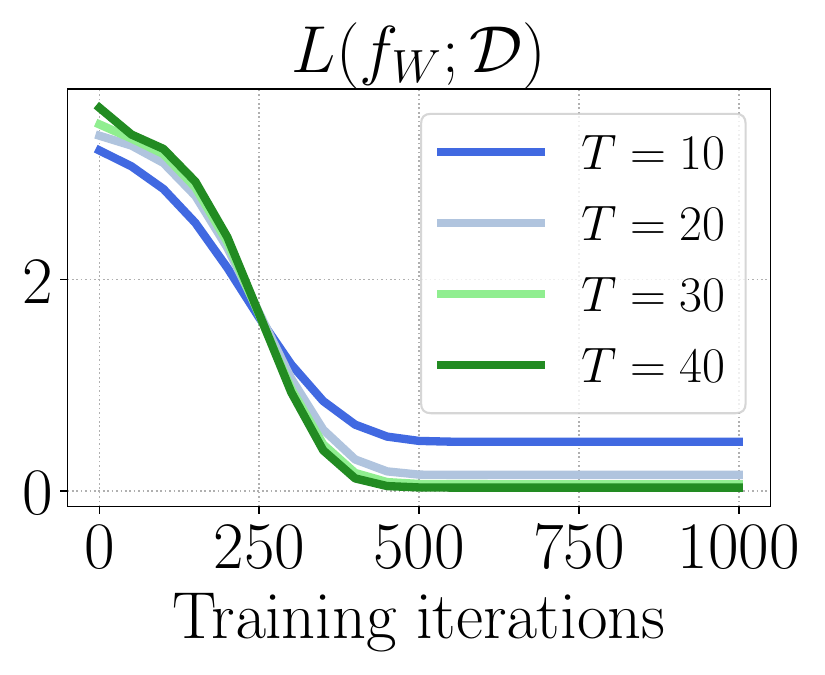}
        \subcaption{Linear functions}
    \end{minipage}
    \begin{minipage}[t!]{0.245\textwidth}
        \centering
        \includegraphics[width=0.95\linewidth]{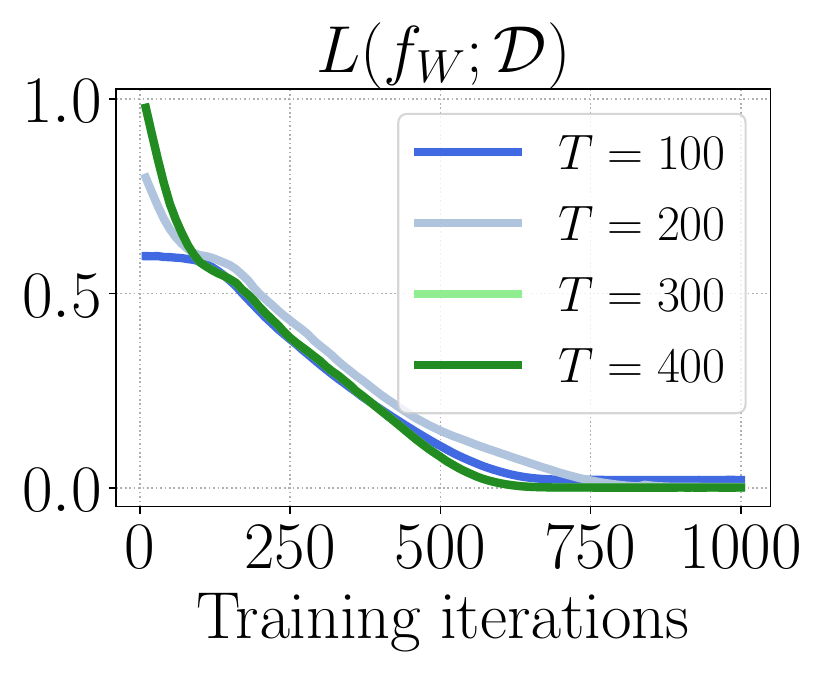}
        \subcaption{Quadratic functions}
    \end{minipage}
    \begin{minipage}[t!]{0.245\textwidth}
        \centering
        \includegraphics[width=0.95\linewidth]{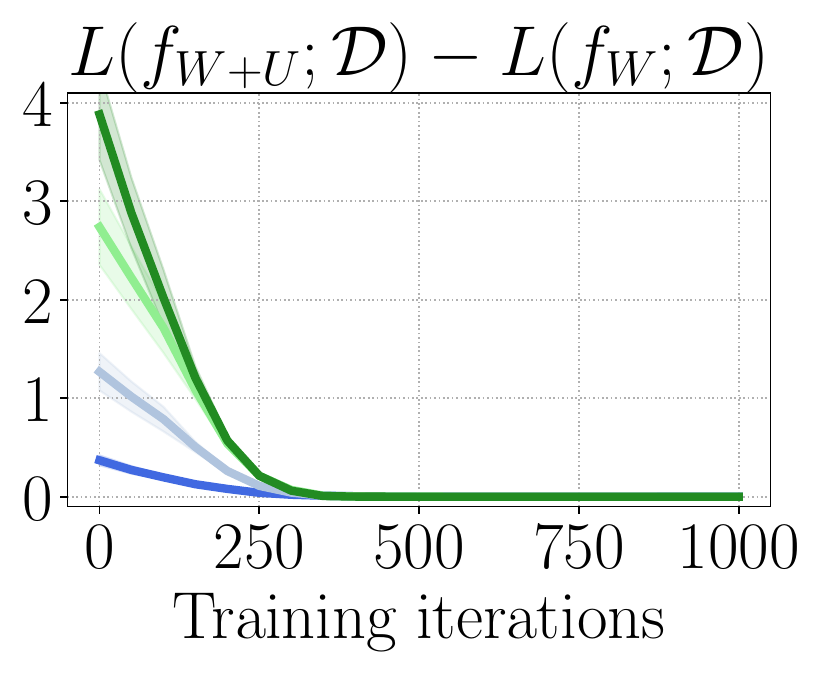}
        \subcaption{Linear functions}
    \end{minipage}    
    \begin{minipage}[t!]{0.245\textwidth}
        \centering
        \includegraphics[width=0.95\linewidth]{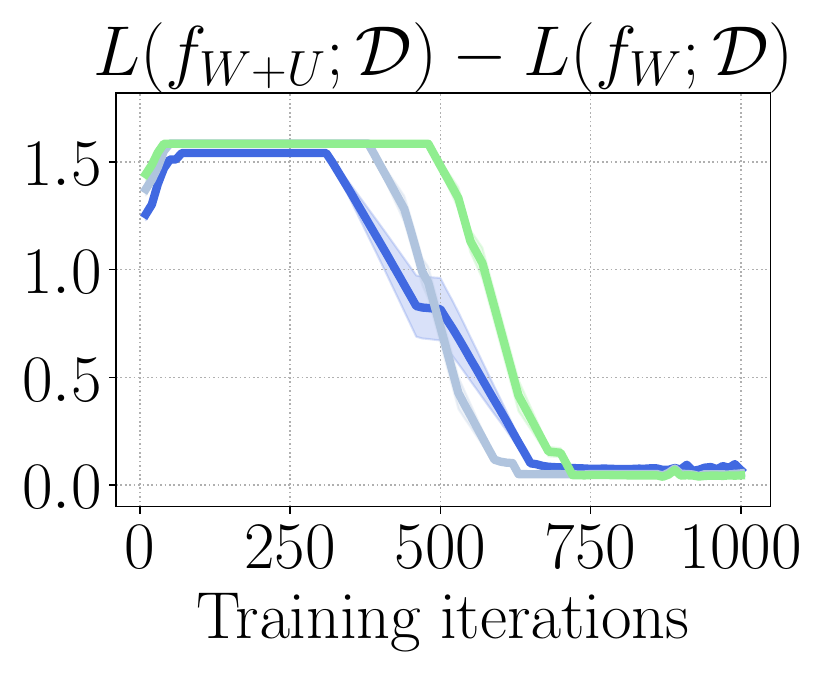}
        \subcaption{Quadratic functions}
    \end{minipage}
    \caption{We illustrate that, in training transformers on both linear and quadratic functions, the inference loss and the noise-perturbed loss converge close to zero, and become lower with longer steps. We set $d=10$ and train a one-head linear transformer on linear functions ($n=20$, $\eta=0.4$) and a three-head non-linear transformer on quadratic functions ($n=200$, $\eta=10^{-3}$).  For the random Gaussian noise, we set $\sigma = 0.002$. On the right, we consider a slightly different setting by first adding noise to the weights before computing the gradients.}\label{fig_simulation}
\end{figure*}

We describe our main results.
For linear functions $y = \theta^\top X$, we prove that the model exhibits inference loss of $O(1/\text{poly(d)})$ with $T = \Theta(\log d)$. This is because the transformer converges to a solution where $J_{i-1}^{(i)} \approx I - \frac{\eta}{n} X X^\top$ and $\rho_{i}^{(T)} \approx \norm{I - \frac{\eta}{n} X X^\top}^{T-i}$. 
The accumulated errors form a converging geometric series. We present the complete proof in Appendix~\ref{sec_proof_linear_transformer}.

Next, we present results on quadratic functions $y = (\theta^{\top} X)^{\odot 2}$, which represent a fundamental yet tractable setting for studying gradient descent dynamics \cite{li2018algorithmic,meterez2026defense}.
We first establish that a three-head transformer can exactly implement the gradient update step using cubic and squared activations. %

\begin{lemma}\label{theorem_construction_three_head} There exist weight matrices $\{W^{PV_h}, W^{KQ_h}\}_{h=1}^3$ and elementwise non-linear functions $\{\phi_h\}_{h=1}^3$, such that a three-head transformer $f_W(\widehat{Z}_{T-1})$ outputs the $T$-step gradient descent solution with a learning rate $\eta$ on the quadratic functions. 
\end{lemma}

We then show that the transformer model converges to this construction during training, yielding near-zero inference loss. 

\begin{theorem}\label{theorem_convergence_three_head}
Let $\kappa\in(0,1/24]$, $\mu\in(0, 4]$, and $\alpha_0\in[3/4,1)$ be the constants. Suppose $n = \Theta(d^2 \log^5 d)$, $\eta = \kappa/d$, $\theta_0 = \alpha_0\theta^\star$, and $T = \lceil C \log d \rceil$ for some $C > 0$ with $aC > 1$ for $a = \log\bigl(\tfrac{1}{1 - \mu \eta d}\bigr)$. There exists an initialization such that if we run gradient flow on the training loss, then after time $t = O\bigl(\log d + \log \tfrac{1}{\epsilon_0}\bigr)$, for any $\epsilon_0 \ge \Theta\bigl(\tfrac{\log d}{d^{2aC - 1}}\bigr)$, the trained transformer satisfies $L(f_W; \cD) \le \epsilon_1$ for any $\epsilon_1 \in \bigl(\Theta(d^{-2aC + 1} \log d),\, 1\bigr)$.
\end{theorem}

The proof tracks an ordinary differential equation for the eigenvalues of the weight matrices. In the converged solution, we have
\[ J_{i-1}^{(i)} \approx I - \frac{2\eta}{n} X D(\theta_i) X^\top \text{ and } \rho_{i}^{(T)} \approx \norm{I - \frac{2\eta}{n} X D(\theta_i) X^\top}^{T-i}, \]
where $D(\theta_i) = \mathrm{diag}(3(X^\top \theta_i)^{\odot 2} - (X^\top \theta^\star)^{\odot 2})$. Then, the accumulated errors again form a converging geometric series. We present the complete proof in Appendix~\ref{sec_proof_nonlinear_transformer}. 

\textbf{Comparison with prior work.} Prior analyses of transformers trained on multi-step reasoning differ from ours in their attention mechanisms, tasks, and tracked quantities. Previous works have examined one-head softmax attention on state-tracking tasks to demonstrate implicit biases and length generalization \cite{huang2025howtransformers,huang2025transformers_length}. \citet{huang2025transformers} study one-head linear attention on in-context linear regression, where the trained transformer implements multi-step gradient descent. \citet{yang2025multi} study multi-head softmax attention on tree path-finding, where different heads specialize and coordinate across subtasks. In contrast, we study multi-head non-linear attention on quadratic regression. We identify the explicit form of Jacobians taken through the inputs at each step at convergence and connect it with the stability of the inference loss.

\section{Proposed Methods}\label{sec_algorithm_design}

Our analysis suggests two strategies to mitigate accumulated errors: reducing the sequence length to limit the error amplification factor $\sum_{t=1}^{T-1}\rho^{(T)}_t$ and regularizing the input Jacobian spectral norms $\norm{J^{(t)}_{i}}$. Motivated by these insights, we design a training method combining two techniques.

\textbf{(1) Chain-of-thought length compression.} First, we compress the reasoning sequence by subsampling intermediate steps. For a training example $(\mathbf{x}, y^{(1)}, \dots, y^{(T)})$ and a sampling factor $\lambda \in (0, 1]$, we uniformly sample $T^\prime = \lfloor \lambda \cdot T \rfloor$ intermediate steps to form a compressed sequence $(\mathbf{x}, y^{(1)}, \dots, y^{(T^\prime)})$, and minimize the training loss over the sampled sequence.

As shown in Figures \ref{fig_downsampled_stpes_jacobian_dijkstra} and \ref{fig_downsampled_stpes_jacobian_prim}, models trained on subsampled sequences consistently outperform those trained on full steps. 
Intermediate numbers of steps achieve the lowest inference loss on the Dijkstra and Prim tasks. Notably, subsampling steps reduces the error amplification factor by 92$\times$ and 181$\times$ compared to training with all steps.
Subsampling steps reduces the amplification factor but increases the per-step errors $\varepsilon_t$. Thus, the best $\lambda$ is the one that balances these two effects.

\textbf{(2) Regularization through quantization-aware training.}
Second, we incorporate quantization during supervised fine-tuning to explicitly regularize the Jacobian spectral norms. This is inspired by observations that quantized training smooths the loss landscape and provides a regularization effect \cite{kwun2025lotion,li2026winq}. During the forward pass at training step $t$, we apply a standard uniform quantization function to the full-precision weights $W^{(t)}$:
\begin{align*}
    Q(W^{(t)}) = \Delta \cdot \text{Round}\left( {\text{Clamp}(W^{(t)}-b, W_{min}, W_{max})}/{\Delta} \right) + b,
\end{align*}
where $\Delta$ and $b$ are bit-width-dependent scaling and bias parameters. We use the straight-through estimator \cite{bengio2013estimating} to approximate the gradient $\nabla_W \hat{L}(f_{Q(W)}) \approx \nabla_{Q(W)} \hat{L}(f_{Q({W}})$.

We hypothesize that quantization-aware training acts as a regularizer on the error amplification factor. 
Under the Pseudo-Quantization Noise model \cite{widrow1996statistical, gholami2022survey}, the quantized model loss can be viewed as the expected loss under mean-zero randomized-rounding noise.
Minimizing this loss penalizes the trace of the loss Hessian, which can be decomposed into the Jacobians taken through the input space. 

Figures \ref{fig_low_precision_stpes_jacobian_dijkstra} and \ref{fig_low_precision_stpes_jacobian_prim} demonstrate this regularization effect. Training with lower bit-widths yields a lower error amplification factor compared to full precision. When combining 1-bit quantization, $\sum_{t=1}^{T-1}\rho_{t}^{(T)}$ is further reduced by $3.6\times$ and $5.8\times$, compared to full precision, on the Dijkstra and Prim algorithmic tasks, respectively.
We also observe that, while quantization increases the prediction error at each step, it reduces the error amplification factor much more significantly, thus leading to improved performance. We report evaluations under various bit-widths in Appendix \ref{sec_additional_experiments}.
Taken together, we summarize our approach in Algorithm \ref{alg_quant_cot}.

\begin{figure*}[t!]
\centering
    \begin{minipage}[t!]{0.245\textwidth}
        \centering
        \includegraphics[width=\linewidth]{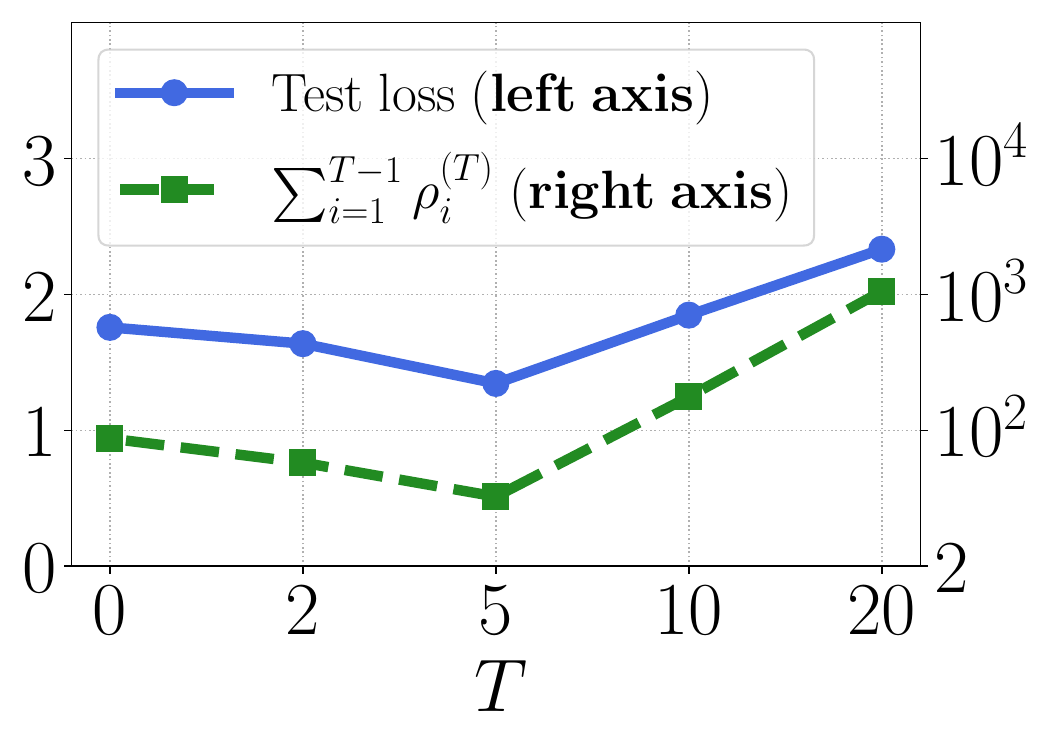}
    \subcaption{Dijkstra}\label{fig_downsampled_stpes_jacobian_dijkstra}
    \end{minipage}
    \begin{minipage}[t!]{0.245\textwidth}
        \centering
        \includegraphics[width=\linewidth]{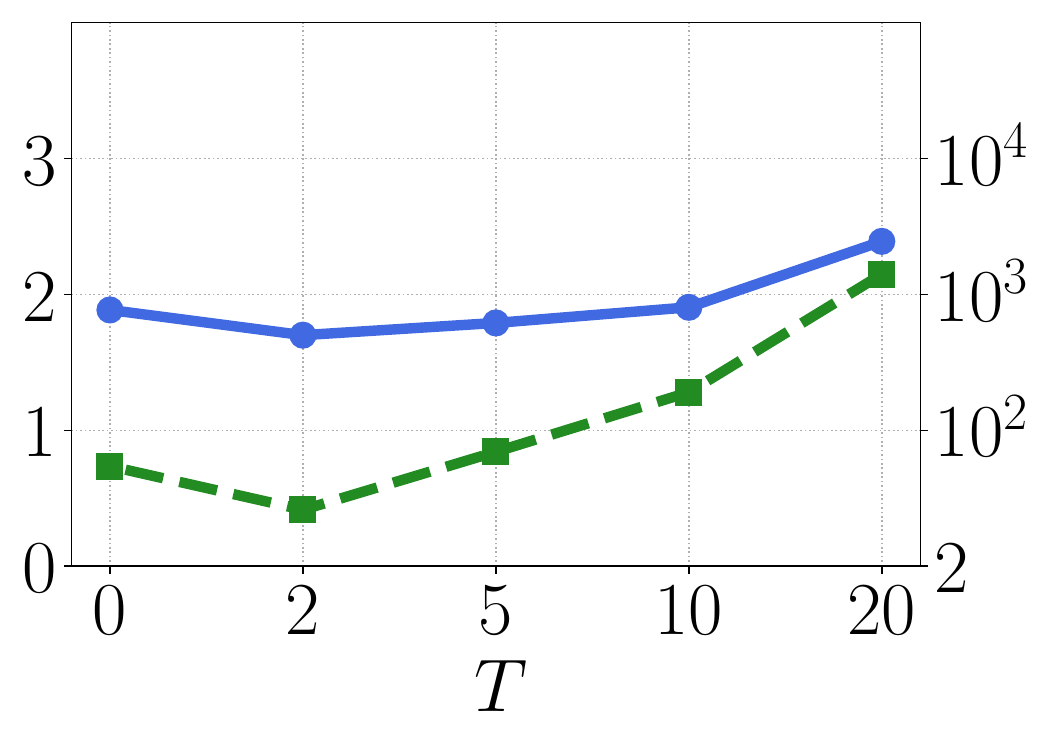}
    \subcaption{Prim's algorithm}\label{fig_downsampled_stpes_jacobian_prim}
    \end{minipage}
    \begin{minipage}[t!]{0.245\textwidth}
        \centering
        \includegraphics[width=\linewidth]{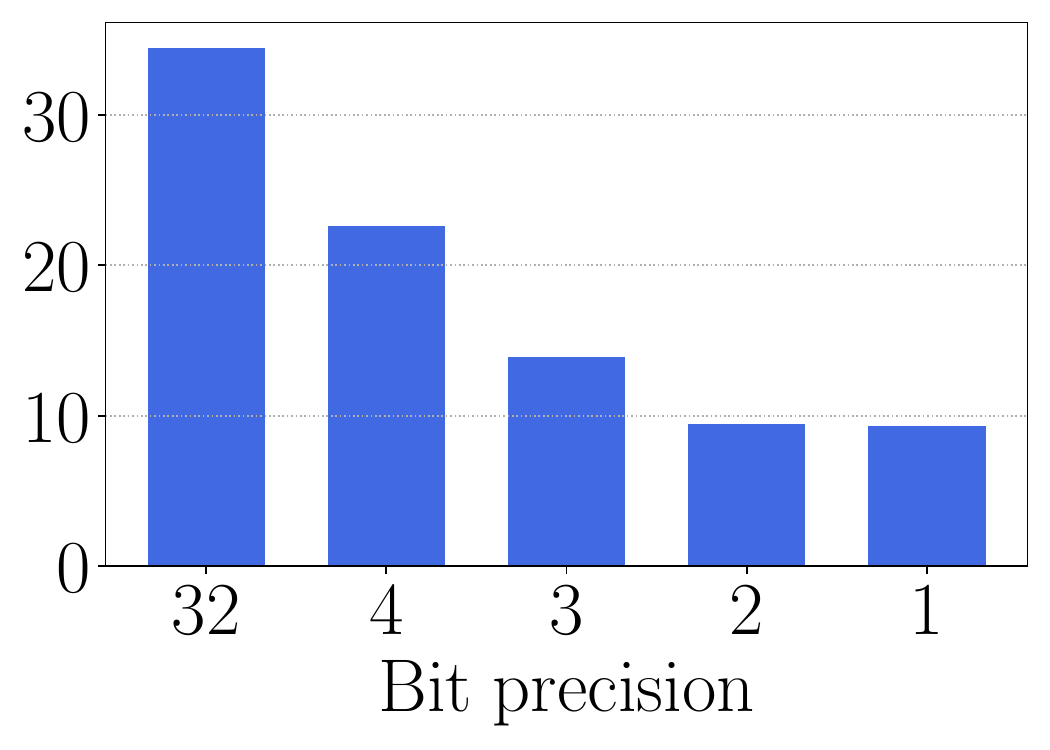}
    \subcaption{Dijkstra}\label{fig_low_precision_stpes_jacobian_dijkstra}
    \end{minipage}
    \begin{minipage}[t!]{0.245\textwidth}
        \centering
        \includegraphics[width=\linewidth]{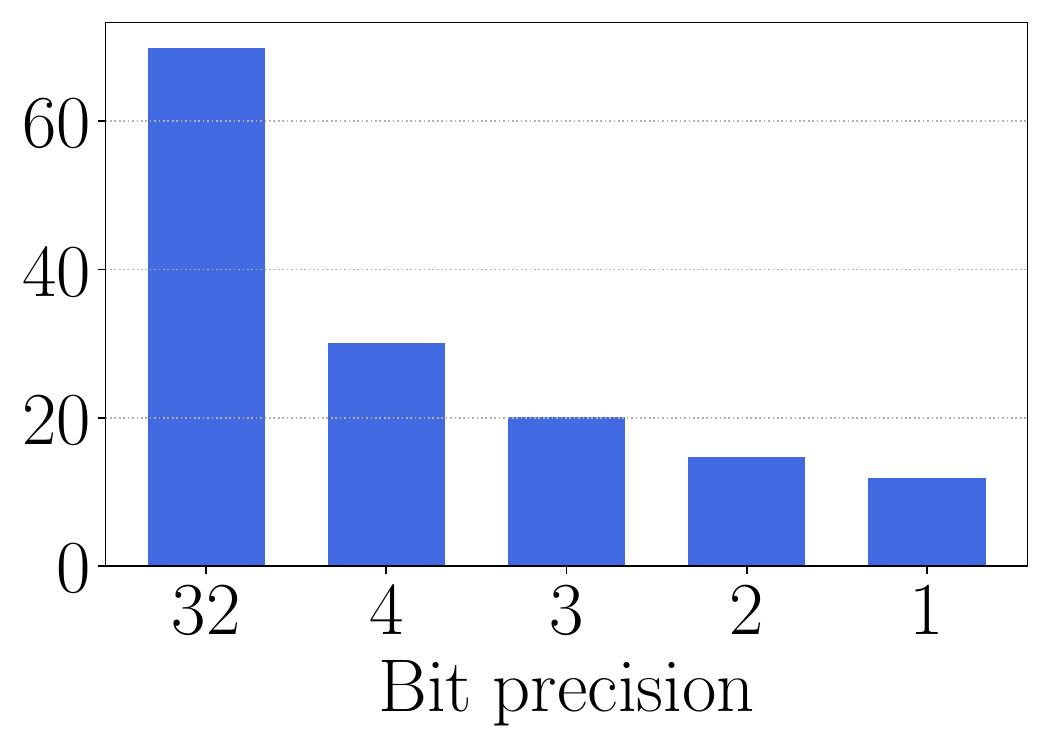}
    \subcaption{Prim's algorithm}\label{fig_low_precision_stpes_jacobian_prim}
    \end{minipage}
    \caption{Figure \ref{fig_downsampled_stpes_jacobian_dijkstra} and \ref{fig_downsampled_stpes_jacobian_prim}: We show that the models trained with reduced intermediate steps yield lower inference losses than the model trained with longer steps. Correspondingly, the models exhibit lower Jacobian-based error amplification factors. 
    Figure \ref{fig_low_precision_stpes_jacobian_dijkstra} and \ref{fig_low_precision_stpes_jacobian_prim}: We show that quantization-aware training with lower bit-widths further reduces the Jacobian-based error amplification factors, yielding lower inference losses.}
\end{figure*}

\begin{remark}[Length generalization]
For out-of-distribution settings, the test error bound from Lemma~\ref{thm_test_error_bound} can increase in two aspects. First, the per-step prediction errors grow due to covariate shift. Second, the error amplification factor accumulates over the additional steps. If the step-to-step Jacobian norm exceeds one, the extra steps multiply the error geometrically, making our regularization approach crucial. Designing training algorithms for length generalization is an interesting direction for future work, and our analysis lays a concrete direction for this. 
\end{remark}

\begin{algorithm}[!t]
	\caption{Jacobian-Regularized Training through Length Compression and Quantization}\label{alg_quant_cot}
    \textbf{Input}: {$n$ training examples with correct sequences $(\mathbf{x}_i, y_i^{(1)}, \dots, y_i^{(T)})$ of length $T$} \\
	\textbf{Require}: Pretrained language model $f_{W^{(0)}}$, number of epochs $S$, learning rate $\eta$, a sampling factor $\lambda$, quantization function $Q(\cdot)$ \\
	\textbf{Output}: Fine-tuned quantized model $f_{Q(W^{(S)})}$
		\begin{algorithmic}[1]
			\FOR{$t=0, 1, \dots, S-1$}
            \STATE {$(y_{i}^{(1)}, \dots, y_{i}^{(T^\prime)}) \leftarrow $ Subsample $T^\prime = \lfloor \lambda \cdot T \rfloor$ steps from example $i$, for $i = 1, \dots, n$}
			\STATE {$\hat{L}(f_{Q(W^{(t)})}) \leftarrow$ Compute the training loss of $f_{Q(W^{(t)})}$} on the sampled sequences
			\STATE {$W^{(t+1)} \leftarrow W^{(t)} - \eta \cdot \nabla_{Q(W^{(t)})} \hat{L}(f_{Q(W^{(t)})})$} \hfill 
			\ENDFOR
            \STATE \textbf{return} $f_{Q(W^{(S)})}$
		\end{algorithmic}
\end{algorithm}

\section{Experiments}
\label{sec:experiments}

We evaluate our proposed method in a range of algorithmic reasoning tasks. 
First, our method outperforms baselines by an average of \textbf{3.5}\% across five graph algorithmic and two symbolic reasoning tasks. Second, our method significantly enhances out-of-distribution length generalization, yielding an average \textbf{8.2}\% improvement. Third, empirical analysis confirms that our approach reduces the error amplification factor (measured via Jacobian spectral norms) by up to \textbf{8}$\times$. Lastly, we also conduct ablation studies for various design choices of our algorithm. 

\subsection{Experimental Setup}

\textbf{Datasets and models.} First, we evaluate our algorithm on graph algorithmic tasks from the CLRS-Text benchmark \cite{li2025efficiently}, which encodes graphs as flattened adjacency lists and represents the
intermediate steps as lists of node labels. We use five classic graph algorithms, including Dijkstra, Bellman-Ford, Prim's algorithm, breadth-first search, and depth-first search. For each task, we use 2,000 training examples generated on random graphs of 10 nodes and 200 examples for validation and testing. 
Second, we evaluate on symbolic state-tracking reasoning tasks, namely LEGO \cite{zhang2022unveiling}. 
We use two task variants: transitive (cyclic) and symmetry group actions.
For each task, we use a dataset of 2,000 training examples with $k=8$ variables, and 200 examples each for validation and testing. We evaluate our method in fine-tuning the QWen-1.5B models.

\textbf{Baselines.} We compare our method with three categories of baselines. First, we consider standard fine-tuning baselines that train without or with the provided chain-of-thought sequences, namely No CoT and SFT-CoT. Second, we consider existing chain-of-thought methods. This includes implicit CoT \cite{deng2024explicit} that trains a model to directly predict the final output by gradually removing intermediate steps during training, and the chain of continuous thought (Coconut) \cite{hao2025training} that trains a model to reason in a continuous latent space by using the hidden state as the next input embedding. 
Third, we consider noise stability optimization (NSO)~\cite{zhang2024noise} that regularizes the Hessian trace through random Gaussian noise injection. 
For each baseline, we evaluate the test accuracy on the final answer. 

\textbf{Implementations.} Our algorithm includes key hyperparameters: the sampling factor $\lambda$ and the quantization bit-width. We vary $\lambda$ across 0.1, 0.2, 0.5, and 1. The bit-width is varied across 1, 2, 3, and 4 bits. We use the quantization methods from existing methods \cite{li2026winq}.  
For the base training method, we use LoRA fine-tuning with a rank between 16 and 64. The training uses the AdamW optimizer with a learning rate in the range of $1 \times 10^{-5}$ to $4 \times 10^{-5}$. We report the hyperparameters used for each result in Appendix \ref{sec_additional_experiments}.

For the length generalization experiment, we additionally evaluate at $20\%$ longer inputs on all five CLRS-Text tasks. Absolute accuracy degrades sharply for every method, falling below $25\%$. While our method still leads, at this accuracy level, the input-length bottleneck dominates. We therefore report the $10\%$ setting as our main result. Further improving the length generalization on longer sequences is left for future work.

\begin{table}[t!]
\centering
\caption{We report the test accuracy (\%) of our method, as compared to supervised fine-tuning, implicit reasoning, and random noise injection. We evaluate five algorithmic reasoning tasks and two symbolic state-tracking reasoning tasks. For length generalization, models are evaluated on inputs 10\% longer than the training sequences. All results are averaged over three random seeds.
}\label{tab_results}
\resizebox{1.0\columnwidth}{!}
{
\begin{tabular}{@{}lccccccc@{}}
\toprule
 & Dijkstra & Bellman-Ford & Prim & BFS & DFS & Cyclic & Symmetry \\ \midrule
No CoT & 53.3 $\pm$ 2.0 & 69.9 $\pm$ 0.3 & 56.2 $\pm$ 0.1 & 85.9 $\pm$ 0.1 & 71.4 $\pm$ 1.4 & 66.8 $\pm$ 0.9 & 48.5 $\pm$ 3.1 \\
SFT-CoT & 73.1 $\pm$ 1.4 & 75.3 $\pm$ 1.5 & 62.3 $\pm$ 0.1 & 86.0 $\pm$ 0.4 & 54.6 $\pm$ 2.0 & 89.2 $\pm$ 1.2 & 55.8 $\pm$ 1.5\\
Implicit CoT & 71.6 $\pm$ 0.5 & 72.3 $\pm$ 0.6 & 64.1 $\pm$ 0.3 & 85.5 $\pm$ 1.7 & 72.4 $\pm$ 0.4 & 75.4 $\pm$ 1.0 & 60.8 $\pm$ 3.6\\
Coconut & 74.1 $\pm$ 0.3 & 75.6 $\pm$ 0.5 & 63.5 $\pm$ 1.2 & 85.7 $\pm$ 0.3 & 72.6 $\pm$ 0.7 & 79.5 $\pm$ 0.1 & 53.9 $\pm$ 0.5 \\
NSO & 74.5 $\pm$ 0.8 & 75.3 $\pm$ 0.0 & 63.2 $\pm$ 0.3 & 86.5 $\pm$ 0.8 & 55.8 $\pm$ 0.1 & 89.3 $\pm$ 0.8 & 57.2 $\pm$ 0.9\\ 
{Algorithm \ref{alg_quant_cot}} & \textbf{76.2} $\pm$ 1.7 & \textbf{80.3} $\pm$ 0.1 & \textbf{66.2} $\pm$ 0.3 & \textbf{88.3} $\pm$ 1.3 & \textbf{74.4} $\pm$ 1.5 & \textbf{94.1} $\pm$ 1.3 & \textbf{62.2} $\pm$ 0.8\\
\midrule\midrule
Length generalization & Dijkstra & Bellman-Ford & Prim & BFS & DFS & Cyclic & Symmetry \\ \midrule
No CoT & 29.8 $\pm$ 0.4 & 42.9 $\pm$ 0.6 & 37.0 $\pm$ 0.5 & 48.5 $\pm$ 0.2 & 23.0 $\pm$ 0.5 & 55.4 $\pm$ 0.5 & 31.2 $\pm$ 0.7\\
SFT-CoT & 42.1 $\pm$ 0.4 & 47.2 $\pm$ 0.8 & 34.1 $\pm$ 0.6 & 56.2 $\pm$ 0.6 & 21.5 $\pm$ 0.2 & 80.0 $\pm$ 0.4 & 48.8 $\pm$ 0.3\\
Implicit CoT & 33.7 $\pm$ 0.3 & 46.5 $\pm$ 0.5 & 39.3 $\pm$ 0.4 & 54.6 $\pm$ 0.5 & 30.0 $\pm$ 0.3 & 82.5 $\pm$ 0.8 & 53.6 $\pm$ 1.0\\
Coconut & 38.2 $\pm$ 0.6 & 38.9 $\pm$ 0.4 & 37.6 $\pm$ 0.9 & 53.9 $\pm$ 0.9 & 32.6 $\pm$ 0.4 & 80.4 $\pm$ 0.5 & 49.4 $\pm$ 0.2\\
NSO & 36.4 $\pm$ 0.5 & 44.1 $\pm$ 0.8 & 38.4 $\pm$ 0.3 & 54.8 $\pm$ 0.2 & 28.2 $\pm$ 0.2 & 82.4 $\pm$ 0.4 & 51.2 $\pm$ 0.9\\
{Algorithm \ref{alg_quant_cot}} & \textbf{44.3} $\pm$ 0.8 & \textbf{50.1} $\pm$ 0.1 & \textbf{42.5} $\pm$ 0.6 & \textbf{59.6} $\pm$ 0.6 & \textbf{35.4} $\pm$ 0.8 & \textbf{87.2} $\pm$ 0.2 & \textbf{59.1} $\pm$ 0.7\\
\bottomrule
\end{tabular}
}
\end{table}

\subsection{Experimental Results}

\textbf{Improved multi-step reasoning}. Table \ref{tab_results} details the test accuracy of our fine-tuned models across seven reasoning datasets. Our method consistently outperforms all baselines, achieving an average improvement of \textbf{3.5}\% across five graph algorithmic and two symbolic state-tracking tasks. Specifically, by combining sequence compression with quantization-aware training, our algorithm surpasses latent-reasoning baselines (implicit reasoning and Coconut) by up to 6.5\%.

\textbf{Enhanced length generalization.} Next, we evaluate the trained models on examples with 10\% longer input length than training examples in the lower half of Table \ref{tab_results}. Our approach consistently exceeds baseline performance by an average of \textbf{8.2}\%, demonstrating its effectiveness in suppressing error amplification during extended generation.

\textbf{Reduction of the error amplification factor.} Finally, to explain these performance gains, we compare the empirical error amplification factors across methods. Our algorithm drastically reduces this factor by \textbf{3}--\textbf{8}$\times$ compared to baselines, validating our method's regularization on Jacobian spectral norms. Quantitatively, on Qwen-1.5B across the three tasks in Figure \ref{ref_error_amplification_factors_intro}, our method reduces $\sum_{i=1}^{T-1}\rho^{(T)}_{i}$ by $43\%$ relative to baselines, with a $72\%$ average reduction in the inference loss.

We further fine-tune Llama-1B and Gemma-2B on Bellman-Ford with SFT-CoT, Coconut, and Algorithm \ref{alg_quant_cot}, and measure the full decomposition of \eqref{eq_main_bound} for each. We also find that the ordering of methods by bound value matches their ordering by empirical loss gap on both models. Relative to Coconut, our method reduces the bound by $77\%$ and the empirical test loss gap by $58\%$ on average. The absolute scale of the bound differs sharply between architectures. We report the full results in Table \ref{tab_cross_model} of Appendix \ref{sec_additional_experiments}.

\textbf{Ablation studies.} We evaluate the individual components of our method through ablation studies. First, removing either length compression or quantization-aware training degrades performance by 4.3\% and 3.2\%, confirming both are essential. Second, we compare intermediate step sampling strategies for length compression. We consider random sampling, removing the first, last, or equally spaced steps. Among them, random sampling yields slightly better performance and is thus adopted. Finally, we find that injecting quantization noise during training noticeably outperforms standard random Gaussian noise by 9\%.

\section{Related Work}
\label{sec:related_work}

\textbf{Theoretical understanding of multi-step reasoning.} 
Studying the training dynamics of transformers is key to understanding how multi-step reasoning improves generalization. 
Prior work analyzes single-head transformers trained on parity-check tasks \cite{wen2024sparse, huang2025howtransformers}, %
and state tracking tasks \cite{huang2025transformers_length}, %
while \cite{yang2025multi} studies multi-head specialization on path-finding in binary trees. %
For the in-context weight prediction task, \cite{huang2025transformers} establishes convergence for linear targets using single-head linear transformers. We extend the analysis to nonlinear multi-head transformers and quadratic functions. %
Complementary theoretical analyses include entropy-calibration analyses of error accumulation in autoregressive generation~\cite{cao2025entropy} and information-theoretic frameworks that quantify per-step information gain in multi-step reasoning~\cite{ton2025understanding}. %

\citet{joshi2025theory} study learning general function classes of next-token generators with CoTs in the realizable setting, showing that the sample complexity scales with the VC dimension of the hypothesis class and $\log(T)$ of the CoT length $T$. \citet{altabaa2025cot} further improve the bound with a notion of CoT-information, which achieves faster rates than applying uniform convergence to training loss on chain-of-thought sequences. %

Prior work has also examined the expressiveness of transformers with chain-of-thought. \citet{li2024chain} show that CoT enables transformers to perform computations beyond the expressiveness limits of models without CoT. Expressiveness has been further characterized through circuit complexity under logarithmic, linear, and polynomial generation lengths \cite{merrill2023expressive}. Increasing CoT length has been shown to outperform parallel scaling of multiple chains on graph connectivity tasks \citep{mirtaheri2025let}.

\textbf{Generalization of language models.} Recent work derives generalization bounds for LLMs via compression and subspace training \cite{lotfi2024non, lotfi2022pac}, %
with tighter martingale-based extensions that explain empirical pretraining scaling laws~\citep{finzi2025compute}. %
A separate line of work derives PAC-Bayes bounds for fine-tuning that depend on the loss Hessian~\citep{ju2022robust, ju2023generalization}. %

\textbf{Optimization methods for multi-step reasoning.} Several recent works aim to improve multi-step reasoning beyond supervised fine-tuning. \citet{yao2025optimizing} frame this process as a latent-variable problem, minimizing gradient variance to efficiently allocate sampling budgets. A complementary line of work explores implicit reasoning. Implicit chain-of-thought \citep{deng2024explicit} gradually removes intermediate steps to internalize computation. Continuous chain-of-thought \citep{hao2025training} shifts reasoning into a continuous latent space. Semantic chain-of-thought \citep{he2025semcot} aligns implicit and explicit steps via contrastive learning. Empirically, \citet{yeo2025demystifying} investigate how long reasoning chains emerge in large language models, highlighting training computation and reward shaping as key factors. More broadly, multi-step reasoning traces also appear in experimental design settings \cite{xiong2024data,xiong2025automated}, where trajectory-level data is revealed over time.

\section{Conclusion}

This work provides an analysis of multi-step reasoning by bounding its inference errors via the product of input Jacobian spectral norms across generation steps. We prove that transformers trained for in-context weight prediction converge to a multi-head construction where this error amplification factor decays, achieving polynomially small test loss. We introduce a training method that explicitly controls this factor using reasoning length compression and quantization-aware training. Across seven algorithmic tasks, our approach consistently improves test accuracy and length generalization, with reduced empirical amplification factors in fine-tuned models.

\section*{Acknowledgement}

This work received financial support from NSF award IIS-2412008 and a startup fund from Northeastern University.
D. Li is supported in part by a JPMorgan Chase PhD Fellowship.
Z. Zhang and M. Duan received additional financial support from Khoury PhD fellowships.

{}

\begin{refcontext}[sorting=nyt]
\printbibliography
\end{refcontext}

\newpage
\appendix
\section{Complete Proofs}
This appendix provides complete proofs for all theoretical results presented in the main paper. We organize the material in the following structure:
\begin{itemize}
    \item   \textbf{Appendix~\ref{sec_proof_test_error_bound}} provides the full proof for Lemma~\ref{thm_test_error_bound}, establishing the error bound between the inference loss and the training loss in $T$-step generation. The proof is via a recursive expansion of the prediction error using first-order Taylor expansion. Then, we derive the bound via spectral norms of Jacobians taken through the input space.
    
    \item \textbf{Appendix~\ref{sec_proof_linear_transformer}} establishes the error bounds and noise stability analysis for linear self-attention transformers trained on the weight prediction task for linear functions. We begin with preliminaries that reformulate the training and inference losses, then present the complete proof of Theorem~\ref{theorem_noise_stability_linear_self_attention_formal}  (Section~\ref{sec_linear_main_proof}), where standard Gaussian concentration bounds are cited and applied directly.
    
    \item \textbf{Appendix~\ref{sec_proof_nonlinear_transformer}} extends the analysis to nonlinear self-attention transformers trained on quadratic functions. We first describe the problem setup, then provide the explicit construction of a three-head transformer that implements gradient descent on quadratic losses (Section~\ref{sec_construction_transformers_quadratic_function}). We derive the gradient dynamics under chain-of-thought training (Section~\ref{sec_nonlinear_gradient}) and prove the convergence in Theorem~\ref{theorem_convergence_three_head} (Section~\ref{sec_nonlinear_convergence}).
\end{itemize}

\paragraph{Notations.} For any vector $v$, let $\norm{v}$ be the Euclidean norm of $v$. For any matrix $X\in\real^{m\times n}$, let $\normFro{X}$ be the Frobenius norm and $\bignorms{X}_2$ be the spectral norm of $X$. If $X$ is a square matrix, then the trace of $X$, $\bigtr{X}$, equals the sum of $X$'s diagonal entries.

For two functions $f(n)$ and $g(n)$, we write $g(n) = O(f(n))$ if there exists a fixed value $C$ that does not grow with $n$ such that $g(n) \le C \cdot f(n)$ when $n$ is large enough. 
We write $g(n) = \Theta(f(n))$ if there exists a fixed value $C_1$ and $C_2$ that do not grow with $n$ such that $C_1 f(n) \leq g(n) \le C_2 \cdot f(n)$ when $n$ is large enough. 

\subsection{Proof of Lemma~\ref{thm_test_error_bound}}\label{sec_proof_test_error_bound}

In this section, we provide the full proof for Lemma \ref{thm_test_error_bound}, which shows a bound on the gap between the inference loss and the true prediction loss conditioned on the ground truth.

For a given input $\mathbf{x} \sim \mathcal{D}$, let $y^{(t)}$ denote the ground-truth label at step $t$. We denote the model's predictions under two regimes:
\begin{enumerate}
    \item The prediction conditioned on the ground-truth preceding steps, as in the training loss, is denoted as $f_W(\mathbf{z}^{(t-1)})$.
    \item the prediction based on $t$-step autoregressive generation in the inference loss is denoted as $f_W(\widehat{\mathbf{z}}^{(t-1)})$.
\end{enumerate}
We define their deviations from the label $y^{(t)}$ at step $t$ as $e^{(t)}(\mathbf{x}) = f_W(\mathbf{z}^{(t-1)}) - y^{(t)}$, and $\widehat{e}^{(t)}(\mathbf{x}) = f_W(\widehat{\mathbf{z}}^{(t-1)}) - y^{(t)}$. The prediction gap caused by autoregressive accumulation is $$\Delta^{(t)}(\mathbf{x}) = f_W(\widehat{\mathbf{z}}^{(t-1)}(\mathbf{x})) - f_W(\mathbf{z}^{(t-1)}(\mathbf{x})).$$ By definition, we have the relation:
\begin{align*}
    \widehat{e}^{(t)}(\mathbf{x}) = e^{(t)}(\mathbf{x}) + \Delta^{(t)}(\mathbf{x}).
\end{align*}
For the Jacobian norms, we write $\norm{J^{(t)}_i} := \sup_{\mathbf{x} \in \cX^\ast}\norm{J^{(t)}_i(\mathbf{x})}$ for the norm with the supremum taken, so that $\rho^{(t)}_i = \norm{J^{(t)}_i} + \sum_{k=i+1}^{t-1}\norm{J^{(t)}_k}\,\rho^{(k)}_i$ with $\rho^{(t)}_t = 1$.
Next, we present the complete proof of Lemma \ref{thm_test_error_bound}, based on a recursive decomposition for $\Delta^{(t)}(\mathbf{x})$. 

\begin{proof}[Proof of Lemma \ref{thm_test_error_bound}]
First, we bound the loss gap via Lipschitz continuity. We evaluate the two losses at step $T$. By the assumption that the loss function $\ell(\cdot, \cdot)$ is $C_1$-Lipschitz continuous with respect to its first argument, the absolute gap between the inference loss and the reference loss for a specific input $\mathbf{x}$ is bounded by:
\begin{equation}
    \bigl| \ell\bigl(f_W(\widehat{\mathbf{z}}^{(T-1)}(\mathbf{x})), y^{(T)}\bigr) - \ell\bigl(f_W(\mathbf{z}^{(T-1)}(\mathbf{x})), y^{(T)}\bigr) \bigr| \; \leq \; C_1 \norm{\Delta^{(T)}(\mathbf{x})}.
    \label{eq_loss_bound}
\end{equation}
To bound the loss gap, we bound the norm of the final prediction gap $\norm{\Delta^{(T)}(\mathbf{x})}$.
We conduct a recursive expansion of the prediction gap $\norm{\Delta^{(T)}(\mathbf{x})}$. 
At the base case of $t=1$, both predictions rely solely on $\mathbf{x}$. Thus, $\Delta^{(1)}(\mathbf{x}) = 0$ and $\norm{\widehat{e}^{(1)}(\mathbf{x})} = \norm{e^{(1)}(\mathbf{x})}$.

For an arbitrary step $t > 1$, we apply first-order Taylor's expansion to expand the prediction on the generated prefix around the correct step sequence. Subtracting $f_W(\mathbf{z}^{(t-1)}(\mathbf{x}))$ from $f_W(\widehat{\mathbf{z}}^{(t-1)})$, we obtain the prediction gap:
\begin{equation}
    \Delta^{(t)}(\mathbf{x}) = \sum_{i=1}^{t-1} J^{(t)}_i(\mathbf{x}) \, \widehat{e}^{(i)}(\mathbf{x}) + \mathcal{R}^{(t)}(\mathbf{x}), \label{eq_deltat}
\end{equation}
where $J^{(t)}_i(\mathbf{x})$ is the step-to-step Jacobian evaluated at the ground-truth context, and $\mathcal{R}^{(t)}(\mathbf{x})$ is the higher-order Taylor remainder.

Under the assumption that the Jacobian of $f_W$ is $C_2$-Lipschitz continuous, the norm of the remainder $\mathcal{R}^{(t)}(\mathbf{x})$ is bounded by a quadratic function of the input deviation, based on the mean value theorem:
\begin{equation}\label{eq_sum_R}
    \norm{\mathcal{R}^{(t)}(\mathbf{x})} \; \leq \; \frac{C_2}{2} \Bigl( \sum_{i=1}^{t-1} \norm{\widehat{e}^{(i)}(\mathbf{x})} \Bigr)^2 \le \frac {C_2} 2 \left( \sum_{i=1}^{t-1} r_i \right)^2,
\end{equation}
based on \eqref{eq_per_step_levels} (with $\mathcal{R}^{(1)}(\mathbf{x}) = 0$).
Applying the triangle inequality to \eqref{eq_deltat}, we have:
\begin{equation}
    \norm{\Delta^{(t)}(\mathbf{x})} \; \leq \; \sum_{i=1}^{t-1} \norm{J^{(t)}_i} \norm{\widehat{e}^{(i)}(\mathbf{x})} + \norm{\mathcal{R}^{(t)}(\mathbf{x})}.
    \label{eq_delta_z_recursive}
\end{equation}
Using $\widehat{e}^{(t)}(\mathbf{x}) = e^{(t)}(\mathbf{x}) + \Delta^{(t)}(\mathbf{x})$ and applying the triangle inequality, the inference error norm is bounded by:
\begin{equation*}
    \norm{\widehat{e}^{(t)}(\mathbf{x})} \leq \norm{e^{(t)}(\mathbf{x})} + \norm{\Delta^{(t)}(\mathbf{x})}.
\end{equation*}
Substituting  \eqref{eq_delta_z_recursive} above yields the full recurrence relation for the inference error:
\begin{equation}
    \norm{\widehat{e}^{(t)}(\mathbf{x})} \; \leq \; s_t(\mathbf{x}) + \sum_{i=1}^{t-1} \norm{J^{(t)}_i} \norm{\widehat{e}^{(i)}(\mathbf{x})}, \text{ where } s_t(\mathbf{x}) := \norm{e^{(t)}(\mathbf{x})} + \norm{\mathcal{R}^{(t)}(\mathbf{x})}. \label{eq_delta_recursion}
\end{equation}
To solve this recurrence, we recursively expand $\norm{\widehat{e}^{(i)}(\mathbf{x})}$ in $s_i(\mathbf{x})$. The first-order components become a sum of all possible directed step-to-step paths. The product of the Jacobian spectral norms along these paths forms the cumulative amplification factor $\rho^{(t)}_i$, defined recursively as:
\begin{equation*}
    \rho^{(t)}_i = \sup_{\mathbf{x} \in \cX^\ast} \norm{J^{(t)}_i(\mathbf{x})} + \sum_{k=i+1}^{t-1} \left( \sup_{\mathbf{x} \in \cX^\ast} \norm{J^{(t)}_k(\mathbf{x})} \cdot \rho^{(k)}_i \right), \text{ with } \rho^{(t)}_t = 1.
\end{equation*}
Formally, we show that
\begin{align}
    \norm{\widehat{e}^{(t)}(\mathbf{x})} \; \leq \; \sum_{i=1}^{t} \rho^{(t)}_i \, s_i(\mathbf{x}), \text{ for every } t \geq 1 ,
    \label{eq_unrolled}
\end{align}
which holds at $t = 1$ by the base case. Assuming \eqref{eq_unrolled} up to step $t-1$, substituting it into \eqref{eq_delta_recursion}, we have
\begin{align}
    \norm{\widehat{e}^{(t)}(\mathbf{x})} & \leq s_t(\mathbf{x}) + \sum_{k=1}^{t-1}\bigl(\sum_{i=k}^{t-1}\norm{J^{(t)}_i}\rho^{(i)}_k\bigr)s_k(\mathbf{x}) \notag \\
    & \leq s_t(\mathbf{x}) + \sum_{k=1}^{t-1}\bigl(\norm{J^{(t)}_k} \, \rho^{(k)}_k + \sum_{i=k+1}^{t-1} \norm{J^{(t)}_i} \, \rho^{(i)}_k \bigr)s_k(\mathbf{x}) \notag  \\
    & \leq  s_t(\mathbf{x}) + \sum_{k=1}^{t-1} \rho^{(t)}_k s_k(\mathbf{x}), 
    \label{eq_rho_telescope}
\end{align}
which completes the induction.
By replacing $\rho^{(T)}_i$ and carrying each remainder to step $T$ by its own factor $\rho^{(T)}_t$, the final prediction gap at step $T$ is bounded by:
\begin{equation}
    \norm{\Delta^{(T)}(\mathbf{x})} \; \leq \; \sum_{i=1}^{T-1} \rho^{(T)}_i \norm{e^{(i)}(\mathbf{x})} \; + \; \sum_{t=2}^{T} \rho^{(T)}_t \norm{\mathcal{R}^{(t)}(\mathbf{x})}.
    \label{eq_final_delta_z}
\end{equation}
Substituting \eqref{eq_final_delta_z} into \eqref{eq_loss_bound}, we obtain:
\begin{align}
    & \bigl| \ell\bigl(f_W(\widehat{\mathbf{z}}^{(T-1)}(\mathbf{x})), y^{(T)}\bigr) - \ell\bigl(f_W(\mathbf{z}^{(T-1)}(\mathbf{x})), y^{(T)}\bigr) \bigr| \notag  \\
    \leq  &  C_1 \sum_{i=1}^{T-1} \rho^{(T)}_i \norm{e^{(i)}(\mathbf{x})} \; + \; C_1 \sum_{t=2}^{T} \rho^{(T)}_t \norm{\mathcal{R}^{(t)}(\mathbf{x})}. 
\end{align}
For every $\mathbf{x}$ in the support of $\cD$, recall from \eqref{eq_per_step_levels} that $\norm{e^{(i)}(\mathbf{x})} \le \varepsilon_i$ and $\norm{\mathcal{R}^{(t)}(\mathbf{x})} \le \frac{C_2}{2}\bigl(\sum_{i<t} r_i\bigr)^2$ from \eqref{eq_sum_R}, we have
\begin{align}
    & \bigl| \ell\bigl(f_W(\widehat{\mathbf{z}}^{(T-1)}(\mathbf{x})), y^{(T)}\bigr) - \ell\bigl(f_W(\mathbf{z}^{(T-1)}(\mathbf{x})), y^{(T)}\bigr) \bigr| \notag \\
    \leq  & C_1 \sum_{i=1}^{T-1} \rho^{(T)}_i \varepsilon_i \; + \; \frac{C_1 C_2}{2} \sum_{t=2}^{T} \rho^{(T)}_t \Bigl( \sum_{i=1}^{t-1} r_i \Bigr)^2. 
    \label{eq_main_bound_radius}
\end{align}
Taking the expectation over the data distribution $\mathbf{x} \sim \mathcal{D}$, we have
\begin{align}
    \bigl|L(f_W; \cD) - L^{\star}(f_W; \cD) \bigr|
    & =  \Bigl| \exarg{\mathbf{x} \sim \cD}{\ell\bigl(f_W(\widehat{\mathbf{z}}^{(T-1)}(\mathbf{x})), y^{(T)}\bigr) - \ell\bigl(f_W(\mathbf{z}^{(T-1)}(\mathbf{x})), y^{(T)}\bigr)} \Bigr| \notag \\
    & \leq  \exarg{\mathbf{x} \sim \cD}{\bigl| \ell\bigl(f_W(\widehat{\mathbf{z}}^{(T-1)}(\mathbf{x})), y^{(T)}\bigr) - \ell\bigl(f_W(\mathbf{z}^{(T-1)}(\mathbf{x})), y^{(T)}\bigr) \bigr|} \notag  \\
    & \leq  C_1 \sum_{i=1}^{T-1} \rho^{(T)}_i \varepsilon_i \; + \; \frac{C_1 C_2}{2} \sum_{t=2}^{T} \rho^{(T)}_t \Bigl( \sum_{i=1}^{t-1} r_i \Bigr)^2,
    \label{eq_expectation_chain}
\end{align}
which completes the proof of \eqref{eq_main_bound}.
\end{proof}

\subsection{Proof for the Linear Case}\label{sec_proof_linear_transformer}

We study the inference loss of a one-layer linear self-attention transformer trained to predict weight vectors of linear functions. 
Following the setup in Section~\ref{sec_training_dynamics}, we describe the in-context weight prediction task with linear functions. The weight vector is $\theta^\star\sim\mathcal{N}(0, \id_d)$, and the $n$ in-context examples are $x_j\sim\mathcal{N}(0, \id_d)$ with labels $y_j = \langle\theta^\star, x_j\rangle$ for $j\in[n]$, all independent. We write $X = [x_1,\dots,x_n]\in\real^{d\times n}$, $y = [y_1,\dots,y_n] = \theta^{\star\top}X$, and $S := \frac1n XX^\top$. The chain of thought consists of the gradient descent iterates on the squared loss on average over $n$ examples with learning rate $\eta$, started at $\theta_0 = 0_d$:
\begin{align}
\theta_{i+1} = \theta_i - \eta S(\theta_i - \theta^\star), \quad\text{so that}\quad \theta_i = \big(I - (I-\eta S)^i\big)\theta^\star \quad\text{for } i\ge 0 . \label{eq_linear_gd_iterates}
\end{align}
With $d_e = 2d+2$, the input that contains the correct steps $\theta_0,\dots,\theta_i$ is
\[
Z_i =
\begin{bmatrix}
X & 0_d & 0_d & \cdots & 0_d \\
y & 0 & 0 & \cdots & 0 \\
0_{d\times n} & \theta_0 & \theta_1 & \cdots & \theta_i \\
0_{1\times n} & 1 & 1 & \cdots & 1
\end{bmatrix}
\in \real^{d_e\times(n+i+1)} .
\]
The label of step $i+1$ is $\bar\theta_{i+1} := \theta_{i+1}$ for $i\le T-2$, and the label of the final step is $\bar\theta_T := \theta^\star$. 

We first rewrite the training and inference losses. During training, the network is trained to predict the next iterate given the preceding ground-truth sequences. We consider the gradient flow dynamics over the population loss: 
\begin{align*}
L^{\mathrm{CoT}}(f_W)
= \exarg{X, \theta^\star}
{\frac{1}{2}
\sum_{i=0}^{T-1}
\norm{f_{W}(Z_i) - (0_d, 0, \bar\theta_{i+1}, 1)}^2}.
\end{align*}

At inference time, we apply the model to generate $T$ steps before it outputs the final prediction. Starting from $Z_0$, we define the model prediction at each step as $\widehat Z_{i+1} = \big[ \widehat Z_i,  f_{W}(\widehat Z_i) \big]$, for $i = 0,1,\dots, T-1$, where the last column of $\widehat{Z}_i$ encodes the estimate $\widehat{\theta}_i$, and $\widehat\theta_0 = \theta_0$. Then, the inference loss, in \eqref{eq_test_loss}, with the squared loss, can be written as:
\begin{align}
L(f_W;\cD)  = \exarg{X, \theta^\star }
{\frac{1}{2}
\norm{f_{W}(\widehat{Z}_{T-1}) - (0_d, 0, \theta^{\star}, 1)}_2^2}. \label{eq_eval_loss_linear}
\end{align}

For linear functions, we consider a one-layer linear self-attention transformer with a residual connection. 
Given an input embedding matrix $Z \in \mathbb{R}^{d_e \times m}$ of length $m$, we denote
\begin{align}
f_W(Z) =  Z_{[:, -1]} + W^{PV} Z \cdot \frac{Z^\top W^{KQ} Z_{[:, -1]}}{n} ,
\end{align}
We denote $W^{PV} \in \mathbb{R}^{d_e \times d_e}$ as the weight matrix that merges the projection and value matrix, and $W^{KQ} \in \mathbb{R}^{d_e \times d_e}$ is the weight matrix that merges the key and query matrix. ${[:, -1]}$ indicates taking the last column of the matrix. $n$ is the number of in-context examples. 

Next, we state results from the previous paper~\cite{huang2025transformers}, which we will reuse in our statements. Under Assumption \ref{ass_init}, prior work~\cite{huang2025transformers} proves that, if we initialize the weight matrices with the block structure in Assumption \ref{ass_init}, the blocks that are initialized with $0$ will remain $0$ along the gradient flow trajectory. The initialization method is specified as follows to zero out the irrelevant blocks.

\begin{assumption}[Initialization]
\label{ass_init}
Let $\sigma_0 > 0$ be a constant. We assume the parameters are initialized as
\begin{align*}
W^{PV}(0) =
\begin{bmatrix}
0_{d\times d} & 0_{d\times 1} & 0_{d\times d} & 0_{d\times 1} \\
0_{1\times d} & 0_{1\times 1} & 0_{1\times d} & 0_{1\times 1} \\
W^{PV}_{31}(0) & 0_{d\times 1} & 0_{d\times d} & 0_{d\times 1} \\
0_{1\times d} & 0_{1\times 1} & 0_{1\times d} & 0_{1\times 1}
\end{bmatrix}, 
W^{KQ}(0) =
\begin{bmatrix}
0_{d\times d} & 0_{d\times 1} & W^{KQ}_{13}(0) & 0_{d\times 1} \\
0_{1\times d} & 0_{1\times 1} & 0_{1\times d} & W^{KQ}_{24} \\
0_{d\times d} & 0_{d\times 1} & 0_{d\times d} & 0_{d\times 1} \\
0_{1\times d} & 0_{1\times 1} & 0_{1\times d} & 0_{1\times 1}
\end{bmatrix}.
\end{align*}
Here, $W_{31}^{PV}(0) = \sum_{i=1}^d \lambda_i^{PV} u_i u_i^\top$ and $W_{13}^{KQ}(0) = \sum_{i=1}^d \lambda_i^{KQ} u_i u_i^\top$, where $\{u_i\}_{i=1}^d$ is an orthonormal basis. The matrices $W_{31}^{PV}(0)$ and $W_{13}^{KQ}(0)$ are symmetric and simultaneously diagonalizable, with eigenvalues satisfying $\lambda_i^{PV} \le -\sigma_0$, $\lambda_i^{KQ} \in [\sigma_0, 1/2]$, and for all $i \in [d]$. Finally, we fix $W^{KQ}_{24} = -1$ for all $t \geq 0$.
\end{assumption}

\begin{lemma}[Block-wise invariance under gradient flow~\cite{huang2025transformers}]
\label{lem_block_invariance}
Under Assumption~\ref{ass_init}, consider the gradient flow on the population CoT loss as
\[
\frac{d W}{dt} = - \nabla L^{\mathrm{CoT}}( W).
\]
Then, for all $t \ge 0$, the parameters $W^{PV}(t), W^{KQ}(t)$ take the following form:
\[
W^{PV}(t) =
\begin{bmatrix}
0_{d\times d} & 0_{d\times 1} & 0_{d\times d} & 0_{d\times 1} \\
0_{1\times d} & 0 & 0_{1\times d} & 0 \\
W^{PV}_{31}(t) & 0_{d\times 1} & 0_{d\times d} & 0_{d\times 1} \\
0_{1\times d} & 0 & 0_{1\times d} & 0
\end{bmatrix},
W^{KQ}(t) =
\begin{bmatrix}
0_{d\times d} & 0_{d\times 1} & W^{KQ}_{13}(t) & 0_{d\times 1} \\
0_{1\times d} & 0 & 0_{1\times d} & W^{KQ}_{24} \\
0_{d\times d} & 0_{d\times 1} & 0_{d\times d} & 0_{d\times 1} \\
0_{1\times d} & 0 & 0_{1\times d} & 0
\end{bmatrix}.
\]
\end{lemma}

Then, we can rewrite the linear self-attention model only in terms of $W^{PV}_{31}$ and $W^{KQ}_{13}$. 

\begin{definition}[Reduced transformer]\label{def_reformulation}
Define $S := \frac{1}{n}XX^\top$. Under Assumption \ref{ass_init}, for a given input $Z_i$, we can rewrite the linear self-attention model by
\[
f_W(Z_i) = \Big [0_d, 0,   \big( \theta_i + W^{PV}_{31} S\bigl(W^{KQ}_{13} \theta_i - \theta^\star\bigr) \big),   1 \Big]^\top.
\]
where $\theta_i$ is the last column of ${Z_i}_{[d+1: 2d+1, -1]}$ and $\theta^\star$ is the ground-truth weights of the linear function. Then, we can denote a reduced model by: 
\[
\tilde{f}_W(Z_i) =  \theta_i + W^{PV}_{31} S\bigl(W^{KQ}_{13} \theta_i - \theta^\star\bigr). 
\]
\end{definition}
The right-hand side depends on $Z_i$ only through its last column. The same formula holds for the inputs during inference $\widehat Z_i$. Below, we write $\tilde f_W(\theta) := \theta + W^{PV}_{31} S(W^{KQ}_{13} \theta - \theta^\star)$ for $\theta\in\real^d$, with $X$ and $\theta^\star$ fixed. This map is affine, and its Jacobian is given by
\begin{align}
J := \frac{\partial \tilde f_W(\theta)}{\partial\theta} = I + W^{PV}_{31} S W^{KQ}_{13}, \label{eq_linear_jacobian}
\end{align}
which does not depend on $\theta$.

\begin{theorem}[Global convergence~\cite{huang2025transformers}]\label{theorem_global_convergence}
Suppose $n = \Theta(d\log^5d)$, $\eta \in (0.1, 0.9)$, $T = \lceil c\log d \rceil$, and $c\log(\frac{1}{1-\eta}) > 2$. Under Assumption \ref{ass_init} with $\sigma_0 > \frac{3(1-\eta)}{(2-\eta)T}$, if we run gradient flow
on the population loss $L^{\mathrm{CoT}}(f_W)$, then after time $t = O(\log d + \log \frac 1 \epsilon)$, for any $\epsilon \geq \Theta(\frac{\log d}{d^{c \log(\frac{1}{1-\eta}) - 2}})$, we have
$L^{\mathrm{CoT}}(f_{W(t)}) \le \varepsilon$. 
Moreover, there exist matrices $A$ and $B$, with $\norm{A}_{\mathrm{op}}, \norm{B}_{\mathrm{op}} \leq \Theta(d^{-\frac{1}{2}c\log (\frac{1}{1-\eta})+\frac{1}{2}})$, such that $W^{PV}_{31}(t) = (-\eta I + A)$ and $W^{KQ}_{13}(t) = (I + B)$. 
\end{theorem}

We state Theorem \ref{theorem_noise_stability_linear_self_attention_formal}, which is the main result for the linear case. 

\begin{theorem}[Slightly adapted from~\citep{huang2025transformers}] \label{theorem_noise_stability_linear_self_attention_formal}
Suppose $n=\Theta(d\log^5 d)$, $\eta\in(0.1,0.9)$, and $T=\lceil c\log d\rceil$, with $c\log \big(\frac1{1-\eta}\big)>2$. Under the initialization in Assumption \ref{ass_init} with $\sigma_0 > \frac{3(1-\eta)}{(2-\eta)T}$, if we run gradient flow on the population loss $L^{\mathrm{CoT}}(f_W)$, after time $t = O(\log d + \log \frac{1}{\epsilon_0})$, for any $\epsilon_0 = \Theta(\frac{\log d}{d^{c \log(\frac{1}{1-\eta}) - 2}})$, we have $L(f_W; \cD) \leq \epsilon_1$, for any $\epsilon_1 \in \Theta(({d^{- c\log(\frac{1}{1-\eta}) +2}} \log d), 1 \big)$. %
\end{theorem}

\subsubsection{Proof of Theorem \ref{theorem_noise_stability_linear_self_attention_formal}}
\label{sec_linear_main_proof}

Next, we prove the bounds on inference loss and the noise stability measure of the linear self-attention model on the weight prediction task. Recall that, to evaluate the inference loss, we apply the model $f_W$ to generate $T$ steps before it outputs the final prediction. 
Starting from the input $Z_0$, we denote the model's generated output at step $i$ (also serves as the model input at step $i+1$) as follows:
\[
\widehat Z_{i+1}(W)
:=
\big[
\widehat Z_i (W), 
f_W(\widehat Z_i (W))
\big],
\qquad
i = 0,1,\dots,T-1, 
\]
where we denote $\widehat Z_0 := Z_0$. 
Then, we denote the last column of $\widehat Z_{i}(W)$ as 
$\widehat{z}_i(W) := \widehat Z_{i}(W)_{[:, -1]}$ and $\widehat{\theta}_i(W) := \widehat Z_{i}(W)_{[d+1: 2d+1, -1]}$. We can write down the block structure of $\widehat Z_i$ as 
\[
\widehat Z_i(W) = 
\begin{bmatrix}
X & 0_d & 0_d & \dots & 0_d \\
y & 0 & 0 & \dots & 0 \\
0_{d\times n} & \theta_0 & \widehat{\theta}_1(W) & \dots & \widehat{\theta}_i(W) \\
0_{1\times n}   & 1 & 1 & \dots & 1
\end{bmatrix}.
\]

Correspondingly, we denote the ground-truth input with $i$ intermediate steps as $Z_i$,  $\theta_i$ as ground-truth weight iterates, and we denote the final label by $z_T := (0_d, 0, \theta^\star, 1)$. When $W$ is fixed, we write $\widehat\theta_i$ for $\widehat\theta_i(W)$.

\begin{proof}[Proof of Theorem \ref{theorem_noise_stability_linear_self_attention_formal}]
We will first write the inference error as a sum of prediction errors at each step, each carried to the final step by a power of the Jacobian of the trained transformer. This is a case of Lemma~\ref{thm_test_error_bound}, where the Jacobian becomes a single matrix $J$, and the second-order remainder is zero. The inference loss is therefore small once $J$ is contractive. At convergence, we will show that $J$ is close to the Jacobian $I - \eta S$, whose norm $\norm{J}_2 < 1$, and the errors are small, which yields the rate for the inference loss. 

Let $W = W(t)$ be the trained weights at time $t$ given by Theorem~\ref{theorem_global_convergence}. By Definition~\ref{def_reformulation}, we have the simplified form of the transformer:
\[\widehat\theta_i = \tilde f_W(\widehat\theta_{i-1}) = \widehat{\theta}_{i-1} + W^{PV}_{31} S\bigl(W^{KQ}_{13} \widehat{\theta}_{i-1} - \theta^\star\bigr), \]
for $i = 1,\dots,T$, and $\widehat\theta_T$ is the $\theta$-block of the final output $f_W(\widehat Z_{T-1})$. Therefore, we can write 
\[L(f_W;\cD) = \frac12\exarg{X, \theta^\star}{\norm{\widehat\theta_T - \theta^\star}^2}.\]

We compare the generated estimates $\widehat\theta_i$ with the gradient descent iterates $\theta_i$. Let
\[
\Delta_i := \tilde f_W(\theta_{i-1}) - \theta_i ,\qquad i = 1,\dots,T ,
\]
be the prediction error of step $i$, when the model is given the correct previous step. 
Since $\tilde f_W$ is affine with the Jacobian $J$ in \eqref{eq_linear_jacobian},
\[
\widehat\theta_i - \theta_i = \big(\tilde f_W(\widehat\theta_{i-1}) - \tilde f_W(\theta_{i-1})\big) + \Delta_i = J(\widehat\theta_{i-1} - \theta_{i-1}) + \Delta_i.
\]
Recursively decomposing the above from $\widehat\theta_0 = \theta_0$, we have
\begin{align}
\widehat\theta_T - \theta^\star = (\theta_T - \theta^\star) + \sum_{i=1}^{T}J^{T-i}\Delta_i . \label{eq_linear_error_unrolled}
\end{align}
This equation is the form of Lemma~\ref{thm_test_error_bound} for the linear transformer. The first term, $\theta_T - \theta^\star$, is the error of $T$ gradient descent steps on the given in-context examples, because the final label is $\theta^\star$. Since $\norm{J}$ is not bounded uniformly over Gaussian inputs, we bound \eqref{eq_linear_error_unrolled} in expectation rather than with the worst-case constants of Lemma~\ref{thm_test_error_bound}. By the Cauchy--Schwarz inequality,
\begin{align}
L(f_W;\cD) \le \frac{T+1}{2}\Big(\exarg{X, \theta^\star}{\norm{\theta_T - \theta^\star}^2} + \sum_{i=1}^{T}\exarg{X, \theta^\star}{\norm{J^{T-i}\Delta_i}^2}\Big). \label{eq_linear_loss_split}
\end{align}
It remains to show that $J$ is contractive at convergence and that the residuals are small.

\noindent\textbf{(1) Upper bound of $\mathbb E {\norm{J}_2}$.}
Now, we analyze the spectral norm of the Jacobian matrix of the linear transformer, $J = I + W^{PV}_{31}SW^{KQ}_{13}$. Based on the results in Theorem \ref{theorem_global_convergence}, for the trained transformer by running the gradient flow on the population loss on the chain-of-thought sequences, there exists matrix $A$ and $B$, such that $W_{31} = (-\eta I + A)$ and $W_{13} = (I + B)$ with $\norm{A}_{\mathrm{op}}, \norm{B}_{\mathrm{op}} \leq \Theta(d^{-\frac{1}{2}c\log (\frac{1}{1-\eta})+\frac{1}{2}})$. 
We denote 
\[\beta := \Theta(d^{-\frac{1}{2}c\log (\frac{1}{1-\eta})+\frac{1}{2}}) \text{, and }  a := c\log(\frac{1}{1-\eta}) > 2.\] 
In particular, $\beta =  \Theta(d^{\frac{1-a}{2}}) \le d^{-1/2}$ for a large enough $d$.

Therefore, we have:
\begin{align*}
J & = I + W^{PV}_{31}SW^{KQ}_{13} \\
& = I + (- \eta I + A) S (I + B) \\
& = I - \eta S + (AS - \eta SB + ASB) \\
& = J_0 + E, 
\end{align*}
where $J_0 = I - \eta S$ and $E = (AS - \eta SB + ASB)$. 

Recall that $S = \frac{1}{n} X X^{\top}$ where $X = [x_1, \dots, x_n]$ with $x_i$ sampled from $\mathcal{N}(0, \id_d)$. By the Gaussian singular-value bound \cite{vershynin2012introduction} applied to $X$, given $n = \Theta(d^2 \log^5 d)$, the event $\mathcal{G} := \{\norm{S - I}_2\le C\sqrt{d/n}\}$ has probability at least $1 - 2e^{-d/2}$ for a constant $C$. On $\mathcal{G}$, we have that 
\[
\norm{J_0}_2 = \norm{I - \eta S}_2 \le 1 - \eta + C\eta\sqrt{d/n} = 1 - \eta + O(\log^{-5/2} d), \text{and } \norm{S}_2 \leq 2. 
\]

For matrix $E$, on $\mathcal{G}$, submultiplicativity gives
\[
\norm{E}_2\le\big(\norm{A}_2 + \eta\norm{B}_2 + \norm{A}_2\norm{B}_2\big)\norm{S}_2 = \Theta(d^{-c\log\frac{1}{1-\eta}+1}).
\]
Therefore, on $\mathcal{G}$,
\begin{align}
\norm{J}_2\le\norm{J_0}_2 + \norm{E}_2\le 1 - \eta + O(\log^{-5/2} d) + \Theta(d^{-c\log\frac{1}{1-\eta}+1}). \label{eq_linear_J_contraction}
\end{align}
Thus, $\norm{J}_2 \le 1 - \eta/2 < 1$ for a large enough $d$.

\noindent\textbf{(2) Upper bound of $\mathbb E {\norm{\Delta_i}^2}$.} 
We have the one-step gradient descent update as $\theta_i = \theta_{i-1} - \eta S(\theta_{i-1} - \theta^\star)$. Then, the prediction error of step $i$ conditioned on the correct steps is
\begin{align*}
\Delta_i & = \tilde{f}_W(\theta_{i-1}) - \theta_i \\
& = (W_{31}SW_{13} + \eta S)\,\theta_{i-1} - (W_{31} + \eta I)S\theta^\star = (AS - \eta SB + ASB)\,\theta_{i-1} - AS\,\theta^\star .    
\end{align*}
Therefore, we have
\begin{align}
\norm{\Delta_i}^2 & \leq 2 \big ( \norm{AS - \eta SB + ASB}_2^2 \norm{\theta_{i-1}}^2 + \norm{AS}_2^2 \norm{\theta^\star}^2 \big ).     \label{eq_norm_b_i}
\end{align}

Note that $\theta_i$ is the $i$-th iterate of the gradient descent applied to $(X, y)$. Given that $\theta^\star \sim \mathcal{N}(0, \id_d)$ and $S = \frac{1}{n} X X^T$, with $n = \Theta(d^2 \log^5 d)$, we have $\exarg{X, \theta^\star}{\norm{\theta_i}^2} \leq O(d)$. Therefore, based on $\theta_i = (I - (I-\eta S)^i) \theta^{\star}$, on $\mathcal{G}$, we have $\norm{I - \eta S}_2\le 1$, so $\norm{\theta_i}\le 2\norm{\theta^\star}$. 

On $\mathcal{G}$, we also have $\norm{AS - \eta SB + ASB}_2\le 6\beta$, and $\norm{AS}_2\le 2\beta$. 

Therefore, plug those back to \eqref{eq_norm_b_i}, for $i = 1, \dots, T$, we have:
\begin{align*}
\mathbb E \big [ \norm{{\Delta_i}}_2^2 \mathbf{1}_{\mathcal{G}} \big ] \le 2\big(36\beta^2\cdot 4d + 4\beta^2 d\big) = O(d\beta^2) = O(d^{2-a}).
\end{align*}

\noindent\textbf{Upper bound of ${L(f_W;\cD)}$.}
Now, we have shown that $\mathbb{E}[\norm{\Delta_i}^2\mathbf{1}_{\mathcal{G}}] = O(d^{2-a})$ for $i = 1,\dots,T$, and $\norm{J}_2 \le 1 - \eta/2$ on $\mathcal{G}$. 
Then, summing the geometric series gives
\[
\sum_{i=1}^{T}\mathbb{E}\big[\norm{J^{T-i}\Delta_i}^2\mathbf{1}_{\mathcal{G}}\big]\le\sum_{i=1}^{T}(1-\eta/2)^{2(T-i)}\,\mathbb{E}\big[\norm{\Delta_i}^2\mathbf{1}_{\mathcal{G}}\big]\le O(d^{2-a}) .
\]

In addition, the first term of \eqref{eq_linear_loss_split} is the error of exact gradient descent. By \eqref{eq_linear_gd_iterates}, $\theta_T - \theta^\star = -(I-\eta S)^T\theta^\star$. On $\mathcal{G}$, we have 
\[\norm{I - \eta S}_2^{2T}\le(1-\eta)^{2T}\big(1 + \frac{C\eta}{1-\eta}\sqrt{d/n}\big)^{2T}\le 2(1-\eta)^{2T}.\]
Since $T\ge c\log d$, we have
\[
\mathbb{E}\big[\norm{\theta_T - \theta^\star}^2\mathbf{1}_{\mathcal{G}}\big]\le 2(1-\eta)^{2T}\,\mathbb{E}\norm{\theta^\star}^2\le 2d^{1-2a} .
\]

Combining these results into {\eqref{eq_linear_loss_split}}, by the Cauchy--Schwarz inequality, we have 
\begin{align}
L(f_W;\cD)\le\frac{T+1}{2}\Big(2d^{1-2a} + O(d^{2-a}) + e^{-\Omega(d)}\Big) = O(d^{2-a}\log d) , \label{eq_linear_L_bound}
\end{align}
since $T+1 = O(\log d)$.

\end{proof}

\subsection{Proof for the Quadratic Case} \label{sec_proof_nonlinear_transformer}

In this section, we study an in-context weight prediction task where the model is trained to predict the parameters of a quadratic function from $n$ examples.
Let the quadratic weight vector be $\theta^\star \in \real^d$ where $\theta^\star \sim \mathcal{N}(0, \id_d)$. We sample $n$ in-context examples $x_i \sim \mathcal{N}(0, \id_d)$ with the corresponding labels according to $y_i = (\theta^{\star\top} x_i)^2$, for $i = 1, \dots, n$. The goal is to predict $\theta^\star$ given the $n$ examples.

Similar to weight prediction for the linear function, define the embedding dimension $d_e = 2d + 2$. The initial input sequence is constructed as
\begin{align}
Z_0
=
\begin{bmatrix}
x_1 & \cdots & x_n & 0 \\
y_1 & \cdots & y_n & 0 \\
0   & \cdots & 0   & \theta_0 \\
0   & \cdots & 0   & 1
\end{bmatrix}
:=
\begin{bmatrix}
X & 0_d \\
y & 0 \\
0_d & \theta_0 \\
0 & 1
\end{bmatrix}
\in \mathbb{R}^{d_e \times (n+1)},
\label{eq:Z0_quadratic}
\end{align}
where $X = [x_1, \dots, x_n] \in \mathbb{R}^{d \times n}$, $y = [y_1, \dots, y_n]^\top \in \mathbb{R}^n$, and $\theta_0 \in \mathbb{R}^d$ is an initial estimate in the basin of $\theta^\star$. We take $\theta_0 = \alpha_0\theta^\star$ for a constant $\alpha_0\in[3/4,1)$. We also append a dummy constant $1$ to each parameter token for formatting convenience.

We construct the chain-of-thought using gradient descent on the squared loss, i.e., $\frac{1}{2n}
\sum_{j=1}^n
\big( (\theta^\top x_j)^2 - y_j \big)^2$. 
Define a residual vector as $e(\theta) = \widehat y(\theta) - y$, where  $\widehat y_j(\theta) = (\theta^\top x_j)^2$. The gradient of the loss is
\[
\frac{2}{n}
\sum_{j=1}^n
\big( (\theta^\top x_j)^2 - y_j \big)
x_j x_j^\top \theta
=
\frac{2}{n}
X \operatorname{diag} \big(e(\theta)\big) X^\top \theta.
\]
Given a learning rate $\eta$, starting from $\theta_0$, define the gradient descent iterates for $i = 1, \dots, T$: 
\[
\theta_i
=
\theta_{i-1}
-
\eta \cdot
\frac{2}{n}
X \operatorname{diag} \big(e(\theta_{i-1})\big) X^\top \theta_{i-1} 
= \theta_{i-1} - \eta \cdot
\frac{2}{n}
X D_i^{+} D_i^{-} X^\top \theta_{i-1}, 
\]
where $D_i^{+} = \operatorname{diag}(X^\top\theta_{i-1} + X^\top\theta^\star)$ and $D_i^{-} = \operatorname{diag}(X^\top\theta_{i-1} - X^\top\theta^\star)$. 

Using these iterates, we define the ground-truth sequences
\begin{align}
Z_i
=
\begin{bmatrix}
x_1 & \cdots & x_n & 0 & 0 & \cdots & 0 \\
y_1 & \cdots & y_n & 0 & 0 & \cdots & 0 \\
0   & \cdots & 0   & \theta_0 & \theta_1 & \cdots & \theta_i \\
0   & \cdots & 0   & 1        & 1        & \cdots & 1
\end{bmatrix}
\in \mathbb{R}^{d_e \times (n+i+1)}, \text{ for } i = 1, 2, \dots, T. 
\label{eq:Zi_quadratic}
\end{align}

\subsubsection{Proof of Lemma \ref{theorem_construction_three_head}}\label{sec_construction_transformers_quadratic_function}

\begin{proof}
We now construct a three-head nonlinear self-attention transformer that performs one gradient descent step for the quadratic squared loss. For head $h\in\{1,2,3\}$, we denote a three-head self-attention model as
\begin{align*}
f_W(Z) = Z_{[:, -1]} + \sum_{i=1}^3 W^{PV_i}Z \cdot \phi_i \Big( \tfrac{Z^\top W^{KQ_i} Z_{[:, -1]}}{n} \Big).
\end{align*}
Here $W^{V_i},W^{K_i},W^{Q_i}\in\mathbb{R}^{d_i\times d_e}$ for $i = 1, 2, 3$. One can set $d_1, d_2$, and $d_3$ separately. 
We construct the transformer in the following way.  
Let $\phi(\cdot)$ be applied elementwise as 
\begin{align}
\phi_1(t)=t^{\odot 3},\qquad
\phi_2(t)=t^{\odot 2},\qquad
\phi_3(t)=t^{\odot 2}. \label{eq_construction_phis}
\end{align}
We then set
\begin{align}
& W^{PV_1} = \begin{bmatrix}
0_{d\times d} & 0_{d\times 1} & 0_{d\times d} & 0_{d\times 1} \\
0_{1\times d} & 0 & 0_{d\times 1} & 0 \\
-2\eta n^2 \id_d & 0_{d\times 1} & 0_{d\times d} & 0_{d\times 1} \\
0_{1\times d} & 0 & 0_{d\times 1} & 0
\end{bmatrix}, 
& W^{KQ_1} =
\begin{bmatrix}
0_{d\times d} & 0_{d\times 1} & \id_d & 0_{d\times 1} \\
0_{1\times d} & 0 & 0_{d\times 1} & 0 \\
0_{d\times d} & 0_{d\times 1} & 0_{d\times d} & 0_{d\times 1} \\
0_{1\times d} & 0 & 0_{d\times 1} & 0
\end{bmatrix}, \notag \\
& W^{PV_2} = \begin{bmatrix}
0_{d\times d} & 0_{d\times 1} & 0_{d\times d} & 0_{d\times 1} \\
0_{1\times d} & 0 & 0_{d\times 1} & 0 \\
\frac{1}{2} \eta n \id_d & 0_{d\times 1} & 0_{d\times d} & 0_{d\times 1} \\
0_{1\times d} & 0 & 0_{d\times 1} & 0
\end{bmatrix},  
& W^{KQ_2} =
\begin{bmatrix}
0_{d\times d} & 0_{d\times 1} & \id_d & 0_{d \times 1} \\
0_{1\times d} & 0 & 0_{d\times 1} & 1 \\
0_{d\times d} & 0_{d\times 1} & 0_{d\times d} & 0_{d\times 1} \\
0_{1\times d} & 0 & 0_{d\times 1} & 0
\end{bmatrix}, \notag \\
& W^{PV_3} = \begin{bmatrix}
0_{d\times d} & 0_{d\times 1} & 0_{d\times d} & 0_{d\times 1} \\
0_{1\times d} & 0 & 0_{d\times 1} & 0 \\
- \frac{1}{2} \eta n \id_d & 0_{d\times 1} & 0_{d\times d} & 0_{d\times 1} \\
0_{1\times d} & 0 & 0_{d\times 1} & 0
\end{bmatrix}, 
& W^{KQ_3} =
\begin{bmatrix}
0_{d\times d} & 0_{d\times 1} & \id_d & 0_{d \times 1} \\
0_{1\times d} & 0 & 0_{d\times 1} & -1 \\
0_{d\times d} & 0_{d\times 1} & 0_{d\times d} & 0_{d\times 1} \\
0_{1\times d} & 0 & 0_{d\times 1} & 0
\end{bmatrix}. \label{eq_construction_PV_and_KQs}
\end{align}
Now, we can inductively calculate the $i$-th step of generation, showing that the output is exactly the parameter after the $i$-th gradient step. 
\begin{align*}
    f_W(Z_i) =& (0_d, 0, \theta_i, 1) + \Big(0_d, 0,  - {2\eta n^2} X \cdot \phi_1 \big( \tfrac{Z^\top W^{KQ_1} Z_{[:, -1]}}{n}\big) \\
    &+  \frac{\eta n}{2}X \cdot \phi_2 \big( \tfrac{Z^\top W^{KQ_2} Z_{[:, -1]}}{n}\big) -  \frac{\eta n}{2}X \cdot \phi_3 \big( \tfrac{Z^\top W^{KQ_3} Z_{[:, -1]}}{n}\big), 0\Big).
\end{align*}
Based on the construction, we have the following: 
\begin{align*}
& X \cdot \phi_1 \big( {Z^\top W^{KQ_1} Z_{[:, -1]}}\big) = X (X^\top \theta_i)^{\odot 3} \\
& X \cdot \phi_2 \big( {Z^\top W^{KQ_2} Z_{[:, -1]}}\big) = X (X^\top \theta_i + y^\top)^{\odot 2} \\
& X \cdot \phi_3 \big( {Z^\top W^{KQ_3} Z_{[:, -1]}}\big) = X (X^\top \theta_i - y^\top)^{\odot 2}
\end{align*}
Therefore, we have
\begin{align*}
f_W(Z_i) & = (0_d, 0, \theta_i, 1) + 
(0_d, 0, -\frac{2\eta}{n} X \big( (X^\top \theta_i)^{\odot 3} - (X^\top \theta_i) \odot y^\top  \big), 0) \\
& = (0_d, 0, \theta_i -\frac{2\eta}{n} X  \operatorname{diag}\big(e(\theta_i)\big) X^\top \theta_i , 1) 
\end{align*}
where we denote the residual error $e(\theta) = \widehat y(\theta) - y$, where  $\widehat y_j(\theta) = (\theta^\top x_j)^2$. 

Then, after $T$ steps of applying $f_W$, we obtain the $T$-step gradient descent solution using $n$ examples of $(x, y)$. Denote the $\theta_{T}$ as the solution obtained by applying $T$-gradient step. Based on Lemma \ref{lem_gd_convergence_quadratic}, we have the test loss of $f_W$ as 
\begin{align*}
L_W = \frac{1}{2} \exarg{X, \theta^\star }
{
\left\|
f_{W}(\widehat{Z}_{T-1}) - (0_d, 0, \theta^{\star}, 1)
\right\|_2^2} = \frac{1}{2} \exarg{X, \theta^\star }
{
\left\|
\theta_{T} -  \theta^{\star}
\right\|_2^2} 
\leq O(d^{-2c\log(\frac{1}{1-\mu\eta d})+1}). 
\end{align*}
\end{proof}

\begin{lemma}\label{lem_gd_convergence_quadratic}
Let $\theta^\star\sim\mathcal N(0,\id_d)$ and $x_1,\dots,x_n\stackrel{\mathrm{i.i.d.}}{\sim}\mathcal N(0,\id_d)$.
Define $y_j=(x_j^\top\theta^\star)^2$ and the empirical loss
\[
\mathcal L_n(\theta)
:=\frac1{2n}\sum_{j=1}^n\Big((x_j^\top\theta)^2-y_j\Big)^2
=\frac1{2n}\sum_{j=1}^n\Big((x_j^\top\theta)^2-(x_j^\top\theta^\star)^2\Big)^2.
\]
Let $\mathrm{dist}(\theta,\{\pm\theta^\star\}) := \min\{\|\theta-\theta^\star\|_2,\ \|\theta+\theta^\star\|_2\}$.
Consider gradient descent $\theta_{t+1}=\theta_t-\eta\nabla\mathcal L_n(\theta_t), t=0,1,2,\dots,T$.
There exist absolute constants $0<m\le L<\infty$ and a radius $\rho\in(0,1)$ such that for all
$\theta$ satisfying $\mathrm{dist}(\theta,\{\pm\theta^\star\})\le \rho\|\theta^\star\|_2$,
the Hessian of $\mathcal L_n$ obeys
\[
m\|\theta^\star\|_2^2\id_d \ \preceq\  \nabla^2 \mathcal L_n(\theta)\ \preceq\ L\|\theta^\star\|_2^2\id_d.
\]
Suppose the initialization lies in the basin, i.e., $\mathrm{dist}(\theta_0,\{\pm\theta^\star\})\le \rho\|\theta^\star\|_2$.
Assume that $n\ge Cd^2\log^5 d$ for a sufficiently large constant $C>0$, $0<\eta\le \frac{1}{2L\|\theta^\star\|_2^2}$, and $T=\lceil c\log d \rfloor$. Then, with probability at least $1-d^{-10}$ over $(\theta^\star,X)$, the following holds
\begin{align*}
\mathrm{dist}(\theta_{T},\{\pm\theta^\star\})^2
\ \le\
\mathrm{dist}(\theta_0,\{\pm\theta^\star\})^2\cdot d^{2c\log\bigl(1-m\eta\|\theta^\star\|_2^2\bigr)}.
\end{align*}

Then, given $\eta=\Theta(1/d)$ and $T=\Theta(\log d)$, we have
\begin{align*}
\mathbb E_{X,\theta^\star} \left[\mathrm{dist}(\theta_{T},{\pm\theta^\star})^2\right]
\le
O \left(d^{1-2c\log(\frac{1}{1-m\eta d})}\right).
\end{align*}

\end{lemma}

\begin{proof}
We prove the claim on a high-probability event under which the empirical loss is uniformly strongly convex and smooth
in a neighborhood of $\{\pm\theta^\star\}$, and then apply a standard contraction argument for gradient descent.
Let $x\in\mathbb R^d$ and define
\[
\ell_x(\theta):=\frac12\big((x^\top\theta)^2-(x^\top\theta^\star)^2\big)^2.
\]
Then
\[
\nabla \ell_x(\theta)
=
\big((x^\top\theta)^2-(x^\top\theta^\star)^2\big)\nabla (x^\top\theta)^2
=
2\big((x^\top\theta)^2-(x^\top\theta^\star)^2\big)(x^\top\theta)x,
\]
and differentiating once more gives
\begin{align*}
\nabla^2 \ell_x(\theta)
&=
2\Big( \nabla\big((x^\top\theta)^2-(x^\top\theta^\star)^2\big)(x^\top\theta)
+\big((x^\top\theta)^2-(x^\top\theta^\star)^2\big)\nabla(x^\top\theta)\Big)x^\top \notag\\
&=
2\Big(2(x^\top\theta)x(x^\top\theta)+\big((x^\top\theta)^2-(x^\top\theta^\star)^2\big)x\Big)x^\top \notag\\
&=
2\Big(3(x^\top\theta)^2-(x^\top\theta^\star)^2\Big)xx^\top.
\label{eq:hess_single}
\end{align*}
Therefore, the empirical Hessian is
\begin{equation*}
\nabla^2\mathcal L_n(\theta)
=
\frac1n\sum_{j=1}^n \nabla^2\ell_{x_j}(\theta)
=
\frac{2}{n}\sum_{j=1}^n \Big(3(x_j^\top\theta)^2-(x_j^\top\theta^\star)^2\Big)x_jx_j^\top.
\end{equation*}

Let $\mathcal B := \{\theta\in\mathbb R^d:\ \mathrm{dist}(\theta,\{\pm\theta^\star\})\le \rho\|\theta^\star\|_2\}$.
We use a claim that when $n\ge C d^2\log^5 d$ (for $C$ sufficiently large), with probability at least $1-d^{-10}$ over $(\theta^\star,X)$, the following event holds:
\begin{equation}
\mathcal E:\quad
\forall \theta\in\mathcal B,\qquad
m\|\theta^\star\|_2^2 \id_d \preceq \nabla^2\mathcal L_n(\theta)\preceq L\|\theta^\star\|_2^2 \id_d,
\label{eq_event_E}
\end{equation}
for some absolute constants $0<m\le L<\infty$ and $\rho\in(0,1)$.
One can prove $\mathcal E$ by a standard $\varepsilon$-net argument over $\mathcal B$ together with Gaussian concentration for fourth-order polynomials. For example, see the local regularity condition established (cf. Section 7 of \cite{candes2015phase}), the empirical Hessian is uniformly well-conditioned in a neighborhood of $\pm \theta^\star$. 
We condition on $\mathcal E$ for the remainder of the proof.

Assume $\theta_0\in\mathcal B$. Since on $\mathcal E$ we have $\nabla^2\mathcal L_n(\theta)\preceq L\|\theta^\star\|_2^2 \id_d$ for all
$\theta\in\mathcal B$, the loss is $L\|\theta^\star\|_2^2$-smooth on $\mathcal B$. With stepsize $\eta\le 1/(2L\|\theta^\star\|_2^2)$, the standard descent lemma implies $\mathcal L_n(\theta_{t+1})\le
\mathcal L_n(\theta_t)$ as long as $\theta_t\in\mathcal B$.
Moreover, since $\mathcal L_n(\theta)$ is strongly convex in a neighborhood of each minimizer $\pm\theta^\star$ on
$\mathcal E$, and $\theta^\star$ is a minimizer with $\nabla\mathcal L_n(\theta^\star)=0$, the GD iterates remain in
the same local basin. Thus $\theta_t\in\mathcal B$ for all $t\le T$.

Next, we show one-step contraction in distance to $\{\pm\theta^\star\}$. Fix the sign $s\in\{+1,-1\}$ that minimizes $\|\theta_0-s\theta^\star\|_2$ and define
\[
\Delta_t := \theta_t - s\theta^\star.
\]
On $\mathcal E$, for all $\theta\in\mathcal B$ we have
\[
m\|\theta^\star\|_2^2 \id_d \preceq \nabla^2\mathcal L_n(\theta)\preceq L\|\theta^\star\|_2^2 \id_d.
\]
Since $\nabla\mathcal L_n(s\theta^\star)=0$ and $\mathcal L_n$ is twice differentiable, the mean value form of the
gradient gives
\[
\nabla\mathcal L_n(\theta_t)-\nabla\mathcal L_n(s\theta^\star)
=
\Big(\int_0^1 \nabla^2\mathcal L_n(s\theta^\star + \tau(\theta_t-s\theta^\star))d\tau\Big)(\theta_t-s\theta^\star)
=: H_t \Delta_t,
\]
where $H_t$ satisfies the same eigenvalue bounds:
\[
m\|\theta^\star\|_2^2 \id_d \preceq H_t \preceq L\|\theta^\star\|_2^2 \id_d.
\]
Therefore, the GD update can be written as
\[
\Delta_{t+1}
=
\theta_{t+1}-s\theta^\star
=
\theta_t-s\theta^\star - \eta \nabla\mathcal L_n(\theta_t)
=
\Delta_t - \eta H_t \Delta_t
=
(I-\eta H_t)\Delta_t.
\]
Taking norms and using $H_t\succeq m\|\theta^\star\|_2^2 \id_d$ yields
\[
\|\Delta_{t+1}\|_2
\le \|I-\eta H_t\|_{\mathrm{op}}\ \|\Delta_t\|_2
\le (1-\eta m\|\theta^\star\|_2^2)\|\Delta_t\|_2,
\]
where the last inequality uses $\eta\le 1/(2L\|\theta^\star\|_2^2)$ so that all eigenvalues of $I-\eta H_t$ lie in
$[0,1-\eta m\|\theta^\star\|_2^2]$.
Squaring gives the one-step contraction
\[
\|\Delta_{t+1}\|_2^2 \le (1-m\eta\|\theta^\star\|_2^2)^2\|\Delta_t\|_2^2.
\]

By induction,
\[
\|\Delta_{T-1}\|_2^2
\le
(1-m\eta\|\theta^\star\|_2^2)^{2T}\|\Delta_0\|_2^2.
\]
Since $\|\Delta_t\|_2=\mathrm{dist}(\theta_t,\{\pm\theta^\star\})$ by construction of $s$, and $T=c\log d$,
\begin{align*}
\mathrm{dist}(\theta_{T-1},\{\pm\theta^\star\})^2 &\le \mathrm{dist}(\theta_0,\{\pm\theta^\star\})^2\cdot (1-m\eta\|\theta^\star\|_2^2)^{2c\log d} \\
&= \mathrm{dist}(\theta_0,\{\pm\theta^\star\})^2\cdot d^{2c\log(1-m\eta\|\theta^\star\|_2^2)}.
\end{align*}
Finally, the probability statement follows since $\mathcal E$ holds with probability at least $1-d^{-10}$.
\end{proof}

\subsubsection{Gradient Computation}
\label{sec_nonlinear_gradient}

In this subsection, we compute the gradient of the full model and show that under certain initialization, running gradient flow on the training loss converges to the construction in Lemma \ref{theorem_construction_three_head}. 

Recall that we train a three-head nonlinear self-attention transformer, denoted as
\begin{align*}
f_W(Z) 
& = Z_{[:, -1]} + \sum_{i=1}^3 W^{PV_i}Z \cdot \phi_i \Big( \tfrac{Z^\top W^{KQ_i} Z_{[:, -1]}}{n} \Big). 
\end{align*}
Here, by the convention in prior work \cite{huang2025transformers,yang2025multi}, we merge $W^{P_i}$ and $W^{V_i}$ into a single matrix and merge $W^{K_i}$ and $W^{Q_i}$ into a single matrix. $W^{PV_i}$ and $W^{KQ_i}\in\mathbb{R}^{d_e\times d_e}$ for $i = 1, 2, 3$. By the construction, we set $\phi_i(\cdot)$ as elementwise non-linear functions  $\phi_1(t)=t^{\odot 3}, \phi_2(t)=t^{\odot 2}$, and $\phi_3(t)=t^{\odot 2}$. 

The chain-of-thought training objective is written as
\begin{align*}
L^{\mathrm{CoT}}(W)
= \exarg{x_i \sim \mathcal{N}(0, \id_d), \theta^\star\sim \mathcal{N}(0, \id_d)}
{\frac{1}{2}
\sum_{i=0}^{T-1}
\left\|
f_{W}(Z_i) - (0_d, 0, \bar\theta_{i+1}, 1)
\right\|_2^2}. 
\end{align*}
Here, $\bar\theta_{i+1} := \theta_{i+1}$ for $i\le T-2$ and $\bar\theta_T := \theta^\star$. As in the training loss \eqref{eq_training_loss}, there are $T$ steps. The label of the last step is the target $\theta^\star$. The reduced loss and all gradients below use this convention.

At inference time, we apply the model to generate $T-1$ steps before it outputs the final prediction.  The test loss is written as
\begin{align*}
L(W) = \exarg{X, \theta^\star }
{\frac{1}{2}
\left\|
f_{W}(\widehat{Z}_{T-1}) - (0_d, 0, \theta^{\star}, 1)
\right\|_2^2}. 
\end{align*}

Now, we compute the gradients of $W^{PV_i}$ and $W^{KQ_i}$ based on the loss of each CoT step. 

\begin{lemma}[Full model gradient]
The gradients of $W^{PV_i}$ and $W^{KQ_i}$ for $i = \set{1, 2, 3}$ are given by the following. 
\begin{align*}
& \nabla_{W^{PV_1}} L^{\mathrm{CoT}}(W) =  
\exarg{X, \theta^\star}{
\sum_{i=0}^{T-1} r_i \Big( Z_i \phi_1 \Big( \tfrac{Z_i^\top W^{KQ_1} Z_{i [:, -1]}}{n} \Big) \Big)^\top
} \\
& \nabla_{W^{PV_2}} L^{\mathrm{CoT}}(W) =  
\exarg{X, \theta^\star}{
\sum_{i=0}^{T-1} r_i \Big( Z_i \phi_2 \Big( \tfrac{Z_i^\top W^{KQ_2} Z_{i [:, -1]}}{n} \Big) \Big)^\top
} \\
& \nabla_{W^{PV_3}} L^{\mathrm{CoT}}(W) =  
\exarg{X, \theta^\star}{
\sum_{i=0}^{T-1} r_i  \Big ( Z_i \phi_3 \Big( \tfrac{Z_i^\top W^{KQ_3} Z_{i [:, -1]}}{n} \Big) \Big)^\top
} \\
& \nabla_{W^{KQ_1}} L^{\mathrm{CoT}}(W) = \exarg{X, \theta^\star}{
\frac{1}{n} 
\sum_{i=0}^{T-1} Z_i 
\Big( \big(Z_i^\top W^{PV_1 \top} r_i \big) \odot 3 (\tfrac{Z_i^\top W^{KQ_1} Z_{i [:, -1]}}{n})^{\odot 2} \Big) Z_{i [:, -1]}^\top }\\
& \nabla_{W^{KQ_2}} L^{\mathrm{CoT}}(W) = \exarg{X, \theta^\star}{
\frac{1}{n} 
\sum_{i=0}^{T-1} Z_i 
\Big( \big(Z_i^\top W^{PV_2 \top} r_i \big) \odot 2 (\tfrac{Z_i^\top W^{KQ_2} Z_{i [:, -1]}}{n}) \Big) Z_{i [:, -1]}^\top } \\
& \nabla_{W^{KQ_3}} L^{\mathrm{CoT}}(W) = \exarg{X, \theta^\star}{
\frac{1}{n} 
\sum_{i=0}^{T-1} Z_i 
\Big( \big(Z_i^\top W^{PV_3 \top} r_i \big) \odot 2 (\tfrac{Z_i^\top W^{KQ_3} Z_{i [:, -1]}}{n}) \Big) Z_{i [:, -1]}^\top }
\end{align*}    
where we define 
\[
r_i = \sum_{h=1}^3 W^{PV_h}Z_i \cdot \phi_h \Big( \tfrac{Z_i^\top W^{KQ_h} Z_{i [:, -1]}}{n} \Big) - (0_d, 0, \bar\theta_{i+1} - \theta_i, 0) 
\]
with $\bar\theta_{i+1}$ defined above.
\end{lemma}

\begin{proof}
For a step index $i = 0,1,\dots,T-1$, write $Z\in\mathbb{R}^{d_e\times m}$ with $m=n+i+1$ and $z=Z_{[:,-1]}\in\mathbb{R}^{d_e}$. We define
\[
s_h:=\frac{Z^\top W^{KQ_h} z}{n},
\qquad v_h:=\phi_h(s_h), 
\quad
u_h:=W^{PV_h}Z v_h, \text{ for } h = 1, 2, 3. 
\]

Then, we can write the squared loss in terms of the residual error in each step as
\[
\ell_i(W):=\frac12\norm{r_i}^2, \text{ where } r_i =\sum_{h=1}^3 u_h-(0_d,0,\theta_{i+1}-\theta_i,0),
\]

Then, we compute the gradients by the chain rule. For each head, the dependence of the loss on $W^{PV_h}$ is only through $u_h$. Hence, 
\[
\nabla_{W^{PV_h}}\ell_i(W)= r_i (Z v_h)^\top
= r_i\Big(Z \phi_h\Big(\tfrac{Z^\top W^{KQ_h} z}{n}\Big)\Big)^\top.
\]
Summing over $i$ and taking the expectation yields the stated formula. 

Then, we compute the gradient for $W^{KQ_h}$.  Since $v_h=\phi_h(s_h)$ is elementwise, we define
\[
q_h := \frac{\partial \ell_i}{\partial s_h}
=\Big(Z^\top (W^{PV_h})^\top r_i\Big)\odot \phi_h'(s_h).
\]
Using the definition of $s_h$, we have
\[
\nabla_{W^{KQ_h}}\ell_i(W)=\frac{1}{n}Z q_h z^\top
=\frac{1}{n}Z\Big(\big(Z^\top (W^{PV_h})^\top r_i\big)\odot \phi_h'(s_h)\Big)z^\top.
\]
Plugging in $\phi_1'(t)=3t^{\odot 2}$ and $\phi_2'(t)=\phi_3'(t)=2t$ gives exactly the gradients after summing over $i$ and taking the expectation.
\end{proof}

According to Lemma \ref{theorem_construction_three_head}, we use the following specific initialization to zero out the irrelevant blocks. 

\begin{assumption}[Initialization]
\label{ass_init_quad}
We assume the parameters are initialized as
\begin{align*}
& W^{PV_1}(0) = \begin{bmatrix}
0_{d\times d} & 0_{d\times 1} & 0_{d\times d} & 0_{d\times 1} \\
0_{1\times d} & 0 & 0_{d\times 1} & 0 \\
W^{PV_1}_{31}(0) & 0_{d\times 1} & 0_{d\times d} & 0_{d\times 1} \\
0_{1\times d} & 0 & 0_{d\times 1} & 0
\end{bmatrix}, 
W^{KQ_1}(0) =
\begin{bmatrix}
0_{d\times d} & 0_{d\times 1} & W^{KQ_1}_{13}(0) & W^{KQ_1}_{14} \\
0_{1\times d} & 0 & 0_{d\times 1} & W^{KQ_1}_{24} \\
0_{d\times d} & 0_{d\times 1} & 0_{d\times d} & 0_{d\times 1} \\
0_{1\times d} & 0 & 0_{d\times 1} & 0
\end{bmatrix} \\
& W^{PV_2}(0) = \begin{bmatrix}
0_{d\times d} & 0_{d\times 1} & 0_{d\times d} & 0_{d\times 1} \\
0_{1\times d} & 0 & 0_{d\times 1} & 0 \\
W^{PV_2}_{31}(0) & 0_{d\times 1} & 0_{d\times d} & 0_{d\times 1} \\
0_{1\times d} & 0 & 0_{d\times 1} & 0
\end{bmatrix},  
W^{KQ_2}(0) =
\begin{bmatrix}
0_{d\times d} & 0_{d\times 1} & W^{KQ_2}_{13}(0) & W^{KQ_2}_{14} \\
0_{1\times d} & 0 & 0_{d\times 1} & W^{KQ_2}_{24} \\
0_{d\times d} & 0_{d\times 1} & 0_{d\times d} & 0_{d\times 1} \\
0_{1\times d} & 0 & 0_{d\times 1} & 0
\end{bmatrix} \\
& W^{PV_3}(0) = \begin{bmatrix}
0_{d\times d} & 0_{d\times 1} & 0_{d\times d} & 0_{d\times 1} \\
0_{1\times d} & 0 & 0_{d\times 1} & 0 \\
W^{PV_3}_{31}(0) & 0_{d\times 1} & 0_{d\times d} & 0_{d\times 1} \\
0_{1\times d} & 0 & 0_{d\times 1} & 0
\end{bmatrix}, 
W^{KQ_3}(0) =
\begin{bmatrix}
0_{d\times d} & 0_{d\times 1} & W^{KQ_3}_{13}(0) & W^{KQ_3}_{14} \\
0_{1\times d} & 0 & 0_{d\times 1} & W^{KQ_3}_{24} \\
0_{d\times d} & 0_{d\times 1} & 0_{d\times d} & 0_{d\times 1} \\
0_{1\times d} & 0 & 0_{d\times 1} & 0
\end{bmatrix}
\end{align*}
Additionally, we impose a \emph{weight-sharing mechanism} that ties weights across $W^{PV_i}$ matrices and across $W^{KQ_i}$ matrices. Let two underlying matrices be $W^{PV}_{31}$ and $W^{KQ}_{13}$. We represent the matrices in the transformer with the two matrices as follows: 
\begin{align*}
W^{PV_1}_{31}(t) = -2\eta n^2 W^{PV}_{31}(t),\quad W^{PV_2}_{31}(t) = \frac{\eta n}{2} W^{PV}_{31}(t),\quad 
W^{PV_3}_{31}(t) = -\frac{\eta n}{2} W^{PV}_{31}(t) \\
W^{KQ_1}_{13}(t) = W^{KQ}_{13}(t),\quad
W^{KQ_2}_{13}(t) = W^{KQ}_{13}(t),\quad
W^{KQ_3}_{13}(t) = W^{KQ}_{13}(t).
\end{align*}
The weight matrices $W^{PV}_{31}(t)$ and $W^{KQ}_{13}(t)$ are updated with the combined gradient of the corresponding three weights. Further, $W_{31}^{PV}(0) = \sum_{i=1}^d \lambda_i^{PV} u_i u_i^\top$ and $W^{KQ}_{13}(0) = \sum_{i=1}^d \lambda_i^{KQ} u_i u_i^\top$ share the same orthonormal basis, which is symmetric and simultaneously
diagonalizable.  We fix $W^{KQ_1}_{24} = 0$, $W^{KQ_2}_{24} = 1$, $W^{KQ_3}_{24} = -1$, and $W^{KQ_1}_{14} = 0_d, W^{KQ_2}_{14} = 0_d, W^{KQ_3}_{14} = 0_d$ for all $t > 0$. 
We take the initialization to be isotropic, $\lambda^{PV}_i = \lambda_V(0)$ and $\lambda^{KQ}_i = \lambda_K(0)$ for all $i\in[d]$, where for a constant $\sigma\in(0,1/4)$,
\[
\lambda_K(0)\in\Big[\tfrac{1}{\sqrt3}+\sigma,\ 1\Big],\qquad \lambda_V(0)\in(0,1/\sigma],\qquad \lambda_V(0)^2\ \ge\ F\big(\lambda_K(0)^2\big)+\sigma^2,
\]
where $F(x):=\frac{x-1}{3}+\frac29\log\frac{2}{3x-1}$.
\end{assumption}

We consider the gradient flow trajectory: 
\begin{align*}
\frac{d W}{dt} = & - \nabla L^{\mathrm{CoT}}(W), \text{ where } \\
W := &  [{W^{KQ_1}}/\set{W^{KQ_1}_{24}, W^{KQ_1}_{14}}, {W^{KQ_2}}/\set{W^{KQ_2}_{24}, W^{KQ_2}_{14}}, {W^{KQ_3}}/\set{W^{KQ_3}_{24}, W^{KQ_3}_{14}}, \\
& {W^{PV_1}}, {W^{PV_2}}, {W^{PV_3}}].
\end{align*}

Next, we prove that if we initialize those matrices as in Assumption \ref{ass_init_quad}, then the zero blocks will stay at 0 along the gradient descent trajectory.

\begin{lemma} \label{lemma_gradient_structure_quadratic}
Under Assumption \ref{ass_init_quad}, when the transformer is trained under gradient flow, $\forall t > 0$,%
\begin{align*}
& W^{PV_1}(t) = \begin{bmatrix}
0_{d\times d} & 0_{d\times 1} & 0_{d\times d} & 0_{d\times 1} \\
0_{1\times d} & 0 & 0_{d\times 1} & 0 \\
W^{PV_1}_{31}(t) & 0_{d\times 1} & 0_{d\times d} & 0_{d\times 1} \\
0_{1\times d} & 0 & 0_{d\times 1} & 0
\end{bmatrix}, 
W^{KQ_1}(t) =
\begin{bmatrix}
0_{d\times d} & 0_{d\times 1} & W^{KQ_1}_{13}(t) & 0_{d\times 1} \\
0_{1\times d} & 0 & 0_{d\times 1} & 0 \\
0_{d\times d} & 0_{d\times 1} & 0_{d\times d} & 0_{d\times 1} \\
0_{1\times d} & 0 & 0_{d\times 1} & 0
\end{bmatrix} \\
& W^{PV_2}(t) = \begin{bmatrix}
0_{d\times d} & 0_{d\times 1} & 0_{d\times d} & 0_{d\times 1} \\
0_{1\times d} & 0 & 0_{d\times 1} & 0 \\
W^{PV_2}_{31}(t) & 0_{d\times 1} & 0_{d\times d} & 0_{d\times 1} \\
0_{1\times d} & 0 & 0_{d\times 1} & 0
\end{bmatrix},  
W^{KQ_2}(t) =
\begin{bmatrix}
0_{d\times d} & 0_{d\times 1} & W^{KQ_2}_{13}(t) & 0_{d\times 1} \\
0_{1\times d} & 0 & 0_{d\times 1} & 1 \\
0_{d\times d} & 0_{d\times 1} & 0_{d\times d} & 0_{d\times 1} \\
0_{1\times d} & 0 & 0_{d\times 1} & 0
\end{bmatrix} \\
& W^{PV_3}(t) = \begin{bmatrix}
0_{d\times d} & 0_{d\times 1} & 0_{d\times d} & 0_{d\times 1} \\
0_{1\times d} & 0 & 0_{d\times 1} & 0 \\
W^{PV_3}_{31}(t) & 0_{d\times 1} & 0_{d\times d} & 0_{d\times 1} \\
0_{1\times d} & 0 & 0_{d\times 1} & 0
\end{bmatrix}, 
W^{KQ_3}(t) =
\begin{bmatrix}
0_{d\times d} & 0_{d\times 1} & W^{KQ_3}_{13}(t) & 0_{d\times 1} \\
0_{1\times d} & 0 & 0_{d\times 1} & -1 \\
0_{d\times d} & 0_{d\times 1} & 0_{d\times d} & 0_{d\times 1} \\
0_{1\times d} & 0 & 0_{d\times 1} & 0
\end{bmatrix}
\end{align*}
\end{lemma}

\begin{proof}
We show that when the irrelevant blocks are 0, the gradients for those blocks are always 0:
\begin{align*}
& W^{PV_1}Z_i \cdot \phi_1 \Big( \tfrac{Z_i^\top W^{KQ_1} Z_{i[:, -1]}}{n} \Big) \\
= &  \frac{1}{n^3} \begin{bmatrix}
0_{d\times d} & 0_{d\times 1} & 0_{d\times d} & 0_{d\times 1} \\
0_{1\times d} & 0 & 0_{d\times 1} & 0 \\
W^{PV_1}_{31}(t) & 0_{d\times 1} & 0_{d\times d} & 0_{d\times 1} \\
0_{1\times d} & 0 & 0_{d\times 1} & 0
\end{bmatrix} 
\begin{bmatrix}
X & 0  & \cdots & 0 \\
y & 0  & \cdots & 0 \\
0_d  & \theta_0 & \cdots & \theta_i \\
0   & 1     & \cdots & 1
\end{bmatrix}
\Big( Z_i^\top W^{KQ_1} 
\begin{bmatrix}
 0 \\
 0 \\
 \theta_i \\
 1
\end{bmatrix}  \Big)^{\odot 3} \\
= & \frac{1}{n^3} \begin{bmatrix}
0_{d\times d} & 0_{d} & \dots & 0_{d} \\
0_{1\times d} & 0 & \dots & 0 \\
W^{PV_1}_{31}(t)X & 0_{d} & \dots & 0_{d} \\
0_{1\times d} & 0 & \dots & 0
\end{bmatrix} 
\Big( 
\begin{bmatrix}
X & 0  & \cdots & 0 \\
y & 0  & \cdots & 0 \\
0_d  & \theta_0 & \cdots & \theta_i \\
0   & 1     & \cdots & 1
\end{bmatrix}^\top 
\begin{bmatrix}
 W^{KQ_1}_{13}(t) \theta_i \\
 0 \\
 0_d \\
 0
\end{bmatrix}
\Big)^{\odot 3} \\
= & \frac{1}{n^3} \begin{bmatrix}
0_{d\times d} & 0_{d} & \dots & 0_{d} \\
0_{1\times d} & 0 & \dots & 0 \\
W^{PV_1}_{31}(t)X & 0_{d} & \dots & 0_{d} \\
0_{1\times d} & 0 & \dots & 0
\end{bmatrix} 
\begin{bmatrix}
 X^\top W^{KQ_1}_{13}(t) \theta_i \\
 0_{i+1} \\
\end{bmatrix}^{\odot 3} \\
= & \frac{1}{n^3}  \begin{bmatrix}
0_d \\
0 \\
W^{PV_1}_{31}(t)X \big(X^\top W^{KQ_1}_{13}(t) \theta_i\big)^{\odot 3} \\
0
\end{bmatrix}. 
\end{align*}

For the second head, we have
\begin{align*}
& W^{PV_2} Z_i \cdot \phi_2 \Big( \tfrac{Z_i^\top W^{KQ_2} Z_{i[:, -1]}}{n} \Big) \\
= &  \frac{1}{n^2} \begin{bmatrix}
0_{d\times d} & 0_{d\times 1} & 0_{d\times d} & 0_{d\times 1} \\
0_{1\times d} & 0 & 0_{d\times 1} & 0 \\
W^{PV_2}_{31}(t) & 0_{d\times 1} & 0_{d\times d} & 0_{d\times 1} \\
0_{1\times d} & 0 & 0_{d\times 1} & 0
\end{bmatrix}
\begin{bmatrix}
X & 0  & \cdots & 0 \\
y & 0  & \cdots & 0 \\
0_d  & \theta_0 & \cdots & \theta_i \\
0   & 1     & \cdots & 1
\end{bmatrix}
\Big( Z_i^\top W^{KQ_2} 
\begin{bmatrix}
 0 \\
 0 \\
 \theta_i \\
 1
\end{bmatrix}  \Big)^{\odot 2} \\
= & \frac{1}{n^2} \begin{bmatrix}
0_{d\times d} & 0_{d\times 1} & 0_{d\times d} & 0_{d\times 1} \\
0_{1\times d} & 0 & 0_{d\times 1} & 0 \\
W^{PV_2}_{31}(t)X & 0_{d\times 1} & 0_{d\times d} & 0_{d\times 1} \\
0_{1\times d} & 0 & 0_{d\times 1} & 0
\end{bmatrix}
\Big( 
\begin{bmatrix}
X & 0  & \cdots & 0 \\
y & 0  & \cdots & 0 \\
0_d  & \theta_0 & \cdots & \theta_i \\
0   & 1     & \cdots & 1
\end{bmatrix}^\top 
\begin{bmatrix}
 W^{KQ_2}_{13}(t) \theta_i \\
 1 \\
 0_d \\
 0
\end{bmatrix}
\Big)^{\odot 2} \\
= & \frac{1}{n^2} \begin{bmatrix}
0_{d\times d} & 0_{d\times 1} & 0_{d\times d} & 0_{d\times 1} \\
0_{1\times d} & 0 & 0_{d\times 1} & 0 \\
W^{PV_2}_{31}(t)X & 0_{d\times 1} & 0_{d\times d} & 0_{d\times 1} \\
0_{1\times d} & 0 & 0_{d\times 1} & 0
\end{bmatrix}
\begin{bmatrix}
 X^\top W^{KQ_2}_{13}(t) \theta_i + y^\top \\
 0_{i+1} \\
\end{bmatrix}^{\odot 2} \\
= & \frac{1}{n^2} 
\begin{bmatrix}
0_d \\
0 \\
W^{PV_2}_{31}(t)X \big(X^\top W^{KQ_2}_{13}(t) \theta_i + y^\top\big)^{\odot 2} \\
0
\end{bmatrix}.
\end{align*}

Similarly, for the third head, we have
\begin{align*}
W^{PV_3}Z_i \cdot \phi_3 \Big( \tfrac{Z_i^\top W^{KQ_3} Z_{i[:, -1]}}{n} \Big) 
=  \frac{1}{n^2} 
\begin{bmatrix}
0_d \\
0 \\
W^{PV_3}_{31}(t)X \big(X^\top W^{KQ_3}_{13}(t) \theta_i - y^\top\big)^{\odot 2} \\
0
\end{bmatrix}.
\end{align*}

Recall that we have the residual for a CoT step defined as 
\begin{align*}
r_i =  \sum_{h=1}^3 W^{PV_h}Z_i \cdot \phi_h \Big( \tfrac{Z_i^\top W^{KQ_h} Z_{i [:, -1]}}{n} \Big) - (0_d, 0, \theta_{i+1} - \theta_i, 0)^\top 
 = (0_d, 0, \widehat{r}_i - (\theta_{i+1} - \theta_i), 0)^\top,
\end{align*}
where we define
\begin{align}
\widehat{r}_i & = \frac{1}{n^3} W^{PV_1}_{31}X \big(X^\top W^{KQ_1}_{13} \theta_i\big)^{\odot 3} + \frac{1}{n^2} W^{PV_2}_{31}X \big(X^\top W^{KQ_2}_{13} \theta_i + y^\top\big)^{\odot 2} \notag \\
&\qquad + \frac{1}{n^2} W^{PV_3}_{31}X \big(X^\top W^{KQ_3}_{13} \theta_i - y^\top\big)^{\odot 2} \notag \\
& = \frac{1}{n^3} W^{PV_1}_{31}X \big(X^\top W^{KQ_1}_{13} \theta_i\big)^{\odot 3} + \frac{4}{n^2} W^{PV_2}_{31}X \big(X^\top W^{KQ_2}_{13} \theta_i \odot y^\top\big).
\end{align}
The second step is because of the symmetry in Assumption \ref{ass_init_quad} that $W^{PV_2}_{31} = -W^{PV_3}_{31}$ and $W^{KQ_2}_{13} = W^{KQ_3}_{13}$.

Now, we consider the gradient for $W^{PV_h}$. Denote the loss at step $i$ as $L_i(W)$. For an $X$ and $\theta^\star$,
\begin{align*}
\nabla_{W^{PV_1}} L_i(W) & =  r_i \big(Z_i\phi_1(\tfrac{Z_i^\top W^{KQ_1}Z_{i[:,-1]}}{n})\big)^\top 
= \frac{1}{n^3} r_i \big(
Z_i \begin{bmatrix}
X^\top W^{KQ_1}_{13}(t)\theta_i\\
0_{i+1}
\end{bmatrix} ^{\odot 3} \big)^\top \\
& = 
\frac{1}{n^3} 
\begin{bmatrix}
0_d \\
0 \\
\widehat{r}_i - (\theta_{i+1} - \theta_i)\\
0
\end{bmatrix}
\begin{bmatrix}
X\big(X^\top W^{KQ_1}_{13}(t)\theta_i\big)^{\odot 3} \\
y \big(X^\top W^{KQ_1}_{13}(t)\theta_i\big)^{\odot 3} \\
0_d \\
0
\end{bmatrix}^\top \\
& = \frac{1}{n^3} \begin{bmatrix}
0_{d \times d} & 0_d & 0_{d \times d} & 0_d \\
0_{1 \times d} & 0 & 0_{1 \times d} & 0 \\
\nabla_{W^{PV_1}_{31}} L_i & \nabla_{W^{PV_1}_{32}} L_i & 0_{d \times d} & 0_d \\
0_{1 \times d} & 0 & 0_{1 \times d} & 0 \\
\end{bmatrix}
\end{align*}
where 
\begin{align*}
& \nabla_{W^{PV_1}_{31}} L_i = \widehat{r_i} \big( X \big(X^\top W^{KQ_1}_{13}(t)\theta_i\big)^{\odot 3} \big)^\top - (\theta_{i+1} - \theta_i) \big(  X \big(X^\top W^{KQ_1}_{13}(t)\theta_i\big)^{\odot 3} \big)^\top\\
& \nabla_{W^{PV_1}_{32}} L_i = \widehat{r_i} \big(y \big(X^\top W^{KQ_1}_{13}(t)\theta_i\big)^{\odot 3} \big)^\top- (\theta_{i+1} - \theta_i) \big( y \big(X^\top W^{KQ_1}_{13}(t)\theta_i\big)^{\odot 3} \big)^\top.
\end{align*}
Similarly, we calculate the gradient of $W^{PV_2}$. We have
\begin{align*}
\nabla_{W^{PV_2}} L_i(W) & =  r_i \big(Z_i\phi_2(\tfrac{Z_i^\top W^{KQ_2}Z_{i[:,-1]}}{n})\big)^\top 
= \frac{1}{n^2} r_i \big(
Z_i \begin{bmatrix}
X^\top W^{KQ_2}_{13}(t)\theta_i + y^\top \\
0_{i+1}
\end{bmatrix}^{\odot 2} \big)^\top \\
& = 
\frac{1}{n^2} 
\begin{bmatrix}
0_d \\
0 \\
\widehat{r}_i - (\theta_{i+1} - \theta_i)\\
0
\end{bmatrix}
\begin{bmatrix}
X\big(X^\top W^{KQ_2}_{13}(t)\theta_i + y^\top \big )^{\odot 2} \\
y \big(X^\top W^{KQ_2}_{13}(t)\theta_i + y^\top \big )^{\odot 2} \\
0_d \\
0
\end{bmatrix}^\top \\
& = \frac{1}{n^2}  \begin{bmatrix}
0_{d \times d} & 0_d & 0_{d \times d} & 0_d \\
0_{1 \times d} & 0 & 0_{1 \times d} & 0 \\
\nabla_{W^{PV_2}_{31}} L_i & \nabla_{W^{PV_2}_{32}} L_i & 0_{d \times d} & 0_d \\
0_{1 \times d} & 0 & 0_{1 \times d} & 0 \\
\end{bmatrix}, \text{ where}
\end{align*}
\begin{align*}
& \nabla_{W^{PV_2}_{31}} L_i = \widehat{r}_i \big(X \big(X^\top W^{KQ_2}_{13}(t)\theta_i + y^\top \big)^{\odot 2}\big)^\top - (\theta_{i+1} - \theta_i) \big( X \big(X^\top W^{KQ_2}_{13}(t)\theta_i + y^\top \big)^{\odot 2}\big)^\top \\
& \nabla_{W^{PV_2}_{32}} L_i = \widehat{r}_i \big( y \big(X^\top W^{KQ_2}_{13}(t)\theta_i + y^\top \big)^{\odot 2} \big)^\top - (\theta_{i+1} - \theta_i) \big( y \big(X^\top W^{KQ_2}_{13}(t)\theta_i + y^\top \big)^{\odot 2} \big)^\top. 
\end{align*}
Also, the gradient for $W^{PV_3}$ is as follows:
\begin{align*}
\nabla_{W^{PV_3}} L_i(W) & =  r_i \big(Z_i\phi_3(\tfrac{Z_i^\top W^{KQ_3}Z_{i[:,-1]}}{n})\big)^\top 
= \frac{1}{n^2} r_i \big(
Z_i \begin{bmatrix}
X^\top W^{KQ_3}_{13}(t)\theta_i - y^\top \\
0_{i+1}
\end{bmatrix}^{\odot 2} \big)^\top \\
& = 
\frac{1}{n^2} 
\begin{bmatrix}
0_d \\
0 \\
\widehat{r}_i - (\theta_{i+1} - \theta_i)\\
0
\end{bmatrix}
\begin{bmatrix}
X\big(X^\top W^{KQ_3}_{13}(t)\theta_i - y^\top \big )^{\odot 2} \\
y \big(X^\top W^{KQ_3}_{13}(t)\theta_i - y^\top \big )^{\odot 2} \\
0_d \\
0
\end{bmatrix}^\top \\
& = \frac{1}{n^2} \begin{bmatrix}
0_{d \times d} & 0_d & 0_{d \times d} & 0_d \\
0_{1 \times d} & 0 & 0_{1 \times d} & 0 \\
\nabla_{W^{PV_3}_{31}} L_i & \nabla_{W^{PV_3}_{32}} L_i & 0_{d \times d} & 0_d \\
0_{1 \times d} & 0 & 0_{1 \times d} & 0 \\
\end{bmatrix}, \text{ where } \\
 \nabla_{W^{PV_3}_{31}} L_i &= \widehat{r}_i \big( X \big(X^\top W^{KQ_3}_{13}(t)\theta_i - y^\top \big)^{\odot 2}\big)^\top  - (\theta_{i+1} - \theta_i)  \big(X \big(X^\top W^{KQ_3}_{13}(t)\theta_i - y^\top \big)^{\odot 2}\big)^\top  \\
 \nabla_{W^{PV_3}_{32}} L_i &= \widehat{r}_i \big(y \big(X^\top W^{KQ_3}_{13}(t)\theta_i - y^\top \big)^{\odot 2}\big)^\top  - (\theta_{i+1} - \theta_i)  \big(y \big(X^\top W^{KQ_3}_{13}(t)\theta_i - y^\top \big)^{\odot 2}\big)^\top.
\end{align*}
Next, we calculate the gradient for $W^{KQ_h}$. We write $r^\theta_i := \widehat r_i - (\bar\theta_{i+1}-\theta_i)$ for the $\theta$-block of $r_i$. Thus, $r_i = (0_d, 0, r^\theta_i, 0)$, and $Z_i^\top W^{PV_h\top}r_i = [X^\top W^{PV_h\top}_{31}r^\theta_i;\,0_{i+1}]$. Then, we have
\begin{align*}
\nabla_{W^{KQ_1}} L_i & = 
\frac{1}{n}  Z_i \Big( \big(Z_i^\top W^{PV_1 \top} r_i \big) \odot 3 (\tfrac{Z_i^\top W^{KQ_1} Z_{i [:, -1]}}{n})^{\odot 2} \Big) Z_{i [:, -1]}^\top \\
& = \frac{1}{n} Z_i \Big(
\begin{bmatrix}
X^\top W^{PV_1\top}_{31}(t) r^\theta_i\\
0_{i+1}
\end{bmatrix} \odot 
\frac{3}{n^2}
\begin{bmatrix}
\big(X^\top W^{KQ_1}_{13}(t)\theta_i\big)^{\odot 2}\\
0_{i+1}
\end{bmatrix} 
\Big) Z_{i [:, -1]}^\top \\
& = \frac{3}{n^3} Z_i 
\begin{bmatrix}
\Big ( X^\top W^{PV_1\top}_{31}(t) r^\theta_i \odot  \big(X^\top W^{KQ_1}_{13}(t)\theta_i\big)^{\odot 2} \Big )\\
0_{i+1}
\end{bmatrix} 
 Z_{i [:, -1]}^\top \\
& = \frac{3}{n^3} \begin{bmatrix}
X \Big ( X^\top W^{PV_1\top}_{31}(t) r^\theta_i \odot  \big(X^\top W^{KQ_1}_{13}(t)\theta_i\big)^{\odot 2} \Big )\\
y \Big ( X^\top W^{PV_1\top}_{31}(t) r^\theta_i \odot  \big(X^\top W^{KQ_1}_{13}(t)\theta_i\big)^{\odot 2} \Big ) \\
0_{d} \\
0
\end{bmatrix} 
\begin{bmatrix} 
0_d & 0 & \theta_i^\top & 1
\end{bmatrix} \\
& = \frac{3}{n^3} \begin{bmatrix}
0_{d \times d} & 0_d & \nabla_{W^{KQ_1}_{13}} L_i & \nabla_{W^{KQ_1}_{14}} L_i \\
0_{1 \times d} & 0 & \nabla_{W^{KQ_1}_{23}} L_i & \nabla_{W^{KQ_1}_{24}} L_i  \\
0_{d\times d} & 0_d & 0_{d \times d} & 0 \\
0_{1\times d} & 0 & 0_{1 \times d} & 0 \\
\end{bmatrix}, \text{ where }\\
 \nabla_{W^{KQ_1}_{13}} L_i &= X \Big ( X^\top W^{PV_1\top}_{31}(t) r^\theta_i \odot  \big(X^\top W^{KQ_1}_{13}(t)\theta_i\big)^{\odot 2} \Big )\theta_i^\top  \\
 \nabla_{W^{KQ_1}_{23}} L_i &= y \Big ( X^\top W^{PV_1\top}_{31}(t) r^\theta_i \odot  \big(X^\top W^{KQ_1}_{13}(t)\theta_i\big)^{\odot 2} \Big ) \theta_i^\top, 
\end{align*}
and we fix $W^{KQ_1}_{24} = 0$ and $W^{KQ_1}_{14} = 0_{d}$ for $t > 0$.  
Then, for head two, we have
\begin{align*}
\nabla_{W^{KQ_2}} L_i & = 
\frac{1}{n}  Z_i \Big( \big(Z_i^\top W^{PV_2 \top} r_i \big) \odot 2 (\tfrac{Z_i^\top W^{KQ_2} Z_{i [:, -1]}}{n}) \Big) Z_{i [:, -1]}^\top \\
& = \frac{1}{n} Z_i \Big(
\begin{bmatrix}
X^\top W^{PV_2\top}_{31}(t) r^\theta_i\\
0_{i+1}
\end{bmatrix} \odot 
\frac{2}{n}
\begin{bmatrix}
\big(X^\top W^{KQ_2}_{13}(t)\theta_i + y^\top\big)\\
0_{i+1}
\end{bmatrix} 
\Big) Z_{i [:, -1]}^\top \\
& = \frac{2}{n^2} Z_i 
\begin{bmatrix}
\Big ( X^\top W^{PV_2\top}_{31}(t) r^\theta_i \odot  \big(X^\top W^{KQ_2}_{13}(t)\theta_i + y^\top\big) \Big )\\
0_{i+1}
\end{bmatrix} 
 Z_{i [:, -1]}^\top \\
& = \frac{2}{n^2} \begin{bmatrix}
X \Big ( X^\top W^{PV_2\top}_{31}(t) r^\theta_i \odot  \big(X^\top W^{KQ_2}_{13}(t)\theta_i + y^\top\big) \Big )\\
y \Big ( X^\top W^{PV_2\top}_{31}(t) r^\theta_i \odot  \big(X^\top W^{KQ_2}_{13}(t)\theta_i + y^\top\big) \Big ) \\
0_{d} \\
0
\end{bmatrix} 
\begin{bmatrix} 
0_d & 0 & \theta_i^\top & 1
\end{bmatrix} \\
& = \frac{2}{n^2} \begin{bmatrix}
0_{d \times d} & 0_d & \nabla_{W^{KQ_2}_{13}} L_i & \nabla_{W^{KQ_2}_{14}} L_i \\
0_{1 \times d} & 0 & \nabla_{W^{KQ_2}_{23}} L_i & \nabla_{W^{KQ_2}_{24}} L_i  \\
0_{d\times d} & 0_d & 0_{d \times d} & 0 \\
0_{1\times d} & 0 & 0_{1 \times d} & 0 \\
\end{bmatrix}, \text{ where} \\
 \nabla_{W^{KQ_2}_{13}} L_i &= X \Big ( X^\top W^{PV_2\top}_{31}(t) r^\theta_i \odot  \big(X^\top W^{KQ_2}_{13}(t)\theta_i + y^\top\big) \Big )\theta_i^\top  \\
 \nabla_{W^{KQ_2}_{23}} L_i &= y \Big ( X^\top W^{PV_2\top}_{31}(t) r^\theta_i \odot  \big(X^\top W^{KQ_2}_{13}(t)\theta_i + y^\top\big) \Big ) \theta_i^\top, 
\end{align*}
and we fix $W^{KQ_2}_{24} = 1$ and $W^{KQ_2}_{14} = 0_d$ for $t > 0$.  
Similarly, we have the gradient for $W^{KQ_3}$ as follows. 
\begin{align*}
& \nabla_{W^{KQ_3}_{13}} L_i = X \Big ( X^\top W^{PV_3\top}_{31}(t) r^\theta_i \odot  \big(X^\top W^{KQ_3}_{13}(t)\theta_i - y^\top\big) \Big )\theta_i^\top  \\
& \nabla_{W^{KQ_3}_{23}} L_i = y \Big ( X^\top W^{PV_3\top}_{31}(t) r^\theta_i \odot  \big(X^\top W^{KQ_3}_{13}(t)\theta_i - y^\top\big) \Big ) \theta_i^\top, 
\end{align*}
and we fix $W^{KQ_3}_{24} = -1$ and $W^{KQ_3}_{14} = 0_d$ for $t > 0$.  

Now, we are going to show that the expected gradients with respect to the irrelevant blocks are zero, i.e., $\nabla_{W^{PV_1}_{32}} L_i$, $\nabla_{W^{PV_2}_{32}} L_i$, $\nabla_{W^{PV_3}_{32}} L_i$, $\nabla_{W^{KQ_1}_{23}} L_i$, $\nabla_{W^{KQ_2}_{23}} L_i$, $\nabla_{W^{KQ_3}_{23}} L_i$. We will show that $\exarg{X, \theta^\star}{\nabla_{W^{PV_1}_{32}} L_i} = 0$, and proofs for others follow similarly. 

We use the sign-flip symmetry of the Gaussian distribution to prove this fact. Denote $g_i := y^\top a_i^{\odot 3}$ and $a_i := X^\top W^{KQ_1}_{13}\theta_i$. Define a sign-flip transformation $\cT:  (X,\theta^\star)\mapsto (-X,\theta^\star)$. Since $X$ has i.i.d.\ standard Gaussian columns, the distribution of $(X, \theta^\star)$ is invariant under $\cT$. 

The parameter token $\theta_i$ and the target increment $\theta_{i+1}-\theta_i$ used in $Z_i$ are generated by the deterministic gradient descent on the quadratic squared loss. Thus, $\theta_i$ is invariant under $\cT$. Then, based on the definitions, we have $y$ is invariant under $\cT$, and $g_i$ will be flipped to $-g_i$ under $\cT$.

Moreover, we can show $\widehat{r}_i$ is even under $\cT$. Under the block expressions, we have
\begin{align*}
r^\theta_i & = \widehat r_i - (\theta_{i+1}-\theta_i) \\
& = \frac{1}{n^3} W^{PV_1}_{31}X \big(X^\top W^{KQ_1}_{13} \theta_i\big)^{\odot 3} + \frac{4}{n^2} W^{PV_2}_{31}X \big(X^\top W^{KQ_2}_{13} \theta_i \odot y^\top\big)  - (\theta_{i+1}-\theta_i)
\end{align*}
One can verify that $r^\theta_i$ is invariant under the transformation $\cT$. Thus, $\nabla_{W^{PV_1}_{32}} L_i(W) = r_i g_i$ is odd under $\cT$. 
Combining the above sign-flipping results and using the distributional invariance, we have
\[
\exarg{X, \theta^\star}{\nabla_{W^{PV_1}_{32}} L_i(W)}
=
\exarg{\cT(X, \theta^\star)}{\nabla_{W^{PV_1}_{32}} L_i(W)\circ \cT}
=
-\exarg{X, \theta^\star}{\nabla_{W^{PV_1}_{32}} L_i(W)},
\]
Therefore, $\exarg{X, \theta^\star}{\nabla_{W^{PV_1}_{32}} L_i(W)}=0$.
\end{proof}

Notice that the final loss only depends on the $(d + 2)$ to $(2d + 1)$ entries of the transformer’s
output. We can simplify the model by pruning the irrelevant and zero blocks. We can define a simplified transformer, where the dynamics of the equivalent model are exactly the same as the original model.

\begin{definition}\label{def_reformulation_quadratic}
Under Assumption \ref{ass_init_quad},  by Lemma~\ref{lemma_gradient_structure_quadratic}, for a given input $Z_i$, we can rewrite the one-layer non-linear transformer by
\[
f_W(Z_i) = \Big [0_d, 0,   \big( \theta_i - \frac{2\eta}{n} W^{PV}_{31}X \big(X^\top W^{KQ}_{13} \theta_i\big)^{\odot 3} + \frac{2\eta}{n} W^{PV}_{31}X \big(X^\top W^{KQ}_{13} \theta_i \odot y^\top\big) \big),   1 \Big]^\top.
\]
where $\theta_i$ is the last column of ${Z_i}_{[d+2: 2d+1, -1]}$ and $y^\top = (X^\top\theta^\star)^{\odot 2}$ where $\theta^\star$ is the ground-truth weights of the function. Then, we can denote a reduced model by: 
\begin{align}\label{eq_reduced_transformer_quadratic}
\tilde{f}_W(Z_i) =  \theta_i - \frac{2\eta}{n} W^{PV}_{31}X \big(X^\top W^{KQ}_{13} \theta_i\big)^{\odot 3} + \frac{2\eta}{n} W^{PV}_{31}X \big(X^\top W^{KQ}_{13} \theta_i \odot y^\top\big). 
\end{align}
For the ease of notations, we will write $\tilde{f}_W(Z_i) = \tilde f (\theta_i)$, where $\theta_i$ denotes the $(d + 2) $ to $(2d + 1)$ entries of the last column of $Z_i$. 
\end{definition}

In the following sections, we will consider the equivalent form of the transformer. Here we present the gradient with respect to the reduced model.

\begin{lemma}\label{lemma_reduced_model_gradient}
The gradients of $W^{PV}_{31}$ and $W^{KQ}_{13}$ are given by the following expectations:
\begin{align*}
& \nabla_{W^{PV}_{31}} L^{\mathrm{CoT}}(W)\\
= & \sum_{i=0}^{T-1} \exarg{X, \theta^\star}{- \frac{2\eta}{n}(\tilde{f}_W(\theta_i) - \theta_{i+1}) \left(\big( X(X^\top W^{KQ}_{13}\theta_i)^{\odot 3}\big)^\top - \big( X \big(X^\top W^{KQ}_{13}\theta_i \odot y^\top \big)\big)^\top\right)} \\
& + \exarg{X, \theta^\star}{- \frac{2\eta}{n}(\theta_{T} - \theta^{\star}) \left(\big( X(X^\top W^{KQ}_{13}\theta_{T-1})^{\odot 3}\big)^\top - \big( X \big(X^\top W^{KQ}_{13}\theta_{T-1} \odot y^\top \big) \big)^\top \right)}, \text{ and } \\
& \nabla_{W^{KQ}_{13}} L^{\mathrm{CoT}}(W) \\
=&  \sum_{i=0}^{T-1} \exarg{X, \theta^\star}{- \frac{2\eta}{n} X \Big ( X^\top W^{PV\top}_{31} (\tilde{f}_W(\theta_i) - \theta_{i+1}) \odot  \big( 3 \big(X^\top W^{KQ}_{13} \theta_i \big)^{\odot 2}  - y^\top \big) \Big )\theta_i^\top}  \\
& + \exarg{X, \theta^\star}{- \frac{2\eta}{n} X \Big ( X^\top W^{PV\top}_{31} (\theta_{T} - \theta^\star) \odot  \big( 3 \big(X^\top W^{KQ}_{13} \theta_{T-1} \big)^{\odot 2} - y^\top \big) \Big )\theta_{T-1}^\top}. 
\end{align*}
\end{lemma}

\begin{proof}
Based on the simplified form of the transformer in Definition \ref{def_reformulation_quadratic}, we can rewrite the CoT training loss as follows. 
\begin{align*}
L^{\mathrm{CoT}}(W) &
= \frac{1}{2} \exarg{X, \theta^\star} {
\sum_{i=0}^{T-2} \norm{f_{W}(Z_i) - (0_d, 0, \theta_{i+1}, 1)}_2^2}  + \frac{1}{2} \exarg{X, \theta^\star} {\norm{f_{W}(Z_{T-1}) - (0_d, 0, \theta^\star, 1)}_2^2} \\
& = \frac{1}{2} \exarg{X, \theta^\star} {
\sum_{i=0}^{T-2} \norm{\tilde{f}_{W}(\theta_i) - \theta_{i+1}}_2^2}  + \frac{1}{2} \exarg{X, \theta^\star} {\norm{\tilde{f}_{W}(\theta_{T-1}) - \theta^\star}_2^2}.
\end{align*}
Here, we split the loss between the last step output and the ground-truth weights $\theta^\star$. 

Then, based on the gradient computation in Lemma \ref{lemma_gradient_structure_quadratic}, we can write out the gradient for the weight matrix $W^{PV}$ in the transformer. Recall that in Assumption \ref{ass_init_quad}, we use a weight-sharing mechanism that defines $W^{PV_1}_{31}(t) = -2\eta n^2 W^{PV}_{31}(t), W^{PV_2}_{31}(t) = \frac{\eta n}{2} W^{PV}_{31}(t), 
W^{PV_3}_{31}(t) = - \frac{\eta n}{2} W^{PV}_{31}(t)$. Therefore,
\begin{align*}
& \nabla_{W^{PV}_{31}} L^{\mathrm{CoT}}(W) \\
= & -2\eta n^2 \nabla_{W^{PV_1}_{31}} L^{\mathrm{CoT}}(W) + \frac{\eta n}{2} \nabla_{W^{PV_2}_{31}} L^{\mathrm{CoT}}(W) - \frac{\eta n}{2} \nabla_{W^{PV_3}_{31}} L^{\mathrm{CoT}}(W) \\
= &  
\sum_{i=0}^{T-2} \exarg{X, \theta^\star}{- \frac{2\eta}{n}(\tilde{f}_W(\theta_i) - \theta_{i+1}) \big( X(X^\top W^{KQ}_{13}\theta_i)^{\odot 3} \big)^\top } \\
& + \exarg{X, \theta^\star}{- \frac{2\eta}{n}(\tilde{f}_W(\theta_{T-1}) - \theta^{\star}) \big(X(X^\top W^{KQ}_{13}\theta_{T-1})^{\odot 3} \big)^\top } \\
& - \sum_{i=0}^{T-2} \exarg{X, \theta^\star}{- \frac{2\eta}{n}(\tilde{f}_W(\theta_i) - \theta_{i+1}) \big( X \big(X^\top W^{KQ}_{13}\theta_i \odot y^\top \big)\big)^\top} \\
& - \exarg{X, \theta^\star}{- \frac{2\eta}{n}(\tilde{f}_W(\theta_{T-1}) - \theta^{\star})\big( X \big(X^\top W^{KQ}_{13}\theta_{T-1} \odot y^\top \big) \big)^\top} \\ 
= & \sum_{i=0}^{T-1} \exarg{X, \theta^\star}{- \frac{2\eta}{n}(\tilde{f}_W(\theta_i) - \theta_{i+1}) \left( \big( X(X^\top W^{KQ}_{13}\theta_i)^{\odot 3}\big)^\top - \big( X \big(X^\top W^{KQ}_{13}\theta_i \odot y^\top \big)\big)^\top \right)} \\
& + \exarg{X, \theta^\star}{- \frac{2\eta}{n}(\theta_{T} - \theta^{\star}) \big( X(X^\top W^{KQ}_{13}\theta_{T-1})^{\odot 3}\big)^\top +  \frac{2\eta}{n}(\theta_{T} - \theta^{\star})\big( X \big(X^\top W^{KQ}_{13}\theta_{T-1} \odot y^\top \big) \big)^\top}.
\end{align*}
The last step is because we subtract \[ \exarg{X, \theta^\star}{- \frac{2\eta}{n}(\tilde{f}_W(\theta_{T-1}) - \theta_{T}) \big (X(X^\top W^{KQ}_{13}\theta_{T-1})^{\odot 3}} \big )^\top \] from the second term and put it into the first term, then subtract \[ \exarg{X, \theta^\star}{- \frac{2\eta}{n}(\tilde{f}_W(\theta_{T-1}) - \theta_{T})\big( X \big(X^\top W^{KQ}_{13}\theta_{T-1} \odot y^\top \big)} \] from the fourth term and put it into the third term. 

Similarly, we can break down the gradient for $W^{KQ}_{13}$. Recall that in Assumption \ref{ass_init_quad}, we use a weight-sharing mechanism that defines $W^{KQ_1}_{13}(t) = W^{KQ}_{13}(t), W^{KQ_2}_{13}(t) = W^{KQ}_{13}(t), 
W^{KQ_3}_{13}(t) = W^{KQ}_{13}(t)$. Therefore, we have
\begin{align*}
& \nabla_{W^{KQ}_{13}} L^{\mathrm{CoT}}(W) \\
= &  \nabla_{W^{KQ_1}_{13}} L^{\mathrm{CoT}}(W) + \nabla_{W^{KQ_2}_{13}} L^{\mathrm{CoT}}(W) + \nabla_{W^{KQ_3}_{13}} L^{\mathrm{CoT}}(W)\\
= &  
\sum_{i=0}^{T-1} \exarg{X, \theta^\star}{-\frac{6\eta}{n} X \Big ( X^\top W^{PV\top}_{31} (\tilde{f}_W(\theta_i) - \theta_{i+1}) \odot  \big(X^\top W^{KQ}_{13} \theta_i\big)^{\odot 2} \Big )\theta_i^\top} + \\
& \exarg{X, \theta^\star}{-\frac{6\eta}{n} X \Big ( X^\top W^{PV\top}_{31} (\theta_{T} - \theta^\star) \odot  \big(X^\top W^{KQ}_{13} \theta_{T-1}\big)^{\odot 2} \Big )\theta_{T-1}^\top} - \\ 
&  \sum_{i=0}^{T-1} \exarg{X, \theta^\star}{-\frac{2\eta}{n} X \Big ( X^\top W^{PV\top}_{31} (\tilde{f}_W(\theta_i) - \theta_{i+1}) \odot  y^\top \Big )\theta_i^\top} - \\
& \exarg{X, \theta^\star}{-\frac{2\eta}{n} X \Big ( X^\top W^{PV\top}_{31} (\theta_{T} - \theta^\star) \odot  y^\top  \Big )\theta_{T-1}^\top} \\
= &  \sum_{i=0}^{T-1} \exarg{X, \theta^\star}{-\frac{2\eta}{n} X \Big ( X^\top W^{PV\top}_{31} (\tilde{f}_W(\theta_i) - \theta_{i+1}) \odot  \big( 3 \big(X^\top W^{KQ}_{13} \theta_i \big)^{\odot 2}  - y^\top \big) \Big )\theta_i^\top}  + \\
& \exarg{X, \theta^\star}{-\frac{2\eta}{n} X \Big ( X^\top W^{PV\top}_{31} (\theta_{T} - \theta^\star) \odot  \big( 3 \big(X^\top W^{KQ}_{13} \theta_{T-1} \big)^{\odot 2} - y^\top \big) \Big )\theta_{T-1}^\top}. 
\end{align*}
\end{proof}

\subsubsection{Proof of Theorem \ref{theorem_convergence_three_head}}
\label{sec_nonlinear_convergence}

With the above computed gradients, we now prove the main theorem on the convergence of transformers trained on quadratic functions. We characterize the training dynamics using ordinary differential equations in terms of the eigenvalues of the transformer weights. We first characterize the gradient of the CoT objective as a main term plus interaction terms, which yields an ordinary differential equation for the eigenvalues of the two trained blocks. Then, we analyze the ordinary differential equation in stages.

\begin{proof}
We first use the gradient computation in Lemma \ref{lemma_reduced_model_gradient} and derive ordinary differential equations for the eigenvalues of the weight matrices $W^{PV}_{31}$ and $W^{KQ}_{13}$. By Assumption~\ref{ass_init_quad}, we have weight matrices in the form for all $t\ge0$:
\[
W^{PV}_{31}(t)=\lambda_V(t)\id_d,\qquad W^{KQ}_{13}(t)=\lambda_K(t)\id_d,
\]
where the dynamics reduce to the two eigenvalues $\lambda_V(t)$ and $\lambda_K(t)$. Consider that the gradient descent initialization satisfies $\theta_0=\alpha_0\theta^\star$ with a constant $\alpha_0\in[3/4,1)$, and take $n = \Theta( d^2 \log^5 d)$.
We will show that the gradients in the eigenbasis admit the following decomposition
\begin{align*}
\nabla_{W^{PV}_{31}} L^{\mathrm{CoT}}(W) = \big(g_V(\lambda_V,\lambda_K)+\Delta^V\big)\id_d,\\
\nabla_{W^{KQ}_{13}} L^{\mathrm{CoT}}(W) = \big(g_K(\lambda_V,\lambda_K)+\Delta^K\big)\id_d,
\end{align*}
where $g_V,g_K$ are explicit main terms, and $\Delta^V,\Delta^K$ are interaction terms. The induced eigenvalue ODEs can be represented by
\[
\dot\lambda_V(t)= -g_V\big(\lambda_V(t),\lambda_K(t)\big) - \Delta^V(t),\qquad \dot\lambda_K(t)= -g_K\big(\lambda_V(t),\lambda_K(t)\big) - \Delta^K(t),
\]
In particular, by expanding the main terms around the construction $\lambda_V=\lambda_K=1$, we will show that
\begin{align*}
g_V(\lambda_V,\lambda_K) &= 36\kappa^2\sum_{i=0}^{T-1}\beta_i(1-\beta_i)\big[(1-\beta_i)(\lambda_V-1)+(1-3\beta_i)(\lambda_K-1)\big]\\
&\qquad + O\big(|\lambda_V-1|^2+|\lambda_K-1|^2\big),\\
g_K(\lambda_V,\lambda_K) &= 36\kappa^2\sum_{i=0}^{T-1}\beta_i(1-3\beta_i)\big[(1-\beta_i)(\lambda_V-1)+(1-3\beta_i)(\lambda_K-1)\big]\\
&\qquad + O\big(|\lambda_V-1|^2+|\lambda_K-1|^2\big),
\end{align*}
where the coefficients depend only on $(T,\kappa,\alpha_0)$, and the weights $\beta_i\in[\alpha_0^2,1)$.

Now, we derive the explicit forms of the ODE. Based on Lemma~\ref{lemma_gradient_structure_quadratic}, if the parameters are initialized as in Assumption~\ref{ass_init_quad}, then under gradient flow all zero blocks remain zero for all $t>0$. Hence, the model dynamics are fully characterized by the two matrices $W^{PV}_{31}(t)$ and $W^{KQ}_{13}(t)$. Then, we write the reduced form of the transformer as in Equation \ref{eq_reduced_transformer_quadratic}. 

The chain-of-thought training loss can be written as
\[
L^{\mathrm{CoT}}(W)
=
\frac12
\mathbb E_{X,\theta^\star}
\Big[
\sum_{i=0}^{T-2}
\|\tilde f_W(\theta_i)-\theta_{i+1}\|_2^2
+
\|\tilde f_W(\theta_{T-1})-\theta^\star\|_2^2
\Big].
\]
with the residuals defined as 
\[
e_i :=
\begin{cases}
\tilde f_W(\theta_i)-\theta_{i+1}, & i=0,\dots,T-2,\\
\tilde f_W(\theta_{T-1})-\theta^\star, & i=T-1.
\end{cases}
\]

Let
\[
\Psi(\theta)
:=
X\Big(
(X^\top W^{KQ}_{13}\theta)^{\odot 3}
-
(X^\top W^{KQ}_{13}\theta)\odot y^\top
\Big).
\]
Then the gradients of the reduced parameters are
\begin{align}
\nabla_{W^{PV}_{31}} L^{\mathrm{CoT}}(W)
&=
\mathbb E\Big[
\sum_{i=0}^{T-1}
-\frac{2\eta}{n}
e_i\Psi(\theta_i)^\top
\Big],
\label{eq:grad_pv}
\\
\nabla_{W^{KQ}_{13}} L^{\mathrm{CoT}}(W)
&=
\mathbb E\Big[
\sum_{i=0}^{T-1}
-\frac{2\eta}{n}
X\Big(
(X^\top W^{PV\top}_{31}e_i)
\odot
\big(3(X^\top W^{KQ}_{13}\theta_i)^{\odot 2}-y^\top\big)
\Big)\theta_i^\top
\Big].
\label{eq:grad_kq}
\end{align}

As in Assumption \ref{ass_init_quad}, assume $W^{PV}_{31}(t)$ and $W^{KQ}_{13}(t)$ are symmetric and share an orthonormal eigenbasis $U$.
Rotating all variables by $U^\top$ leaves the data distribution invariant, so we can work in the basis where $U=\id_d$ and both matrices are diagonal. Thus, the isotropy is preserved. The dynamics reduce to ODEs for the eigenvalues $\lambda_V(t)$ and $\lambda_K(t)$. 

Then, we further write the ODE as the expectation over $\theta^\star$ and $X$. Let $\theta^\star\sim\mathcal N(0,\id_d)$ and initialize $\theta_0=\alpha_0\theta^\star$ with {$\alpha_0\in[3/4,1)$}. Then, gradient descent on the quadratic functions preserves collinearity, so
$\theta_i = \alpha_i\theta^\star, \
\alpha_{i+1}
= \alpha_i -6\eta\|\theta^\star\|_2^2\alpha_i(\alpha_i^2-1)$. Here, $\alpha_i$ are scalars that only depend on the data $\|\theta^\star\|^2$. Let $\bar\alpha_i$ follow the same recursion with $\|\theta^\star\|^2$ replaced by its mean $d$, and let $\beta_i := \bar\alpha_i^2$. The empirical iterates $\theta_i$ stay close to $\alpha_i\theta^\star$, and $\|\theta^\star\|^2$ is close to $d$.

Given $W^{PV}_{31}=\lambda_V\id_d$ and $W^{KQ}_{13}=\lambda_K\id_d$, let
\[
s := 1-\lambda_V\lambda_K^3,\qquad \tau := 1-\lambda_V\lambda_K,
\]
which both vanish at the construction $\lambda_V=\lambda_K=1$. Then, we have 
\[
\Psi(\theta) = \lambda_K^3X(X^\top\theta)^{\odot3}-\lambda_KX\big((X^\top\theta)\odot y^\top\big),
\]
for $i\le T-2$,
\[
e_i = \frac{2\eta}{n}\Big(s\,X(X^\top\theta_i)^{\odot3}-\tau\,X\big((X^\top\theta_i)\odot y^\top\big)\Big),
\]
and $e_{T-1}$ is the same expression plus $\theta_T-\theta^\star$. As the other terms do not depend on $(\lambda_V,\lambda_K)$, $L^{\mathrm{CoT}}$ is exactly a quadratic function of $(s,\tau)$. 

Then, we apply the expectation to the Gaussian moment. for $x\sim\mathcal N(0,\id_d)$ and fixed vectors $v,w\in\mathbb R^d$, we have
$
\mathbb E[x(x^\top v)^3]=3\|v\|^2v,
\qquad
\mathbb E[x(x^\top v)(x^\top w)^2]=\|w\|^2v+2(v^\top w)w.
$
Applying these identities yields the population expectation
\begin{align}
\mathbb E_X[\Psi(\theta)\mid\theta,\theta^\star]
=
n\Big(3\lambda_K^3\|\theta\|^2\theta-\lambda_K\big(\|\theta^\star\|^2\theta+2(\theta^\top\theta^\star)\theta^\star\big)\Big).
\label{eq_psi_expectation}
\end{align}
In particular, at $\theta=\alpha\theta^\star$ the expectations of $\frac1nX(X^\top\theta)^{\odot3}$ and $\frac1nX\big((X^\top\theta)\odot y^\top\big)$ are $3\alpha^3 \|\theta^\star\|^2\theta^\star$ and $3\alpha \|\theta^\star\|^2\theta^\star$.

We keep only the population parts and replace $\|\theta^\star\|^2$ by $d$ and $\alpha_i$ by $\bar\alpha_i$, which defines the main terms. The remaining terms are the interaction terms. Then, taking the exact expectation over
$\theta^\star$ and $X$, the population gradient-flow dynamics of the eigenvalues are given by
\begin{align}
\dot\lambda_V
&=
36\kappa^2\lambda_K\sum_{i=0}^{T-1}\beta_i(\beta_is-\tau)(\lambda_K^2\beta_i-1) - \Delta^V,
\label{eq_lambda_exact_ode}
\\
\dot\lambda_K
&=
36\kappa^2\lambda_V\sum_{i=0}^{T-1}\beta_i(\beta_is-\tau)(3\lambda_K^2\beta_i-1) - \Delta^K.
\label{eq_mu_exact_ode}
\end{align}
The interaction terms $\Delta^V$ and $\Delta^K$ are the differences between the exact expressions and the main terms in \eqref{eq_lambda_exact_ode} and \eqref{eq_mu_exact_ode}. 
On any bounded set of $(\lambda_V,\lambda_K)$, their orders are as follows
\[
|\Delta^V|+|\Delta^K| = O\Big(\frac{1}{\log d}\Big)\Big(\sum_{i=0}^{T-1}\beta_i|\beta_is-\tau|+|s|+|\tau|\Big)+O\big(d^{-aC}\big).
\]
The first part vanishes at the construction where $s=\tau=0$. The second part is set by the length of training examples.
Consequently, the population dynamics close exactly in terms of $(\lambda_V,\lambda_K)$ and the deterministic weights $\beta_i$, up to the interaction terms, yielding a finite-dimensional ODE.

Now, we prove the convergence of training transformers on quadratic functions based on the above forms of ODEs. By Lemma~\ref{lemma_gradient_structure_quadratic}, the block-sparse structure is preserved under gradient flow, so the dynamics are fully characterized by the two reduced matrices
$W^{PV}_{31}(t)$ and $W^{KQ}_{13}(t)$. 
We further use an isotropic symmetry: $W^{KQ}_{13}(t)=\lambda_K(t)\id_d, \ W^{PV}_{31}(t)=\lambda_V(t)\id_d,  \forall t\ge 0$, in order to identify the fixed point for the ODE. Then, the model reduces to scalars $(\lambda_K(t),\lambda_V(t))$.

We now show the convergence in two stages.
In the first stage, using the structure of the ODE and the condition (here enforced by the colinear initialization $\theta_0=\alpha_0\theta^\star$ with $\alpha_0 \in [3/4, 1)$), we show that whenever $(\lambda_K,\lambda_V)$ is outside a small neighborhood of $(1, 1)$, the gradient norm is bounded below by a positive quantity depending at most polylogarithmically on $d$: 
$\|\nabla L^{\mathrm{CoT}}(\lambda_K,\lambda_V)\|_2^2 \;\ge\; c_1L^{\mathrm{CoT}}(\lambda_K,\lambda_V), 
\text{when } \|(\lambda_K,\lambda_V)-(1, 1)\|\ge \delta_0$, for some constants $c_1>0$ and $\delta_0=O(\frac{1}{\log d})$. 
This yields an exponential decrease of $L^{\mathrm{CoT}}$ until the trajectory enters the $\delta_0$-ball around $(1, 1)$, in time
$t_1=O(\log d)$.

In the second stage, inside the $\delta_0$-ball of $(1, 1)$, $L^{\mathrm{CoT}}$ is smooth, and its Hessian is positive definite (this follows from the identifiability of the construction in the assumption). Hence a local PL-condition inequality holds: $\|\nabla L^{\mathrm{CoT}}(\lambda_K,\lambda_V)\|_2^2 \;\ge\; c_2L^{\mathrm{CoT}}(\lambda_K,\lambda_V), 
\text{when } \|(\lambda_K,\lambda_V)-(1, 1)\|\le \delta_0, $for some $c_2>0$. This implies $L^{\mathrm{CoT}}(t)\le L^{\mathrm{CoT}}(t_1)\exp\big(-c_2(t-t_1)\big)$, so reaching $L^{\mathrm{CoT}}(t)\le \varepsilon$ takes an additional
$t_2=O(\log\frac1\varepsilon)$ time. Thus, the transformer converges to the construction as described in Lemma \ref{theorem_construction_three_head}. Following Lemma \ref{theorem_construction_three_head} and Theorem \ref{theorem_noise_stability_linear_self_attention_formal}, we similarly obtain the rate of the inference loss when training converges. 
\end{proof}

\section{Omitted Experimental Results}
\label{sec_additional_experiments}

\subsection{Evaluations of the Stability Measure}
 
Recall that Lemma \ref{thm_test_error_bound} is a worst-case upper bound, and the figures in the main text compare it to the measured loss gap on separate axes. We report the difference in scale explicitly. For each model and each reasoning length $T$, we report six quantities:
\begin{itemize}
    \item The loss gap $|L(f_W;\cD) - L^{\star}(f_W;\cD)|$ between the inference loss and the reference loss;
    \item The value of our bound, i.e., the right-hand side of \eqref{eq_main_bound};
    \item The bound-to-gap ratio, i.e., the ratio of the previous two quantities;
    \item The first-order term $C_1 \sum_{t=1}^{T-1} \rho^{(T)}_t \varepsilon_t$;
    \item The second-order term $\frac{C_1 C_2}{2} \sum_{t=2}^{T} \rho^{(T)}_t \bigl(\sum_{i=1}^{t-1} r_i\bigr)^2$;
    \item The two constants entering the bound: the Lipschitz constant $C_1$ of the loss function, and the Lipschitz constant $C_2$ of the model's input Jacobian.
\end{itemize}
We estimate $C_1$ as the maximum norm of the loss gradient with respect to the model output over the evaluation set, and $C_2$ as the maximum finite-difference change in the input Jacobian over pairs of nearby inputs from the same set. Table \ref{tab_bound_scale} reports all six quantities for Qwen-1.5B, Llama-1B, and Gemma-2B trained on the Bellman-Ford task, for $T = 1, \ldots, 5$. At $T=1$ there is no accumulated error, and both the gap and the bound are zero.
 
We report the results in Table \ref{tab_bound_scale}. First, the bound is relatively tight on Qwen-1.5B, where the bound-to-gap ratio stays between $6$ and $9$ across all $T$. On Llama-1B and Gemma-2B, the ratio reaches $2.3\times 10^{6}$ and $1.5\times 10^{5}$ at $T=5$. Second, this difference is attributed to the constants. $C_1$ is small and comparable across all three models (between $0.01$ and $0.07$), while $C_2$ differs by roughly three orders of magnitude, from around $5$ on Qwen-1.5B to above $3\times10^3$ on Llama-1B. Since $C_2$ multiplies the squared accumulated deviations in the second-order term, a large $C_2$ increases the bound value. Third, the second-order term grows faster in $T$ than the first-order term. For Llama-1B and Gemma-2B, it dominates the bound from $T=3$ onward. On Qwen-1.5B, the first-order term remains the larger of the two.
 
\begin{table}[t!]
\centering
\caption{Decomposition of the bound in \eqref{eq_main_bound} for three models trained on the Bellman-Ford task. {Loss gap denotes $|L(f_W;\cD) - L^{\star}(f_W;\cD)|$, Bound denotes the right-hand side of \eqref{eq_main_bound}, and Ratio denotes the bound value divided by the loss gap. The first- and second-order terms are $C_1\sum_{t=1}^{T-1}\rho^{(T)}_t\varepsilon_t$ and $\frac{C_1C_2}{2}\sum_{t=2}^{T}\rho^{(T)}_t\bigl(\sum_{i=1}^{t-1}r_i\bigr)^2$.} The absolute scale of the bound varies by orders of magnitude across architectures. At $T=1$ there is no accumulated error, so the loss gap and the bound value are zero.}\label{tab_bound_scale}
\small
\setlength{\tabcolsep}{5pt}
\renewcommand{\arraystretch}{1.1}
\begin{tabular}{@{}lcccccccc@{}}
\toprule
 & & & & & \multicolumn{2}{c}{Terms of \eqref{eq_main_bound}} & \multicolumn{2}{c}{Constants} \\
\cmidrule(lr){6-7} \cmidrule(l){8-9}
Model & $T$ & {Loss gap} & Bound & Ratio & First-order & Second-order & {$C_1$} & {$C_2$} \\
\midrule
 \multirow{5}{*}{Qwen-1.5B} & 1 & 0 & 0 & -- & 0 & 0 & 0.01 & 0 \\
  & 2 & 0.08 & 0.68 & 8.0 & 0.68 & $9.5\times 10^{-4}$ & 0.02 & 5.44 \\
  & 3 & 0.25 & 1.5 & 6.1 & 1.3 & 1.7 $\times 10^{-1}$ & 0.03 & 5.83 \\
  & 4 & 0.47 & 3.4 & 7.2 & 2.5 & 8.4 $\times 10^{-1}$ & 0.05 & 5.83 \\
  & 5 & 0.83 & 6.9 & 8.3 & 4.5 & 2.4 $\times 10^{0}$ & 0.07 & 5.83 \\
\midrule
 \multirow{5}{*}{Llama-1B} & 1 & 0 & 0 & -- & 0 & 0 & 0.04 & 0 \\
  & 2 & 0.05 & $3.3\times 10^{3}$ & $6.2\times 10^{4}$ & $2.3\times 10^{3}$ & $9.7\times 10^{2}$ & 0.03 & $3.03\times 10^{3}$ \\
  & 3 & 0.30 & $5.4\times 10^{4}$ & $1.8\times 10^{5}$ & $6.2\times 10^{3}$ & $4.8\times 10^{4}$ & 0.04 & $3.17\times 10^{3}$ \\
  & 4 & 0.85 & $5.5\times 10^{5}$ & $6.4\times 10^{5}$ & $6.4\times 10^{3}$ & $5.4\times 10^{5}$ & 0.05 & $3.18\times 10^{3}$ \\
  & 5 & 1.84 & $4.3\times 10^{6}$ & $2.3\times 10^{6}$ & $5.0\times 10^{4}$ & $4.2\times 10^{6}$ & 0.05 & $3.21\times 10^{3}$ \\
\midrule
 \multirow{5}{*}{Gemma-2B} & 1 & 0 & 0 & -- & 0 & 0 & 0.00 & 0 \\
  & 2 & 0.09 & $1.1\times 10^{3}$ & $3.6\times 10^{4}$ & $7.3\times 10^{2}$ & $4.3\times 10^{2}$ & 0.01 & $1.05\times 10^{3}$ \\
  & 3 & 0.48 & $3.5\times 10^{4}$ & $7.3\times 10^{4}$ & $1.5\times 10^{3}$ & $3.3\times 10^{4}$ & 0.03 & $1.15\times 10^{3}$ \\
  & 4 & 1.21 & $7.6\times 10^{4}$ & $6.2\times 10^{4}$ & $2.0\times 10^{3}$ & $7.4\times 10^{4}$ & 0.03 & $1.24\times 10^{3}$ \\
  & 5 & 2.11 & $3.2\times 10^{5}$ & $1.5\times 10^{5}$ & $2.7\times 10^{3}$ & $3.2\times 10^{5}$ & 0.06 & $1.37\times 10^{3}$ \\
\bottomrule
\end{tabular}
\end{table}

\begin{table}[t!]
\centering
\caption{Decomposition of \eqref{eq_main_bound} under quantization-aware training, for Qwen-1.5B at $T=5$. Lowering the bit-width monotonically reduces the error amplification factor {$\sum_{t=1}^{T-1}\rho^{(T)}_t$} and increases the {per-step prediction errors, summarized by $\sum_{t=1}^{T-1}\varepsilon_t$}; the bound and the test accuracy are determined by the balance of the two. The best bit-width per task is in bold.}\label{tab_bitwidth}
\small
\setlength{\tabcolsep}{6pt}
\renewcommand{\arraystretch}{1.1}
\begin{tabular}{@{}lcccccc@{}}
\toprule
Task & Bit-width & Test accuracy (\%) & {Loss gap} & Bound & {$\sum_{t} \rho^{(T)}_t$} & {$\sum_{t} \varepsilon_t$} \\
\midrule
 \multirow{5}{*}{Bellman-Ford} & 1 & 69.6 & 2.0 & 21.8 & 143.6 & 3.6 \\
  & 2 & 77.3 & 1.5 & 12.1 & 157.2 & 1.9 \\
  & 3 & 79.5 & 1.4 & 10.8 & 159.6 & 1.9 \\
  & \textbf{4} & \textbf{80.3} & 0.6 & 7.3 & 166.7 & 1.8 \\
  & 32 & 75.3 & 1.8 & 13.4 & 177.0 & 2.3 \\
\midrule
 \multirow{5}{*}{Dijkstra} & \textbf{1} & \textbf{76.2} & 0.9 & 2.1 & 9.3 & 2.9 \\
  & 2 & 75.3 & 0.9 & 2.4 & 9.4 & 2.9 \\
  & 3 & 74.2 & 1.1 & 3.3 & 13.9 & 2.7 \\
  & 4 & 73.8 & 1.6 & 5.2 & 22.6 & 1.9 \\
  & 32 & 73.1 & 2.3 & 14.2 & 34.4 & 1.8 \\
\midrule
 \multirow{5}{*}{Prim} & \textbf{1} & \textbf{66.3} & 1.1 & 1.1 & 11.9 & 3.0 \\
  & 2 & 65.6 & 1.3 & 1.4 & 14.7 & 3.0 \\
  & 3 & 64.7 & 1.7 & 1.5 & 20.1 & 2.8 \\
  & 4 & 63.5 & 2.0 & 2.0 & 30.1 & 2.5 \\
  & 32 & 62.3 & 2.1 & 2.7 & 69.8 & 2.3 \\
\bottomrule
\end{tabular}
\end{table}

\begin{table}[t!]
\centering
\caption{The error amplification factor tracks the {inference loss} across base models. We fine-tune Llama-1B and Gemma-2B on Bellman-Ford ($T=5$) and report both terms of \eqref{eq_main_bound} together with the estimated constants; {columns are as in Table~\ref{tab_bound_scale}}. Within each model, the ranking of methods by our bound matches their ranking by the loss gap, even though the absolute scale of the bound differs by orders of magnitude between the two models.}\label{tab_cross_model}
\small
\setlength{\tabcolsep}{5pt}
\renewcommand{\arraystretch}{1.1}
\begin{tabular}{@{}llcccccc@{}}
\toprule
 & & & & \multicolumn{2}{c}{Terms of \eqref{eq_main_bound}} & \multicolumn{2}{c}{Constants} \\
\cmidrule(lr){5-6} \cmidrule(l){7-8}
Model & Method & {Loss gap} & Bound & First-order & Second-order & {$C_1$} & {$C_2$} \\
\midrule
 \multirow{3}{*}{Llama-1B} & SFT-CoT & 1.84 & $4.3\times 10^{6}$ & $5.0\times 10^{4}$ & $4.2\times 10^{6}$ & 0.050 & $3.21\times 10^{3}$ \\
  & Coconut & 0.699 & $1.1\times 10^{5}$ & $1.3\times 10^{3}$ & $1.1\times 10^{5}$ & 0.077 & $1.63\times 10^{3}$ \\
  & Alg.~\ref{alg_quant_cot} & \textbf{0.491} & $\mathbf{1.3\times 10^{4}}$ & $7.4\times 10^{2}$ & $1.2\times 10^{4}$ & 0.091 & $1.79\times 10^{3}$ \\
\midrule
 \multirow{3}{*}{Gemma-2B} & SFT-CoT & 2.11 & $3.2\times 10^{5}$ & $2.7\times 10^{3}$ & $3.2\times 10^{5}$ & 0.060 & $1.37\times 10^{3}$ \\
  & Coconut & 0.662 & $1.5\times 10^{4}$ & $1.34\times 10^{2}$ & $1.5\times 10^{4}$ & 0.094 & $9.27\times 10^{2}$ \\
  & Alg.~\ref{alg_quant_cot} & \textbf{0.091} & $\mathbf{4.9\times 10^{3}}$ & 26.6 & $4.9\times 10^{3}$ & 0.064 & $4.78\times 10^{2}$ \\
\bottomrule
\end{tabular}
\end{table}

In Figure \ref{fig_scaling_other_models}, we measure the loss gap and the error amplification factor {$\sum_{t=1}^{T-1}\rho^{(T)}_t$} for $T$ ranging from 1 to 5. We report the results evaluated with Gemma-2B and Llama-1B models on the Bellman-Ford and Prim's algorithms. The {inference loss increases significantly with the number of steps}. The error amplification factor scales exponentially with $T$, closely tracking the growth in the inference loss. 

\begin{figure*}[h!]
\centering
    \begin{minipage}[t!]{0.3\textwidth}
        \centering
        \includegraphics[width=\linewidth]{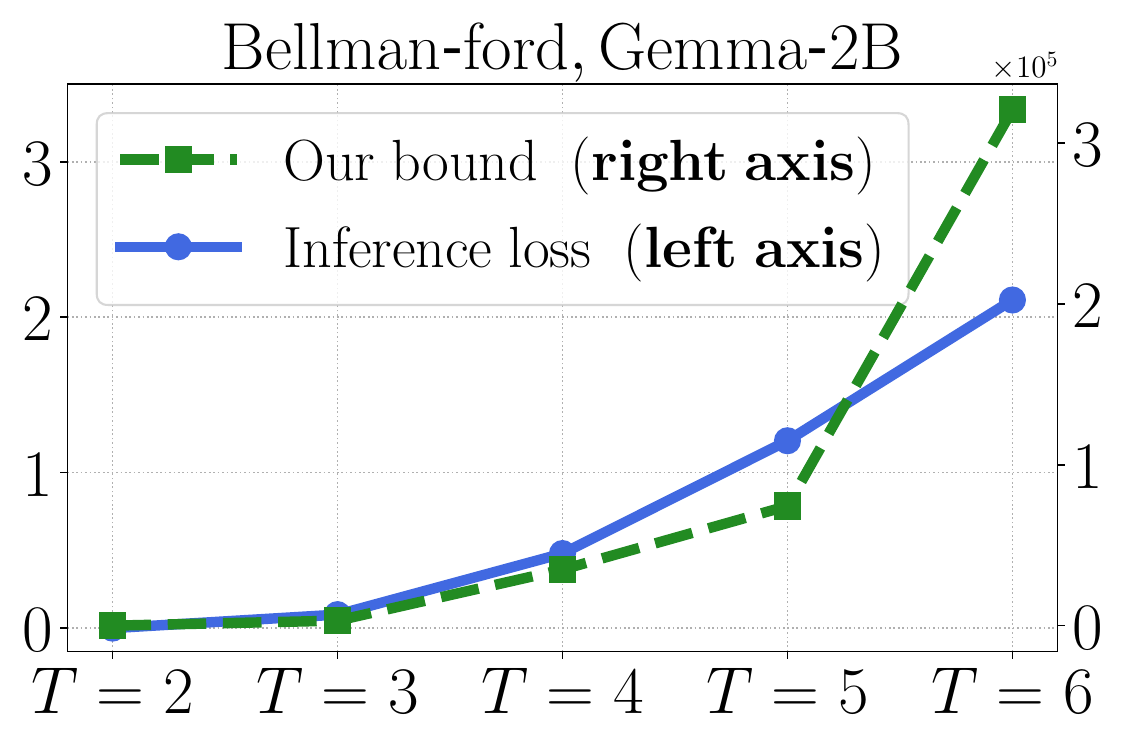}
    \end{minipage}\hspace{0.15\textwidth}
    \begin{minipage}[t!]{0.3\textwidth}
        \centering
        \includegraphics[width=\linewidth]{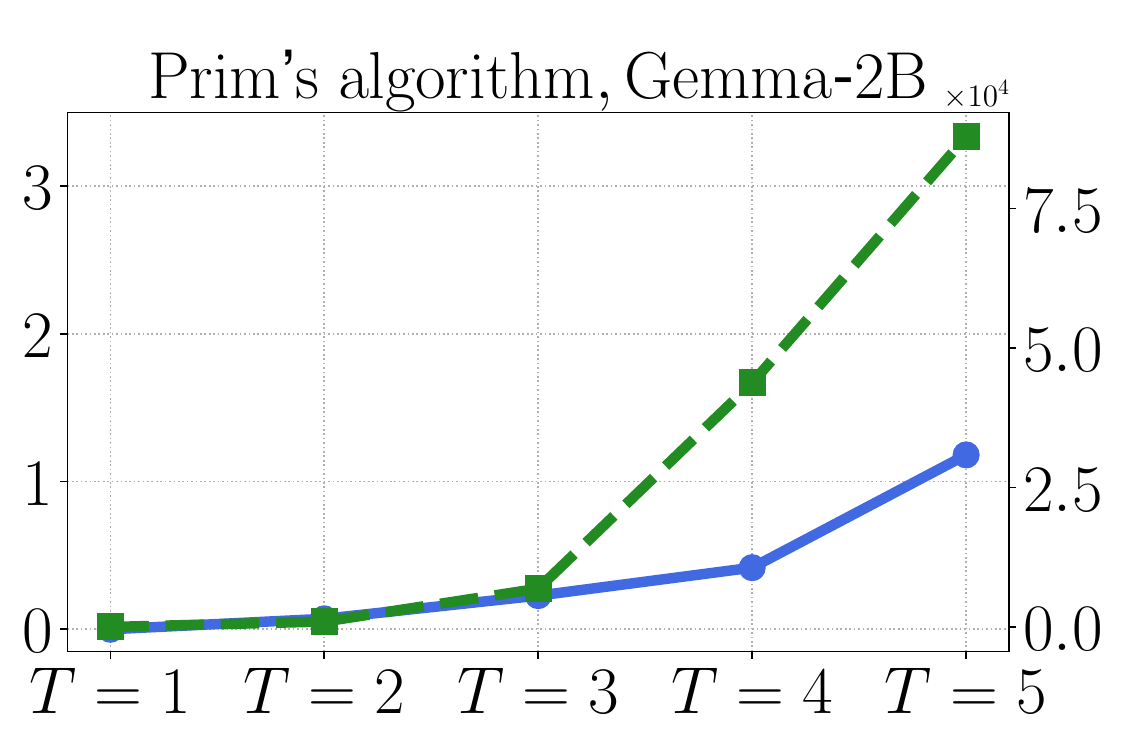}
    \end{minipage}\vfill
    \begin{minipage}[t!]{0.3\textwidth}
        \centering
        \includegraphics[width=\linewidth]{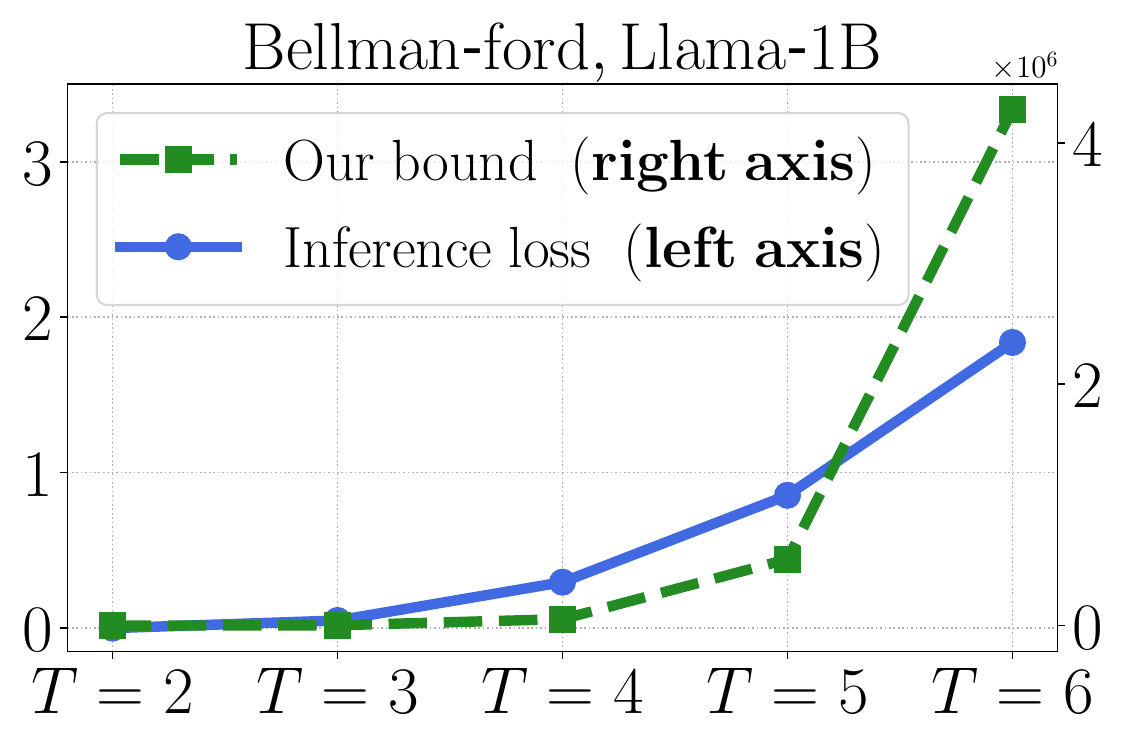}
    \end{minipage}\hspace{0.15\textwidth}
    \begin{minipage}[t!]{0.3\textwidth}
        \centering
        \includegraphics[width=\linewidth]{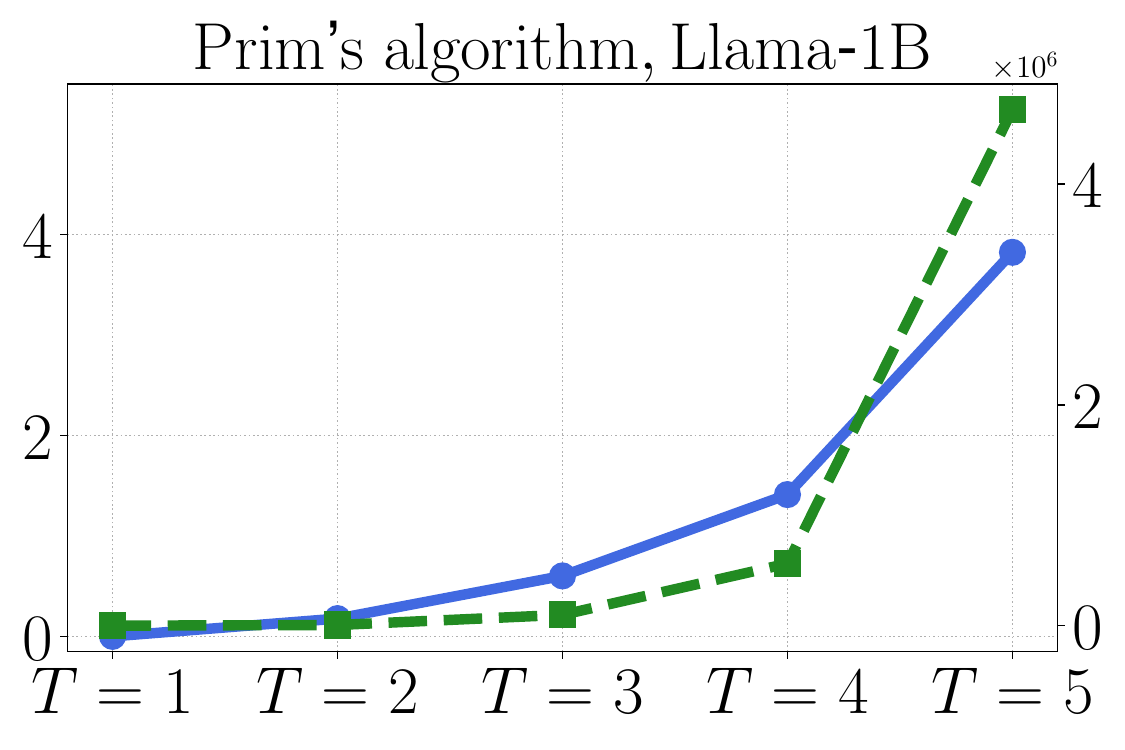}
    \end{minipage}
    \caption{We illustrate the loss gap and our bound in \eqref{eq_main_bound}, which can scale exponentially across generation steps, closely tracking the growth of the inference loss. Results are evaluated on Bellman-Ford and  Prim's algorithm, using Gemma-2B and Llama-1B models. The two quantities are plotted on separate axes and differ in absolute scale; see Table \ref{tab_bound_scale} for the bound-to-gap ratios.}\label{fig_scaling_other_models}
\end{figure*}

\paragraph{Evaluations of the bound under quantization.} 
Table~\ref{tab_bitwidth} decomposes \eqref{eq_main_bound} across bit-widths for Qwen-1.5B on three tasks. Lowering the bit-width has two opposing effects on the bound. 
First, it reduces the error amplification factor: from full precision to $1$ bit, $\sum_t \rho^{(T)}_t$ decreases by $3.7\times$ on the Dijkstra task, $5.9\times$ on the Prim task, and $1.2\times$ on the Bellman-Ford task. 
Second, the quantized model fits the training data less well, so the per-step prediction errors tend to grow. 

In practice, we observe that QAT reduces the error amplification factor more significantly than it increases prediction errors. As a result, lower bit-width quantization typically achieves a reduction in the bound, leading to improved test accuracy. On Dijkstra and Prim, the prediction errors increase by only $1.6\times$ and $1.3\times$. The bound decreases monotonically as the bit-width is lowered, by $6.8\times$ and $2.5\times$ at $1$ bit. The test accuracy improves by $3.1\%$ and $4.0\%$ over full precision. 
On Bellman-Ford, the reduction of $\sum_t \rho^{(T)}_t$ is small, while the prediction errors double from  $4$ bits to $1$ bit. Therefore, the bound and the test accuracy are both best at $4$ bits. The bit-width at which the growth of the error levels begins to dominate depends on the task, and we select it on the validation set.

\paragraph{Evaluations of our bound across algorithms.}
To evaluate whether the relationship between our bound and the inference loss extends beyond Qwen-1.5B, we apply Algorithm~\ref{alg_quant_cot} to Llama-1B and Gemma-2B on Bellman-Ford with $T=5$, and compare it with standard chain-of-thought fine-tuning (SFT-CoT) and continuous chain-of-thought (Coconut)~\citep{hao2025training}. As shown in Table~\ref{tab_cross_model}, our algorithm attains the lowest loss gap and the lowest bound on both models. Compared with Coconut, it reduces the loss gap by $58\%$ on average and the bound by $78\%$ on average. Compared with SFT-CoT, it lowers the bound by $330\times$ and $65\times$. This confirms that the reduction in test loss correlates with the reduction of our bound components. 

\paragraph{Numerical evaluation of Jacobian norms.}
We further analyze the training dynamics by examining the Hessian trace of models fine-tuned with chain-of-thought on linear and quadratic functions. Specifically, we evaluate the norm of the Jacobian matrix product as a function of the reasoning length $T$ and report results for both settings. As shown in Figure~\ref{fig_hessian_trace}, as the reasoning length $T$ increases, the norm of the Jacobian matrix product converges to a value. This behavior is consistent and aligned with our analysis of the gradient descent update rule for linear and quadratic functions.

\begin{figure*}[t!]
\centering

\begin{minipage}[t]{0.485\textwidth}
  \centering
  \includegraphics[width=0.5\linewidth]{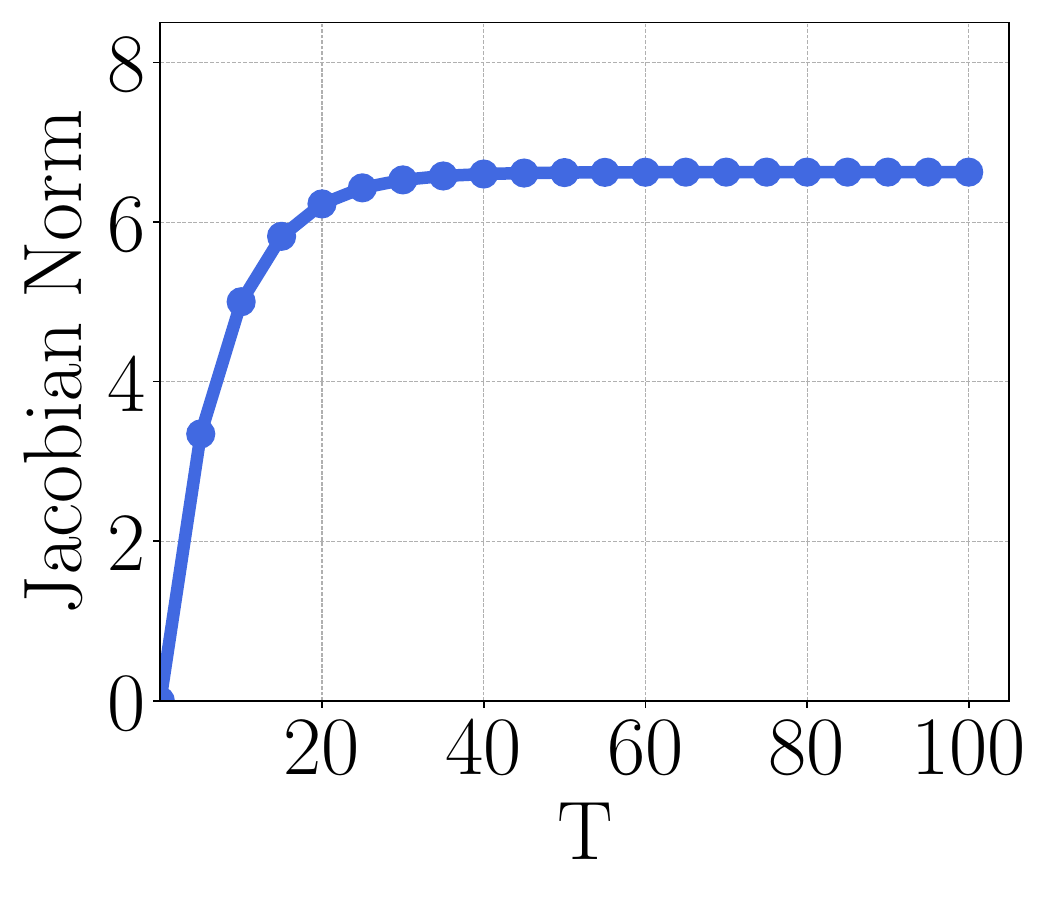}
  \subcaption{Jacobian matrix norm on linear functions}
  \label{fig_jacobian_linear}
\end{minipage}
\hfill
\begin{minipage}[t]{0.485\textwidth}
  \centering
  \includegraphics[width=0.5\linewidth]{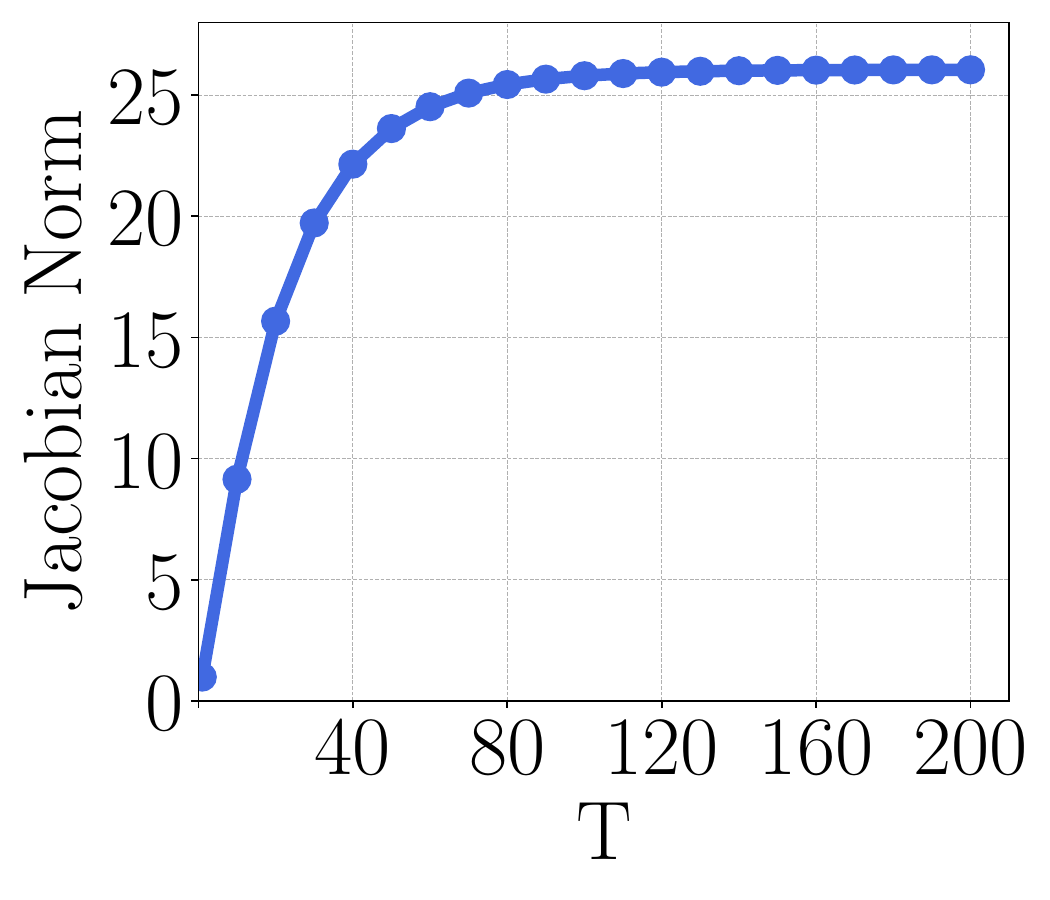}
  \subcaption{Jacobian matrix norm in quadratic functions}
  \label{fig_jacobian_quadratic}
\end{minipage}
\caption{We illustrate how transformers realize multi-step gradient descent behavior when trained on linear and quadratic functions. We verify that the Jacobian matrix norm is a converging series on linear functions and quadratic functions. The Jacobian norm of linear functions is $\sum_{i=1}^T \norm{(I - \frac{\eta}{n}X X^\top)^i}_2^2$ and the Jacobian norm of quadratic functions is $\sum_{i=1}^T \norm{(I - \frac{2\eta}{n}X D(\theta_i) X^\top)^i}_2^2$.}
\label{fig_hessian_trace}
\end{figure*}

\subsection{Sample Complexity Comparison}

We compare models fine-tuned with a varied number of training examples between $1000, 2000$, and $5000$. We also evaluate a model fine-tuned without intermediate steps. After training, we evaluate the noise stability using the fine-tuned model. We inject isotropic Gaussian noise with standard deviation $\sigma=0.005$ into the model weights and measure the increase in the inference loss. We compute the average loss increase with $10$ random seeds. We evaluate the noise stability of models fine-tuned with chain-of-thought on a range of algorithmic reasoning tasks, including DFS, Dijkstra, Prim, and DAG shortest path. 

First, we find that increasing the number of training samples leads to improved noise stability across all tasks. This suggests a strong correlation between generalization performance and robustness to weight noise.
For the algorithms of the same task, we noted a difference in noise stability. Specifically, models fine-tuned with the Dijkstra algorithm have more than $2\times$ lower Hessian traces than the Bellman-Ford algorithm. Additionally, models fine-tuned by Kruskal's algorithm yield $15\times$ lower Hessian trace values than Prim's. 

Second, we study the effect of intermediate steps on the sample complexity. We randomly subsample the intermediate steps $T$ to $20$, $10$, $5$, $2$, and $0$ where $T=0$ indicates training without intermediate steps. We find that a larger $T$ substantially reduces sample complexity. Using intermediate steps consistently achieves better sample complexity compared to training without intermediate steps.

\paragraph{Hyperparameters.}
We report the hyperparameters used in our experiments in Table \ref{tab_hyperparams}. 

\begin{table}[h!]
  \centering
  \caption{Summary of the hyperparameters used for each reasoning task. The exact configurations are selected from the defined search spaces during supervised fine-tuning.}\label{tab_hyperparams}
  \resizebox{\textwidth}{!}{
  \begin{tabular}{lcccccc}
    \toprule
    Task & Quantization & Optimizer & Learning rate & LoRA Rank & Sampling factor $\lambda$ & Latent steps $k$ \\ \midrule
    Dijkstra     & 1-bit & AdamW & $2\times10^{-5}$ & 16 & 0.1 & 1 \\
    Bellman-Ford & 4-bit & AdamW & $2\times10^{-5}$ & 16 & 0.2 & 1 \\
    Prim         & 4-bit & AdamW & $2\times10^{-5}$ & 16 & 0.1 & 1 \\
    BFS          & 1-bit & AdamW & $2\times10^{-5}$ & 16 & 0.1 & 1 \\
    DFS          & 3-bit & AdamW & $2\times10^{-5}$ & 16 & 0.1 & 1 \\
    Cyclic       & 1-bit & AdamW & $2\times10^{-5}$ & 16 & 1.0 & 1 \\
    Symmetry     & 1-bit & AdamW & $2\times10^{-5}$ & 16 & 1.0 & 1 \\
    \bottomrule
  \end{tabular}
  }
\end{table}

\subsection{Discussions}

Our formulation targets supervised fine-tuning on ground-truth steps. Reinforcement learning methods like Group Relative Policy Optimization or Proximal Policy Optimization optimize over the model's own rollouts. Our analysis extends to reinforcement learning in two key ways. First, because test-time generation is independent of the training objective, our stability analysis applies equally to models trained with reinforcement learning. Second, our work can be applied to study whether policy-gradient training implicitly regularizes input Jacobian norms along high-reward trajectories. We leave the formal extension of our bounds to the reinforcement learning setting for future work.

Lastly, we discuss several limitations of this work. %
First, our method regularizes these Jacobian norms during training. Future work could develop test-time regularization strategies that adapt to each input example during inference. Second, for black-box models, where gradients and Jacobians are not directly accessible, proxy stability measures based on noise perturbations may provide a practical alternative. Finally, extending the Jacobian-based error accumulation analysis to broader language-model tasks, including tool use, generation, and question answering, remains an interesting direction.

{
}

\end{document}